\documentclass[11pt]{article}

\usepackage{header_arxiv}

\usepackage{commands}

\usepackage{tikz}
\usetikzlibrary{arrows.meta}

\title{Convergence of Policy Mirror Descent \\
Beyond Compatible Function Approximation}

\author{
Uri Sherman%
\thanks{Blavatnik School of Computer Science, Tel Aviv University; \texttt{urisherman@mail.tau.ac.il}.
}
\and
Tomer Koren%
\thanks{Blavatnik School of Computer Science, Tel Aviv University, and Google Research; \texttt{tkoren@tauex.tau.ac.il}.}
\and
Yishay Mansour%
\thanks{Blavatnik School of Computer Science, Tel Aviv University, and Google Research; \texttt{mansour.yishay@gmail.com}. 
}
}

\begin{document}

\maketitle


\begin{abstract}
    Modern policy optimization methods roughly follow the policy mirror descent (PMD) algorithmic template, for which there are by now numerous theoretical convergence results.
    However, most of these either target tabular environments, or can be applied effectively only when the class of policies being optimized over satisfies strong closure conditions, which is typically not the case when working with parametric policy classes in large-scale environments. 
    In this work, we develop a theoretical framework for PMD for general policy classes where we replace the closure conditions with a strictly weaker variational gradient dominance assumption, and obtain upper bounds on the rate of convergence to the best-in-class policy. Our main result leverages a novel notion of smoothness with respect to a local norm induced by the occupancy measure of the current policy, and casts PMD as a particular instance of smooth non-convex optimization in non-Euclidean space.
\end{abstract}
\section{Introduction}
Modern policy optimization algorithms 
\citep{peters2006policy,peters2008reinforcement,lillicrap2015continuous, schulman2015trust, schulman2017proximal} operate by solving a sequence of stochastic optimization problems, each of which being roughly equivalent to:
\begin{align}\label{def:pmd_step}
    \pi^{k+1}
    \gets 
    \argmin_{\pi\in \Pi}\E_{s \sim \mu^{k}}\sbr{
        \abr[b]{\widehat Q^{k}_s, \pi_s} 
        + \frac1\eta B(\pi_s, \pi_s^k)
    },
\end{align}
where $\mu^k$ is a state probability measure (typically related, or equal to, the occupancy measure of the current policy $\pi^k$) from which sampling is granted through interaction with the environment; $\widehat Q^k$ is an estimate of the action-value function of $\pi^k$, and $B$ is a distance-like function employed to regularize the update so as to not stray too far from $\pi^k$.
The solution to \cref{def:pmd_step} is usually produced by optimizing a parametric neural network model $\pi_\theta$ (known as the actor, or policy network) via multiple steps of stochastic gradient descent, and consequently, the policy class $\Pi$ is the set of policies representable by the model; $\Pi = \cbr{\pi_\theta \mid \theta \in \R^p}$, where $p$ denotes the number of parameters in the network.

Contemporary theoretical analyses of this algorithm \citep{shani2020adaptive,agarwal2021theory,xiao2022convergence,ju2022policy,zhan2023policy, yuan2023loglinear,alfano2023novel} all have their roots in the online Markov decision process (MDP) framework, and roughly build on decomposing \cref{def:pmd_step} state-wise and casting the problem as a collection of independent online mirror descent 
steps~\citep{even2009online}. The disadvantage of such an approach lies in the requirement that the update step be exact (or almost exact) in each state independently, effectively limiting the applicability of such analyses to policy classes that are \emph{complete}, (i.e., $\Pi = \Delta(\cA)^\cS$), or otherwise satisfy strong closure conditions.

Largely, papers that develop convergence upper bounds for algorithms following \cref{def:pmd_step}, commonly known as Policy Mirror Descent (PMD; \citealp{tomar2020mirror,xiao2022convergence,lan2023policy}), fall into two main categories. The first category includes studies that target the tabular setup (e.g., \citealp{geist2019theory,shani2020adaptive,agarwal2021theory,xiao2022convergence,johnson2023optimal,lan2023policy,zhan2023policy}), where no sampling distribution $\mu^k$ is involved (or it has no effect) and updates are performed in a per-state manner. The second category consists of papers that consider parametric policy classes (e.g., \citealp{agarwal2021theory,alfano2022linear,ju2022policy,yuan2023loglinear,alfano2023novel,xiong2024dual}) often building---either directly or indirectly---on the compatible function approximation framework~\citep{sutton1999policy}. 
As such, these works essentially assume that the update in \cref{def:pmd_step} remains ``close'' to the one that would have been performed over the complete policy class (see  \cref{sec:intro_discussion} for further discussion).
This state of affairs is (at least partially) due to the fact that policy gradient methods in the general policy class setting are prone to local optima~\citep{bhandari2024global}, and as a result, structural assumptions are necessary to establish global optimality guarantees.

The present paper aims to establish \emph{best-in-class} convergence of PMD (\cref{def:pmd_step}) for general policy classes, relaxing the stringent closure conditions and assuming instead a \emph{variational gradient dominance} (VGD) condition \citep{bhandari2024global,agarwal2021theory,xiao2022convergence}.
It can be shown that a general form of closure conditions implies VGD and that the converse does not hold, hence it is a strict relaxation of the setup assumptions (see detailed discussion in \cref{sec:intro_discussion,sec:vgd_discussion}).
Our main result features a novel analysis technique that casts \cref{def:pmd_step} as a particular instance of smooth non-convex optimization in a non-Euclidean space, where the smoothness of the objective is w.r.t.~a \emph{local} norm induced by the current policy occupancy measure. Importantly, this approach leads to rates independent of the cardinality of the state space.
In contrast, previous results that establish convergence of gradient based methods (though not of PMD; e.g., \citealp{agarwal2021theory,bhandari2024global,xiao2022convergence}) that are applicable in our setting, lead to bounds that depend on the size of the state-space, thus rendering them useful only in tabular setups.

\subsection{Main results}
We consider the problem of finding an (approximately) optimal policy in a
discounted MDP $\cM = (\cS, \cA, \P, r, \gamma, \rho_0)$ within a general policy class $\Pi\subset \Delta(\cA)^{\cS}$.
We assume the action set is finite 
$A\eqq |\cA|$, and denote the effective horizon by $H\eqq \frac{1}{1-\gamma}$.
Our goal is to minimize the value $V(\pi)$, defined as the long term discounted cost (we interpret $r\colon\cS \times \cA \to [0,1]$ as measuring regret, or cost).
Our central structural assumption, that replaces and relaxes specific closure conditions, is the following.
\begin{definition}[Variational Gradient Dominance]\label{def:vgd_mdp}
    We say that $\Pi$ satisfies a $(C_\star, \epsvgd)$-variational gradient dominance (VGD) condition w.r.t.~$\cM$, 
    if there exist constants $C_\star, \epsvgd>0$, such that for any policy $\pi\in \Pi$:
    \begin{align}\label{eq:M_vgd}
        V(\pi) - V^\star(\Pi) 
        \leq 
        C_\star \max_{\tilde \pi \in \Pi}\abr{\nabla V(\pi), \pi - \tilde \pi} + \epsvgd.
    \end{align}
\end{definition}
We note that any policy class satisfies the above conditions with some $C_\star \geq 1, \epsvgd \leq H$, and that the complete policy class is $\br{H\norm{\fraci{\mu^\star}{\rho_0}}_\infty, 0}$-VGD w.r.t.~any MDP (see \citealp{bhandari2024global,agarwal2021theory}, and \cref{lem:vgd_complete} for completeness).
Our main result is the following.
\begin{theorem*}[informal]
    Let $\Pi\subset \Delta(\cA)^\cS$ be convex and assume it satisfies $(C_\star, \epsvgd)$-VGD w.r.t.~$\cM$. 
    Suppose further that the actor and critic are approximately optimal up to some error $\epsstatM > 0$.
    Then, with well tuned $\varepsilon$-greedy exploration and learning rate $\eta$, we have that
    the PMD method (\cref{def:pmd_step}) converges as follows.
    With Euclidean regularization,
    \begin{align*}
            V(\pi^k) - \min_{\pi^\star\in \Pi}V(\pi^\star)=
            \bigO\br{
            \frac{C_\star^2 H^3 A^{3/2} }{k^{2/3}}
            + \br{C_\star H + A H^2 k^{1/6}}\sqrt{\epsstatM}
            + \epsvgd},
        \end{align*}
    and with negative Entropy regularization, we have that 
        \begin{align*}
        V(\pi^k) - \min_{\pi^\star\in \Pi}V(\pi^\star)=
        \bigO\br{
                \frac{C_\star^2 H^3A^{3/2}  }{k^{2/7}}
                + \br{C_\star H + A^2 H^3 k^{4/7}}\sqrt{\epsstatM} 
                + \epsvgd},
        \end{align*}
    where the big-O only suppresses constant numerical factors.
\end{theorem*}

To obtain our main result, our analysis casts
PMD as a proximal point algorithm 
in a non-Euclidean setting (see \citealp{teboulle2018simplified} for a review), where the proximal operator uses a regularizer that adapts to \emph{local} smoothness of the objective. As we demonstrate in \cref{lem:value_local_smoothness}, the approximation error of the linearization of the objective $V(\cdot)$ at $\pi^k$ can be bounded w.r.t.~the local norm $\norm{\cdot}_{L^2(\mu^k)}$; crucially, a norm according to which the decision set $\Pi$ has diameter independent of the cardinality of the state-space. 
This significantly deviates from the commonly used smoothness of the value function w.r.t.~the Euclidean norm \citep{agarwal2021theory}, which assigns a diameter of $|\cS|$ to $\Pi$, and therefore leads to rates that have merit only in tabular environments.

\subsection{Discussion: VGD vs.~Closure}
\label{sec:intro_discussion}
Our work establishes best-in-class convergence subject to the VGD condition presented in the previous section. This is a substantially different starting point than that of the prevalent closure conditions based on the compatible function approximation approach \citep{sutton1999policy} assumed in recent works on parametric policy classes \citep{agarwal2021theory,yuan2023loglinear,alfano2023novel,xiong2024dual}. 
The assumptions employed in these works fall into two main categories; The first and more general one is that of a bounded \emph{approximation} error \citep[e.g.,][]{alfano2023novel,yuan2023loglinear}, which essentially requires that the update step in \cref{def:pmd_step} be close (up to a small error) to the update that would have been performed over the complete policy class $\Pi_{\rm all} \eqq \Delta(\cA)^\cS$.
The second is that of bounded \emph{transfer} error \citep[e.g.,][]{agarwal2021theory,yuan2023loglinear}, which roughly requires that the update be accurate (up to a small error) when accuracy is measured over the optimal policy occupancy measure. This assumption is commonly employed in the specific log-linear policy class setup; to the best of our knowledge, there do not exist results that employ these conditions in a fully general policy class setting (\citealp{agarwal2021theory} consider a non-PMD method in a bounded transfer error setup where the policy class satisfies additional smoothness assumptions). 

\begin{table*}[ht]
\centering
\small
\renewcommand{\arraystretch}{2.2}
\begin{threeparttable}
\caption{
Comparison of assumptions and bounds of representative prior works for PMD with fixed step size. Columns refer to assumptions required either implicitly or explicitly by different works.
The \textbf{Realizability} column refers to approximate realizability, which is implied by closure conditions.
The \textbf{Rate} column suppresses all factors other than $K$, and ignores error floors.
\textbf{{\scriptsize\textbullet} Closure (perfect):} The policy class is closed to a PMD update up to $\ell_\infty$-norm error. 
\textbf{{\scriptsize\textbullet} Closure (approx):} The policy class is closed to a PMD update up to error that depends on the sampling distribution.
\textbf{{\scriptsize\textbullet} General dual with EMaP parametrization:} EMaP stands for Exact Mirror and Project; in these works the policy class is induced by a general dual variable parametrization, combined with an operator that performs the mirror and project steps accurately.
}{\small
\begin{tabular}{
    @{} >{\centering\arraybackslash}p{3.1cm} 
        >{\centering\arraybackslash}p{1.2cm} 
        >{\centering\arraybackslash}p{1.9cm}
        >{\centering\arraybackslash}p{1.8cm}
        >{\centering\arraybackslash}p{2.3cm}     
        >{\centering\arraybackslash}p{2.3cm} 
        >{\centering\arraybackslash}p{1.2cm} 
    }
\toprule 
\textbf{Paper} &
\textbf{\makecell{VGD\tnote{*}}} &
\textbf{\makecell{$\boldsymbol{\Pi}$ \\ Convexity}} &
\textbf{\makecell{Realiz-\\ability}} &
\textbf{\makecell{Closure \\ Assumptions}} &
\textbf{\makecell{Parametric \\ Assumptions}} &
\textbf{Rate}
\\
\midrule
\makecell{\citet{xiao2022convergence}; \\ \citet{lan2023policy}} 
& Yes & No\tnote{a} & Yes & Yes (perfect) & Tabular & $1/K$ \\
\makecell{\citet{yuan2023loglinear}} 
& Yes\tnote{b} & No & Yes & Yes (approx) & Log-linear & $1/K$ \\
\makecell{\citet{ju2022policy}\tnote{c}} 
& Yes & No & Yes & Yes (perfect) 
& \makecell{General dual\\w/ EMaP} & $1/\sqrt K$ \\
\makecell{\citet{alfano2023novel}} 
& Yes & No & Yes & Yes (approx) 
& \makecell{General dual\\w/ EMaP}  & $1/K$ \\
\midrule
\makecell{\textbf{This Work}\tnote{d}}
& Yes & Yes\tnote{e} & No & No & No & $1/K^{2/3}$ \\
\bottomrule
\end{tabular}}
\label{tab:po_algs_comparison}
\vspace{1ex}

\begin{tablenotes}
\footnotesize
\item[*] Prior works do not assume VGD directly.
VGD is implied by closure conditions subject to a slight variation of the concentrability coefficient assumption; see \cref{sec:closure_implies_vgd} for further details.
\item[a] Prior works on the tabular setting typically assume the policy class is complete $\Pi=\Delta(\cA)^\cS$, and thus convex. However their arguments extend to the case that $\Pi$ satisfies perfect closure, which eliminates the need for $\Pi$ being convex.
\item[b] We refer to the bounds obtained by \citet{yuan2023loglinear} subject to bounded approximation error. \citet{yuan2023loglinear} also obtain convergence subject to bounded transfer error ---
it is unclear to what extent (if at all) bounded transfer error implies VGD.
\item[c] \citet{ju2022policy} also obtain an $O(1/K)$ rate for regularized PMD.
\item[d] We report our rate for Euclidean PMD. More generally, our bounds depend on the smoothness of the action regularizer, and dependence on $K$ degrades for non-Euclidean regularizers such as negative entropy.
\item[e] Assuming only VGD without closure, our analysis requires convexity of $\Pi$. However, in the presence of closure assumptions such as those of \citet{alfano2023novel}, our analysis does not require convexity of $\Pi$ (see \cref{sec:closure_non_convex} for further details).
\end{tablenotes}
\end{threeparttable}
\vspace{-2ex}
\end{table*}

The relation between closure and VGD is subtle, primarily because closure conditions are algorithm dependent. Typically, they relate to one or more of the following three elements; step-size range, action regularizer, and the particular algorithmic approach employed to solve \cref{def:pmd_step}.
At the same time, the VGD condition is algorithm independent, as it relates only to the policy class-MDP combination.
Nonetheless, as we show in \cref{sec:closure_implies_vgd}, a reasonable extension of PMD closure conditions implies variational gradient dominance, effectively establishing PMD closure $\Rightarrow$ VGD. At a high level, this builds on a similar claim
from \citet{bhandari2024global}, that closure to policy improvement implies VGD.
We further demonstrate in \cref{sec:vgd_not_imply_closure} that the converse does not hold; that there exist simple examples where the VGD condition holds whereas closure does not take place. We refer to \cref{tab:po_algs_comparison} for a high level comparison between our work and prior art, and conclude this section with the following additional remarks. 
\begin{itemize}
	\item \textbf{Realizability.} Closure conditions generally imply (approximate) realizability, thus under this assumption convergence w.r.t.~the true optimal policy $\pi^\star=\argmin_{\pi\in \Delta(\cA)^\cS}V(\pi)$ is possible. We do not assume realizability and therefore prove convergence to the optimal \emph{in-class} policy. Specifically, all prior works prove bounds that only hold in (approximately) realizable settings, while our bounds do not require realizability.	
    
    \item \textbf{Geometric rates.} \cref{tab:po_algs_comparison} reports rates for fixed step size PMD. Many prior works that study PMD in the tabular setup or subject to closure conditions establish linear convergence for geometrically increasing step size sequences \citep{xiao2022convergence,johnson2023optimal,yuan2023loglinear,alfano2023novel}. We do not expect such rates are possible assuming only VGD. Roughly speaking, the reason these rates are attainable subject to closure is that the algorithm dynamics mimic those of policy iteration in the tabular setting, where convergence is indeed at a linear rate. Assuming only VGD, policy iteration no longer converges, as the policy class loses the favorable structure allowing for convergence of such an aggressive algorithm. This should highlight the value in studying the function approximation setup without closure assumptions.

	\item \textbf{Convexity of the policy class.} Unlike prior works, we consider VGD instead of closure but additionally require convexity of the policy class $\Pi$. However, subject to perfect closure, it can be shown that the iterates of PMD satisfy optimality conditions w.r.t.~a convex policy class that contains $\Pi$ (concretely, it will be the complete policy class $\Delta(\cA)^\cS$), which is the key element required in our analysis.
    Thus, our analysis accommodates non convex policy classes as long as perfect closure holds. 
    We refer the reader to \cref{sec:closure_non_convex} for a more formal discussion.
\end{itemize}

\subsection{Additional related work}
\label{sec:related_work}
\paragraph{PMD with non-tabular policy classes.} Most closely related to our work are papers that study convergence of PMD in setups where the policy class is given by function approximators \citep{vaswani22ageneral,ju2022policy,grudzien2022mirror,alfano2023novel,xiong2024dual}. 
The motivation of \citet{alfano2023novel,xiong2024dual} is somewhat related to ours but they address a different aspect of the problem in question. These works focus on the approximation errors in the update step (thus essentially assuming closure) and propose algorithmic mechanisms to ensure it is small, but obtain meaningful upper bounds only when it is indeed small w.r.t.~the exact steps over the complete policy class (as discussed in the previous section). 
There is a long line of works on parametric policy classes and specific instantiations of PMD such as the Natural Policy Gradient (NPG; \citealp{kakade2001natural}); which is the focus of, e.g., \citet{alfano2022linear,yuan2023loglinear,cayci2024convergence} as well as \citet{agarwal2021theory}. Many works also study convergence dynamics induced by particular policy classes, e.g., \citet{liu2019neural,wang2020neural,liu2020improved}; we refer the reader to \citet{alfano2023novel} for an excellent and more detailed account of these works.

Several prior works have made the observation that PMD is a mirror descent step on the linearization of the value function with a dynamically weighted regularization term \citep{shani2020adaptive,tomar2020mirror,vaswani22ageneral,xiao2022convergence}, which is the starting point of our work. In particular, this perspective is the focus of \citet{vaswani22ageneral}; however this work did not establish any convergence guarantees.

\paragraph{PMD in the tabular setting.}
The modern analysis approach for PMD in the generic (agnostic to the regularizer) tabular setup is due to \citet{xiao2022convergence}. Additional works that study the tabular setup include \citet{geist2019theory,lan2023policy,johnson2023optimal,zhan2023policy}.
As in the function approximation case, many works study convergence of the prototypical PMD instantiation; the NPG or its derivatives TRPO \citep{schulman2015trust} and PPO \citep{schulman2017proximal} in tabular or softmax-tabular settings, e.g., \citet{agarwal2021theory,shani2020adaptive,cen2022fast,bhandari2021linear, khodadadian2021linear, khodadadian2022linear}.

\paragraph{Policy Gradients in parameter space.} 
There is a rich line of work into policy gradient algorithms that take gradient steps \emph{in parameter space}, both in the tabular and non-tabular setups \citep{zhang2020global,mei2020global, mei2021leveraging,yuan2022general,mu2024second}. Most of the results in the case of non-tabular, generic parameterizations characterize convergence in terms of conditions on the parametric representation. We refer the reader to \citet{yuan2022general} for further review.

One particular work of interest into policy gradients that shares some conceptual elements with ours is that of \citet{bhandari2024global}, which characterizes conditions that allow policy gradients to converge. Roughly speaking, these conditions include (i) VGD holds in parameter space w.r.t.~the per iteration advantage objective, and (ii) that the policy class is closed to policy improvement (in the policy iteration sense).
Our work, on the other hand, establishes VGD in policy space (i.e., w.r.t.~the direct parametrization) allows PMD to converge, and furthermore with rates independent of $S$.

\paragraph{Bregman proximal point methods.}
As mentioned, our analysis builds on realizing PMD as an instance of a Bregman proximal point algorithm --- roughly, this is a proximal point algorithm \cite{rockafellar1976monotone} in a non-Euclidean setting (see \citet{teboulle2018simplified} for a review).
There are numerous studies that investigate non-Euclidean proximal point methods for both convex and non-convex objective functions
(e.g., \citealp{tseng2010approximation,ghadimi2016mini,bauschke2017descent,lu2018relatively,zhang2018convergence,fatkhullin2024taming}; see also \citealp{beck2017first})
, although none of them accommodate the particular setup that PMD fits into (see \cref{sec:opt_analysis} for details).
Our analysis for the proximal point method presented in \cref{sec:opt_main} is mostly inspired by the work of \citet{xiao2022convergence}; specifically, their upper bounds for projected gradient descent, where they apply a proximal point analysis in the euclidean setting.

\section{Preliminaries}
\paragraph{Discounted MDPs.}
A discounted MDP $\cM$ is defined by a tuple
	$\cM = (\cS, \cA, \P, r, \gamma, \rho_0)$,
where $\cS$ denotes the state-space, $\cA$ the action set, $\P\colon \cS \times \cA \to \Delta(\cS)$ the transition dynamics, $r \colon \cS \times \cA \to [0, 1]$ the reward function, $0<\gamma< 1$ the discount factor, $H\eqq \frac{1}{1-\gamma}$ the effective horizon, and $\rho_0\in \Delta(\cS)$ the initial state distribution.
For notational convenience, for $s,a\in \cS\times\cA$ we let $\P_{s, a} \eqq \P(\cdot \mid s, a) \in \Delta(\cS)$ denote the next state probability measure.

We assume the action set is finite with $A\eqq |\cA|$, and identify $\R^A$ with $\R^\cA$. We additionally assume, for clarity of exposition and in favor of simplified technical arguments, that the state space is finite with $S\eqq |\cS|$, and identify $\R^\cS$ with $\R^S$.
We emphasize that all our arguments may be extended to the infinite state-space setting with additional technical work.
An agent interacting with the MDP is modeled by a policy $\pi \colon \cS \to \Delta(\cA)$, for which we let $\pi_s \in \Delta(\cA)\subset\R^A$ denote the action probability vector at $s$ and $\pi_{s, a}\in [0,1]$ denote the probability of taking action $a$ at $s$. We denote the \emph{value} of $\pi$ when starting from a state $s\in \cS$ by $V_s(\pi)$:
\begin{align*}
    V_s(\pi) \eqq \E\sbr{\sum_{t=0}^\infty \gamma^t r(s_t, a_t) \mid s_0 = s, \pi},
\end{align*}
and more generally for any $\rho\in \Delta(\cS)$,
$V_\rho(\pi) \eqq \E_{s \sim \rho}V_{s}(\pi)$. When the subscript is omitted, $V(\pi)$ denotes value of $\pi$ when starting from the initial state distribution $\rho_0$:
\begin{align*}
    V(\pi) \eqq V_{\rho_0}(\pi) =  \E\sbr{\sum_{t=0}^\infty \gamma^t r(s_t, a_t) \mid s_0 \sim \rho_0, \pi}.
\end{align*}
For any state action pair $s,a\in \cS\times \cA$, the action-value function of $\pi$, or $Q$-function, measures the value of $\pi$ when starting from $s$, taking action $a$, and then following $\pi$ for the reset of the interaction:
\begin{align*}
    Q^\pi_{s, a} \eqq \E\sbr{\sum_{t=0}^\infty \gamma^t r(s_t, a_t) \mid s_0 = s, a_0=a, \pi}
\end{align*}
We further denote the discounted state-occupancy measure of $\pi$ induced by any start state distribution $\rho\in \Delta(\cS)$ by $\mu^\pi_\rho$:
\begin{align*}
    \mu^\pi_{\rho}(s) \eqq 
    \br{1-\gamma}\sum_{t=0}^\infty \gamma^t \Pr(s_t = s \mid s_0 \sim \rho, \pi).
\end{align*}
It is easily verified that $\mu^\pi\in \Delta(\cS)$ is indeed a state probability measure.
In the sake of brevity, we take the MDP true start state distribution $\rho_0$ as the default in case one is not specified:
\begin{align}
    \mu^\pi \eqq \mu^\pi_{\rho_0}.
\end{align}

\paragraph{Learning objective.}
In the conventional formulation of MDPs, the objective is to maximize the discounted total reward, i.e., $\max_\pi V(\pi)$. In this paper, we follow \citet{xiao2022convergence} and adopt a minimization formulation in order to better align with
conventions in the optimization literature. To this end, we regard each $r(s, a) \in [0, 1]$ as a value measuring regret, or cost, rather than reward. Given any reward function $r$, we may reset $r(s, a) \gets 1 - r(s, a)$ for all $s,a\in \cS \times \cA$ to transform it into a regret function.
With this in mind, we consider the problem of finding an approximately optimal policy within a given policy class $\Pi\subset \Delta(\cA)^\cS$:
\begin{align}\label{def:rl_opt_problem}
    \argmin_{\pi \in \Pi} V(\pi).
\end{align}
To avoid ambiguity, we denote the optimal value attainable by an in-class policy (a solution to \cref{def:rl_opt_problem}) by $\VstarPi$, and the optimal value attainable by any policy by $V^\star$:
\begin{align}
    \VstarPi \eqq \argmin_{\pi^\star \in \Pi}V(\pi^\star);
    \quad
    V^\star \eqq \argmin_{\pi^\star \in \Delta(\cA)^\cS}V(\pi^\star).
\end{align}
We note that we do not make any explicit structural assumptions about $\cM$. We will however make some assumptions about the policy class $\Pi$, which will be made clear in the statements of our theorems.

\subsection{Problem setup}
In this work, we focus on the PMD method \cref{alg:pmd} for solving \cref{def:rl_opt_problem} in the case that the policy class is non-complete, $\Pi\neq \Delta(\cA)^\cS$.
\begin{algorithm}[tb]
   \caption{Policy Mirror Descent (on-policy)}
   \label{alg:pmd}
\begin{algorithmic}
   \STATE {\bfseries Input:} learning rate $\eta > 0$, regularizer $R\colon\R^\cA \to \R$
   \STATE Initialize $\pi^1 \in \Pi$
   \FOR{$k=1$ {\bfseries to} $K$}
   \STATE Set $\mu^k \eqq \mu^{\pi^k}$; $\widehat Q^k \eqq \widehat Q^{\pi^k}$.
   \STATE 
   \begin{aligni*}
        \pi^{k+1}
        \gets 
        \argmin\limits_{\pi\in \Pi}\E_{s\sim \mu^k}\sbr{
            H \abr{\widehat Q^{k}_s, \pi_s} 
            + \frac1\eta B_R(\pi_s, \pi_s^k)
        }
   \end{aligni*}
   \ENDFOR
\end{algorithmic}
\end{algorithm}
In each iteration, PMD solves a stochastic optimization sub-problem formed by an estimate of the current policy $Q$-function and a Bregman divergence term which is defined below.
\begin{definition}[Bregman divergence]\label{def:bregman_divergenvce}
    Given a convex differentiable regularizer $R \colon \R^\cA \to \R$, the Bregman divergence w.r.t.~$R$ is:
    \begin{align*}
        B_R(u, v) \eqq R(u) - R(v) - \abr{\nabla R(v), u -v}.
    \end{align*}
\end{definition}
Throughout, we make the following assumptions regarding the solutions to the sub-problems and the $Q$-function estimates \cref{alg:pmd}. 
\begin{assumption}[Sub-problem optimization oracle]
\label{assm:opt_oracle}
    We assume that for all $k$, $\pi^{k+1}$ is approximately optimal, in the sense that constrained optimality conditions hold up to error $\epsactM$:
    \begin{align*}
        \forall \pi \in \Pi, 
        \abr{\nabla \phi_k(\pi^{k+1}), \pi - \pi^{k+1}} \geq -\epsactM,
    \end{align*}
    where
    \begin{aligni*}
        \phi_k(\pi) \eqq \E_{s\sim \mu^k}\sbr{
            H \abr{\widehat Q^{k}_s, \pi_s} 
            + \frac1\eta B_R(\pi_s, \pi_s^k)}.
    \end{aligni*}
\end{assumption}
\begin{assumption}[Q-function oracle]
\label{assm:Q_oracle}
    We assume that for all $\pi$, 
    \begin{align*}
        \E_{s \sim \mu^\pi}\sbr{\norm[b]{\widehat Q_{s}^\pi - Q_{s}^\pi}_2^2} \leq \epscritM.
    \end{align*}
\end{assumption}
    We remark that our results can be easily adapted to 
    somewhat weaker conditions on the critic error; we defer the discussion to \cref{sec:discuss_critic_error}.
    
\paragraph{Additional notation.}
Given a state probability measure $\mu\in \Delta(\cS)$ and an action space norm $\norm{\cdot}_\circ\colon \R^A \to \R$, we define the induced state-action weighted $L^p$ norm $\norm{\cdot}_{L^p(\mu), \circ}\colon \R^{SA}\to \R$ as follows:
\begin{align*}
	\norm{u}_{L^p(\mu), \circ}
	&\eqq \br{\E_{s \sim \mu}\norm{u_s}_\circ^p}^{1/p}.
\end{align*}
For any norm $\norm{\cdot}$, we let $\norm{\cdot}^*$ denote its dual. When discussing a generic norm and there is no risk of confusion, we may use $\norm{\cdot}_*$ to refer to its dual.
We repeat the following notation that is used throughout the paper for convenience:
\begin{align*}
    \mu^\pi \eqq \mu^\pi_{\rho_0}, \quad 
    S \eqq |\cS|, \quad
    A \eqq |\cA|, \quad
    H \eqq \frac{1}{1-\gamma}.
\end{align*}

\subsection{Optimization preliminaries}
We proceed with several basic definitions before concluding the setup.
\begin{definition}[Lipschitz Gradient]\label{def:Lipschitz_gradient}
We say a function $h\colon\Omega \to \R$, $\Omega\subseteq \R^d$ has an $L$-Lipschitz gradient or is $L$-smooth w.r.t.~a norm $\norm{\cdot}$ if for all $x,y\in \Omega$:
\begin{align*}
    \norm{\nabla h(x) - \nabla h(y)}_* \leq L\norm{x - y}.
\end{align*}
\end{definition}

\begin{definition}[Gradient Dominance]\label{def:gradient_dominance}
    We say $f\colon \cX \to \R$ satisfies the variational gradient dominance condition with parameters $(C_\star, \delta)$, or that $f$ is $(C_\star, \delta)$-VGD, if
    here exist constants $C_\star, \delta>0$, such that for any $x \in \cX$, it holds that:
    \begin{align*}
        f(x) - \argmin_{x^\star\in \cX}f(x^\star)
        \leq 
        C_\star \max_{\tilde x \in \cX}\abr{\nabla f(x), x - \tilde x} + \delta.
    \end{align*}
\end{definition}

\begin{definition}[Local Norm]\label{def:local_norm}
    We define a \emph{local} norm over a set $\cX\subseteq \R^d$ by a mapping $x \mapsto \norm{\cdot}_x$ such that $\norm{\cdot}_x$ is a norm for all $x\in \cX$.
    We may denote a local norm by $\norm{\cdot}_{(\cdot)}$ or by $x \mapsto \norm{\cdot}_x$.
\end{definition}

\begin{definition}[Local Smoothness]\label{def:local_smoothness}
    We say $f\colon \cX \to \R$ is
    $\beta$-\emph{locally} smooth w.r.t.~a local norm $x \mapsto \norm{\cdot}_x$ if for all $x, y\in \cX$:
    \begin{align*}
    \av{f(y) - f(x) - \abr{\nabla f(x), y - x}} \leq 
    \frac\beta2\norm{y-x}_x^2 .
    \end{align*}
\end{definition}

\section{Best-in-class convergence of Policy Mirror Descent}
In this section, we present our main results which establish convergence rates for the PMD method in the non-complete class setting we consider. 
Our main theorem, given below, provides convergence rates for two classic instantiations of PMD; with Euclidean regularization and negative entropy regularization.
Our results require that $\epsilon$-greedy exploration be incorporated into the policy class.
To that end, let $\Pi^\epsilon$ denote the policy class obtained by adding $\epsilon$-greedy exploration to $\Pi$:
\begin{align*}
    \Pi^\epsilon \eqq \cbr{(1-\epsilon)\pi + \epsilon u \mid \pi \in \Pi},
    \text{ where }u_{s, a} \equiv 1/A.
\end{align*}
We have the following.
\begin{theorem}\label{thm:pmd_main}
    Let $\Pi\subset \Delta(\cA)^\cS$ be convex and assume it is $(C_\star, \epsvgd)$-VGD w.r.t.~$\cM$. 
    Consider the on-policy PMD method \cref{alg:pmd}  when run over $\Pi^\epsexpl$.
    Then, assuming $\epsactM + \epscritM \leq \epsstatM$, and with proper tuning of $\eta, \epsexpl$, it holds that:
    
    \begin{enumerate}[label=\roman*., leftmargin=0.4cm] 
        \item If $R(p) = \frac12\norm{p}_2^2$ is the Euclidean action-regularizer, we have  
        \begin{align*}
            V(\pi^K) - \VstarPi=
            \bigO\br{
            \frac{C_\star^2 A^{3/2} H^3}{K^{2/3}}
            + \br{C_\star H + A H^2 K^{1/6}}\sqrt{\epsstatM}
            + \epsvgd}
        \end{align*}

        \item If $R(p) = \sum_i p_i\log p_i$ is the negative entropy action-regularizer, we have  
        \begin{align*}
            V(\pi^K) - \VstarPi=
            \bigO\br{
                \frac{C_\star^2 A^{3/2} H^3 }{K^{2/7}}
                + \br{C_\star H + A^2 H^3 K^{4/7}}\sqrt{\epsstatM} 
                + \epsvgd}.
        \end{align*}
        
    \end{enumerate}
    In both cases, 
    big-O notation suppresses only constant factors.
\end{theorem}
To our knowledge, \cref{thm:pmd_main} is the first result to establish best-in-class convergence (at any rate) of PMD without closure conditions. Two additional comments are in order:
(1) Our current analysis technique requires the action regularizer to be smooth. This is also the source of the degraded rate in the negative entropy case.
(2) The greedy exploration stems from the smoothness parameter we establish for the value function, and leads to worse rates in the Euclidean case (for negative entropy, it actually implies smoothness of the regularizer, though this is not the primary reason for which it is introduced). We discuss this point further in \cref{sec:value_smooth_main}.

\paragraph{Analysis overview.}
The analysis leading up to \cref{thm:pmd_main} builds on casting PMD as an instance of a Bregman proximal point (or equivalently, a mirror descent) algorithm. This follows by demonstrating PMD proceeds by optimizing subproblems formed by linear approximations of the value function and a proximity term that adapts to \emph{local} smoothness of the objective, as measured by the norm induced by the current policy occupancy measure.

In fact, it has already been previously observed \citep[e.g.,][]{shani2020adaptive, xiao2022convergence} that the on-policy PMD update step is completely equivalent to a mirror descent step w.r.t.~the value function gradient equipped with a dynamically weighted proximity term. For any two policies $\pi$ and $\pi^k$, by the policy gradient theorem (\citealp{sutton1999policy}, see also \cref{lem:value_pg} in \cref{sec:aux_lemmas}):
\begin{align}\label{eq:pmd_omd_equiv}
    \E_{s \sim \mu^{k}}\Big[
        & H \abr{Q^{k}_s, \pi_s} 
        + \frac1\eta B_R(\pi_s, \pi_s^k)
    \Big]
    =
        \abr{\nabla V(\pi^k), \pi} 
        + \frac1\eta B_{\pi^k}(\pi, \pi^k),
\end{align}
where we denote $\mu^k \eqq \mu^{\pi^k}, Q^k \eqq Q^{\pi^k}$, and $B_{\pi^k}(u, v) \eqq \E_{s \sim \mu^k}B_R(u_s, v_s)$. 
However, these prior observations did not yield new convergence results, as the algorithm in question significantly deviates from a standard instantiation of mirror descent; a priori, it is unclear how the regularizer associated with $B_{\pi^k}$ relates to the objective in optimization terms.

The high level components of our analysis are outlined next. In \cref{sec:value_smooth_main} we establish local smoothness of the value function (\cref{lem:value_local_smoothness}), which is the key element in establishing convergence of PMD through a proximal point algorithm perspective.
Then, in \cref{sec:opt_main} we introduce the optimization setup that accommodates proximal point methods that adapt to local smoothness of the objective, and present the  convergence guarantees for this class of algorithms. Finally, we return to prove \cref{thm:pmd_main} in \cref{sec:proof:thm:pmd_main}, where we apply both \cref{lem:value_local_smoothness} and the result of \cref{sec:opt_main} to establish convergence of PMD.

\subsection{Local smoothness of the value function}
\label{sec:value_smooth_main}
The principal element of our approach builds on smoothness of the value function w.r.t.~the local norm induced by the occupancy measure of the policy at which we take the linear approximation, given by the below lemma. We defer the proof to \cref{sec:proof:lem:value_local_smoothness}.
\begin{lemma}\label{lem:value_local_smoothness}
    Let $\pi\colon \cS \to \Delta(\cA)$ be any policy such that $\epsilon\eqq \min_{s, a}\cbr{\pi_{sa}} > 0$.
    Then, for any $\tilde \pi \in \cS \to \Delta(\cA)$, we have:
    \begin{align*}
        &\av{V(\tilde \pi) - V(\pi) - \abr{\nabla V(\pi), \tilde \pi - \pi}}
        \leq 
        \min\cbr{
        \frac{ H^3}{\sqrt \epsilon}
            \norm{\tilde \pi - \pi}_{L^2(\mu^\pi), 1}^2,
        \frac{ A H^3}{\sqrt \epsilon}
            \norm{\tilde \pi - \pi}_{L^2(\mu^\pi), 2}^2
        }.
    \end{align*}
\end{lemma}
It is instructive to consider
\cref{lem:value_local_smoothness} in the context of the more standard non-weighted $L^2$ smoothness property established in \citet{agarwal2021theory}.
\begin{itemize}[topsep=0pt, leftmargin=0.3cm] 
    \item \textbf{Dependence on $\boldsymbol S$:} The standard $L^2$ smoothness leads to rates that scale with $\norm{\pi^1 - \pi^\star}_2$, which scales with $S$ in general.  Indeed, prior works that exploit smoothness of the value function (e.g., \citealp{agarwal2021theory,xiao2022convergence}) derive bounds for PGD (i.e., mirror descent with non-local, euclidean regularization) that do in fact hold in the  setting we consider here, but inevitably lead to convergence rates that scale with the cardinality of the state-space. This is while the diameter assigned to the decision set $\Pi$ by $\norm{\cdot}_{L^2(\mu^\pi), \circ}$, for any $\pi$, depends only on the diameter assigned to $\Delta(\cA)$ by $\norm{\cdot}_\circ$, and thus is independent of $S$.
    \item \textbf{Relation to PMD:} The standard $L^2$ smoothness does not naturally integrate with the PMD framework, and leads to algorithms (such as vanilla projected gradient descent) where the update step cannot be framed as a solution to a stochastic optimization problem induced by some policy occupancy measure. As such, these do not admit a formulation that is easily implemented in practical applications.
    \item \textbf{Smoothness parameter:} The smoothness parameter in \cref{lem:value_local_smoothness} depends on the minimum action probability assigned by the policy at which we linearize the value function (and as we discuss in \cref{sec:discuss_value_local_smoothness}, this is not an artifact of our analysis). A simple resolution for this is given by adding $\epsilon$-greedy exploration. Notably, the relatively large $\bigO(1/\sqrt\epsilon)$ smoothness constant ultimately leads to a rate that is worse than the $\bigO(1/K)$ achievable with the standard $L^2$ smoothness (but that crucially, does not scale with $S$).
\end{itemize}

\subsection{Digression: Constrained non-convex optimization for locally smooth objectives}
\label{sec:opt_main}
In this section, we consider the constrained optimization problem:
\begin{align}\label{opt:objective}
	\min_{x\in \cX} f(x),
\end{align}
where the decision set $\cX\subseteq \R^d$ is convex and endowed with a local norm $x \mapsto \norm{\cdot}_x$ (see \cref{def:local_norm}), and $f$ is differentiable over an open domain that contains $\cX$. We assume access to the objective is granted through an approximate first order oracle, as defined next.
\begin{assumption}\label{assm:opt_grad_oracle}
We have first order access to $f$ through an $\epsgrad$-approximate gradient oracle; For all $x\in \cX$, we have \begin{align*}
\norm{\hgrad{f}{x} - \nabla f(x)}_x^* \leq \epsgrad \leq 1.
\end{align*}
\end{assumption}
\cref{thm:opt_main} given below establishes convergence rates for the algorithm we describe next.
Given an initialization $x_1\in \cX$, learning rate $\eta > 0$, and local regularizer $R_x\colon \R^d \to \R$ for all $x\in \cX$, iterate for $k=1, \ldots, K:$
\begin{align}\label{eq:alg_opt_omd}
    &x_{k+1} = \argmin_{y\in \cX} 
    		\cbr{\abr{\hgrad{f}{x}, y} + \frac1\eta B_{R_x}(y, x)} .
\end{align}
The above algorithm can be viewed as either a mirror descent algorithm \citep{nemirovskij1983problem,beck2003mirror}  or a proximal point algorithm \citep{rockafellar1976monotone} in a non-Euclidean setup (see \citealp{teboulle2018simplified} for a review), where the non-smooth term is the decision set indicator function.
Our analysis (detailed in \cref{sec:opt_analysis}) hinges on a descent property of the algorithm, thus naturally takes the proximal point perspective. We prove the following.
\begin{theorem}\label{thm:opt_main}
    Suppose that $f$ is $(C_\star, \epsvgd)$-VGD as per \cref{def:gradient_dominance}, and that $f^\star \eqq \min_{x\in \cX}f(x) > -\infty$. 
    Assume further that:
    \begin{enumerate}[label=(\roman*)]
        \item The local regularizer $R_x$ is $1$-strongly convex and has an $L$-Lipschitz gradient w.r.t.~$\norm{\cdot}_x$ for all $x\in \cX$.
        
        \item For all $x\in \cX$, 
        $\max_{u,v\in \cX}\norm{u - v}_x\leq D$, and
        $\norm{\nabla f(x)}^*_x \leq M$.
        
        \item $f$ is $\beta$-locally smooth w.r.t.~$x \mapsto \norm{\cdot}_{x}$. 
    \end{enumerate}
    Then, assuming $x^{k+1}$ are $\epsopt$-approximately optimal (in the same sense of \cref{assm:opt_oracle}),
    the proximal point algorithm \cref{eq:alg_opt_omd}
	has the following guarantee when $\eta\leq 1/(2\beta)$:
    \begin{align*}
		f(x_{K+1}) - f^\star 
		&= O\br{\frac{C_\star^2 L^2 c_1^2}{\eta K}
            + \cE_{\rm err}
            + \epsvgd
            }
        \end{align*}
	where $c_1 \eqq D + \eta M$ and
    \begin{align*}
        \cE_{\rm err} \eqq \br{C_\star D + c_1 L^2}\epsgrad
        + C_\star \epsopt 
        + c_1 L \sqrt{\epsopt/\eta}.
    \end{align*}
	where $c_1 \eqq D + \eta M$.

\end{theorem}
The proof of \cref{thm:opt_main} as well as additional technical details for this section are provided in 
\cref{sec:opt_analysis}.
\subsection{Proof of main result}
\label{sec:proof_main}
To prove our main result, we begin with a lemma that essentially ``maps'' the PMD setup into the optimization framework of \cref{sec:opt_main}. The proof consists of showing that the appropriate assumptions on actor, critic, and action regularizer translate to the conditions of \cref{thm:opt_main} for locally smooth optimization.
\begin{lemma}\label{lem:pmd_main}
    Let $\Pi$ be a convex policy class that is $\br{C_\star, \epsvgd}$-VGD w.r.t.~the MDP $\cM$.
    Consider the on-policy PMD method \cref{alg:pmd}, and assume that the following conditions hold:
    \begin{enumerate}[label=(\roman*)]
        \item \label{thmassm:pmd_main_R} $R\colon \R^{\cA} \to \R$ is $1$-strongly convex and has an $L$-Lipschitz gradient w.r.t.~an action-space norm $\norm{\cdot}_\circ$.
        
        \item \label{thmassm:pmd_main_DM} $\max_{p,q\in \Delta(\cA)}\norm{p  - q}_\circ\leq D$, and $\norm{Q_s^\pi}_\circ^*\leq M$ for all $s \in \cS, \pi\in \Pi$.
        
        \item \label{thmassm:pmd_main_smoothness} The value function is $\beta$-locally smooth over $\Pi$ w.r.t.~the local norm $\norm{\cdot}_\pi \eqq \norm{\cdot}_{L^2(\mu^\pi), \circ}$.
    \end{enumerate}
    Then, we have the following guarantee:
    \begin{align*}
        V(\pi^{K}) - \VstarPi
        = \bigO\br{
        \frac{C_\star^2 L^2 c_1^2}{\eta K}
        + \cE_{\rm stat}
        + \epsvgd
        }
    \end{align*}
    where $c_1 \eqq D + \eta H M$, and
    \begin{align*}
        \cE_{\rm stat} = \br{C_\star D + c_1 L^2}H\sqrt \epscritM
        + C_\star \epsactM
        + c_1 L \sqrt{\epsactM/\eta}.
    \end{align*}
\end{lemma}
\begin{proof}[\unskip\nopunct]
    For $\mu\in \R^S, Q\in \R^{SA}$, we define the state to state-action element-wise product $\mu \circ Q \in \R^{SA}$ by 
\begin{aligni*}
	\br{\mu \circ Q}_{s, a} \eqq \mu(s)Q_{s, a}.
\end{aligni*}
        Observe that for all $k$, it holds that
        \begin{align*}
            \E_{s \sim \mu^{k}}\Big[
                H\abr{\widehat Q^{k}_s, \pi_s}& 
                + \frac1\eta B_R(\pi_s, \pi_s^k)
            \Big]
            \\&= 
            \abr{\hgrad{V}{\pi^k}, \pi} 
            + \frac1\eta B_{\pi^k}(\pi, \pi^k),
        \end{align*}
        with:
        \begin{aligni*}
        B_{\pi^k}(\pi, \tilde \pi) \eqq \E_{s\sim \mu^k} B_R(\pi_s, \tilde \pi_s),\;
        \hgrad{V}{\pi} \eqq H \mu^\pi \circ \widehat Q^\pi.
        \end{aligni*}
        Next, we demonstrate PMD is an instance of the optimization algorithm \cref{eq:alg_opt_omd}, and verify that all of the conditions in \cref{thm:opt_main} hold w.r.t.~the local norm $\pi \mapsto \norm{\cdot}_{L^2(\mu^\pi), \circ}$.
        First, to see that the gradient error is bounded by $H \sqrt{\epscritM}$, observe:
        \begin{align*}
            \Big\Vert\hgrad{V}{\pi}
            - \nabla V(\pi)\Big\Vert_{L^2(\mu^\pi), \circ}^*
            &=
            \norm{\mu^\pi \circ \br{\widehat Q^\pi - Q^\pi}}_{L^2(\mu^\pi), \circ}^*
            \\
            &= 
            \sqrt{\E_{s \sim \mu^\pi}\br{\norm{\widehat Q^\pi - Q^\pi}_\circ^*}^2}
            \\
            &\leq \sqrt{\epscritM},
        \end{align*}
        where second inequality follows from \cref{lem:wnorm_dual_mu} and the inequality from \cref{assm:Q_oracle}.
        Further:
        \begin{enumerate} 
            \item By a simple relation (\cref{lem:reg_transform}) between $R$ and the state-action it regularizer it induces defined below,
        \begin{align*}
            R_{\pi^k}(\pi)\eqq \E_{s \sim \mu^k}R(\pi_s),
        \end{align*}
        we have that $B_{\pi^k}(\cdot, \cdot)$ is the Bregman divergence of $R_{\pi^k}$, and further using
        \ref{thmassm:pmd_main_R} that 
        $R_{\pi^k}$ is $1$-strongly convex and has an $L$-Lipschitz gradient w.r.t.~$\norm{\cdot}_{L^2(\mu^k), \circ}$.
            \item For all $\pi, \pi', \tilde \pi$, by \ref{thmassm:pmd_main_DM},
        \begin{align*}
            \norm{\pi' - \tilde \pi}_{L^2(\mu^\pi), \circ}
            =\sqrt{\E_{s\sim \mu^\pi}\norm{\pi'_s - \tilde \pi_s}^2}
            \leq D.
        \end{align*}
        In addition by \ref{thmassm:pmd_main_DM} and the dual norm expression (\cref{lem:wnorm_dual_mu}), for any $\pi$:
        \begin{align*}
            \norm{\nabla V(\pi)}_{L^2(\mu^\pi), \circ}^*
            &=
            H\norm{\mu^\pi \circ Q^\pi}_{L^2(\mu^\pi), \circ}^*
            \\
            &=
            H\sqrt{\E_{s \sim \mu^\pi} \br{\norm{Q^\pi_s}_\circ^*}^2}
            \leq H M.
        \end{align*}
            \item Finally, the objective is $\beta$-locally smooth by assumption \ref{thmassm:pmd_main_smoothness}.
        \end{enumerate}
        The result now follows from \cref{thm:opt_main}.
    \end{proof}
    We conclude with a proof sketch of \cref{thm:pmd_main} for the Euclidean case; the full technical details are provided in \cref{sec:proof:thm:pmd_main}.
    \begin{proof}[Proof sketch of \cref{thm:pmd_main} (Euclidean case)]
    The first step is showing that the $\epsexpl$-greedy exploration introduces an error term that scales with $\delta \eqq \epsexpl C_\star H^2 A$ (see \cref{lem:epsgreedy_vgd}). This implies that $\Pi^{\epsexpl}$ is $(C_\star, \epsvgd + \delta)$-VGD w.r.t.~$\cM$.
    In addition, by definition of $\Pi^{\epsexpl}$ we have
    $\min_{s, a} \cbr{\pi_{s, a}} \geq \epsexpl/A$
    for all $\pi\in \Pi^\epsexpl$.
    We now argue the following:
   
    \begin{enumerate}
        \item The action regularizer $R(p)=\frac12\norm{p}_2^2$ is $1$-strongly convex and has $1$-Lipschitz gradient w.r.t.~$\norm{\cdot}_2$.
        \item $\forall s, \norm{\pi_s - \tilde \pi_s}_2\leq D = 2$, $\norm{Q_s}_2\leq M=\sqrt{A H}$.
        \item By \cref{lem:value_local_smoothness}, the value function is 
        $\br{\beta \eqq \frac{A^{3/2} H^3}{\sqrt \epsexpl}}$-locally smooth w.r.t.~$\pi \mapsto \norm{\cdot}_{L^2(\mu^\pi), 2}$.
    \end{enumerate}
    The result now follows from \cref{lem:pmd_main} with $\eta=1/(2\beta)$ and $\epsexpl=K^{-2/3}$.
\end{proof}

\section{Conclusions and outlook}
In this work, we introduced a novel theoretical framework and established best-in-class convergence of PMD for general policy classes, subject to an algorithm independent variational gradient dominance condition instead of a closure condition. In addition, we discussed the relation between VGD and closure thoroughly, and demonstrated closure implies VGD but not the other way around (\cref{sec:intro_discussion,sec:vgd_discussion}).
We conclude by outlining two directions for valuable (in our view) future research.
\begin{itemize}[leftmargin=0.4cm] 
    \item \textbf{$\epsilon$-greedy exploration.} Our approach builds on ensuring descent on each iteration, 
    which we establish by demonstrating local smoothness holds \emph{globally}, for any reference policy $\tilde \pi$. 
    As we discuss in \cref{sec:discuss_value_local_smoothness}, it seems that this technique cannot yield better results. However, when the multiplicative ratio $|\pi_{s,a}/ \tilde \pi_{s,a}-1|$ is bounded, arguments similar to those given in \cref{sec:proof:lem:value_local_smoothness} demonstrate a somewhat weaker notion of smoothness --- but \emph{without} dependence on the exploration parameter.
    Furthermore, an analysis approach that combines with the classic mirror descent analysis might do without the per iteration descent property.

    \item \textbf{Non-smooth action regularizers.} Our approach encounters an obstacle that seems related
to existing techniques for non-convex, non-Euclidean proximal point methods, which leads to
the requirement of a smooth regularizer. This is also the source of the degraded rate in the
negative entropy case. Progress can be made by either advancing state-of-the-art in this area
of optimization (or showing the limitation is inherent to the setup), or alternatively exploiting
additional structure specific to the value function.

\end{itemize}

\section*{Acknowledgements}
This project has received funding from the European Research Council (ERC) under the European Union’s Horizon 2020 research and innovation program (grant agreements No.~101078075; 882396).
Views and opinions expressed are however those of the author(s) only and do not necessarily reflect those of the European Union or the European Research Council. Neither the European Union nor the granting authority can be held responsible for them.
This work received additional support from the Israel Science Foundation (ISF, grant numbers 3174/23; 1357/24), and a grant from the Tel Aviv University Center for AI and Data Science (TAD). This work was partially supported by the Deutsch Foundation.


\bibliography{main}

\begin{thebibliography}{57}
\providecommand{\natexlab}[1]{#1}
\providecommand{\url}[1]{\texttt{#1}}
\expandafter\ifx\csname urlstyle\endcsname\relax
  \providecommand{\doi}[1]{doi: #1}\else
  \providecommand{\doi}{doi: \begingroup \urlstyle{rm}\Url}\fi

\bibitem[Agarwal et~al.(2021)Agarwal, Kakade, Lee, and Mahajan]{agarwal2021theory}
A.~Agarwal, S.~M. Kakade, J.~D. Lee, and G.~Mahajan.
\newblock On the theory of policy gradient methods: Optimality, approximation, and distribution shift.
\newblock \emph{Journal of Machine Learning Research}, 22\penalty0 (98):\penalty0 1--76, 2021.

\bibitem[Alfano and Rebeschini(2022)]{alfano2022linear}
C.~Alfano and P.~Rebeschini.
\newblock Linear convergence for natural policy gradient with log-linear policy parametrization.
\newblock \emph{arXiv preprint arXiv:2209.15382}, 2022.

\bibitem[Alfano et~al.(2023)Alfano, Yuan, and Rebeschini]{alfano2023novel}
C.~Alfano, R.~Yuan, and P.~Rebeschini.
\newblock A novel framework for policy mirror descent with general parameterization and linear convergence.
\newblock \emph{Advances in Neural Information Processing Systems}, 36:\penalty0 30681--30725, 2023.

\bibitem[Bauschke et~al.(2017)Bauschke, Bolte, and Teboulle]{bauschke2017descent}
H.~H. Bauschke, J.~Bolte, and M.~Teboulle.
\newblock A descent lemma beyond lipschitz gradient continuity: first-order methods revisited and applications.
\newblock \emph{Mathematics of Operations Research}, 42\penalty0 (2):\penalty0 330--348, 2017.

\bibitem[Beck(2017)]{beck2017first}
A.~Beck.
\newblock \emph{First-order methods in optimization}.
\newblock SIAM, 2017.

\bibitem[Beck and Teboulle(2003)]{beck2003mirror}
A.~Beck and M.~Teboulle.
\newblock Mirror descent and nonlinear projected subgradient methods for convex optimization.
\newblock \emph{Operations Research Letters}, 31\penalty0 (3):\penalty0 167--175, 2003.

\bibitem[Bhandari and Russo(2021)]{bhandari2021linear}
J.~Bhandari and D.~Russo.
\newblock On the linear convergence of policy gradient methods for finite mdps.
\newblock In \emph{International Conference on Artificial Intelligence and Statistics}, pages 2386--2394. PMLR, 2021.

\bibitem[Bhandari and Russo(2024)]{bhandari2024global}
J.~Bhandari and D.~Russo.
\newblock Global optimality guarantees for policy gradient methods.
\newblock \emph{Operations Research}, 2024.

\bibitem[Cayci et~al.(2024)Cayci, He, and Srikant]{cayci2024convergence}
S.~Cayci, N.~He, and R.~Srikant.
\newblock Convergence of entropy-regularized natural policy gradient with linear function approximation.
\newblock \emph{SIAM Journal on Optimization}, 34\penalty0 (3):\penalty0 2729--2755, 2024.

\bibitem[Cen et~al.(2022)Cen, Cheng, Chen, Wei, and Chi]{cen2022fast}
S.~Cen, C.~Cheng, Y.~Chen, Y.~Wei, and Y.~Chi.
\newblock Fast global convergence of natural policy gradient methods with entropy regularization.
\newblock \emph{Operations Research}, 70\penalty0 (4):\penalty0 2563--2578, 2022.

\bibitem[Chen and Jiang(2019)]{chen2019information}
J.~Chen and N.~Jiang.
\newblock Information-theoretic considerations in batch reinforcement learning.
\newblock In \emph{International Conference on Machine Learning}, pages 1042--1051. PMLR, 2019.

\bibitem[Epasto et~al.(2020)Epasto, Mahdian, Mirrokni, and Zampetakis]{epasto2020optimal}
A.~Epasto, M.~Mahdian, V.~Mirrokni, and E.~Zampetakis.
\newblock Optimal approximation-smoothness tradeoffs for soft-max functions.
\newblock \emph{Advances in Neural Information Processing Systems}, 33:\penalty0 2651--2660, 2020.

\bibitem[Even-Dar et~al.(2009)Even-Dar, Kakade, and Mansour]{even2009online}
E.~Even-Dar, S.~M. Kakade, and Y.~Mansour.
\newblock Online markov decision processes.
\newblock \emph{Mathematics of Operations Research}, 34\penalty0 (3):\penalty0 726--736, 2009.

\bibitem[Fatkhullin and He(2024)]{fatkhullin2024taming}
I.~Fatkhullin and N.~He.
\newblock Taming nonconvex stochastic mirror descent with general bregman divergence.
\newblock In \emph{International Conference on Artificial Intelligence and Statistics}, pages 3493--3501. PMLR, 2024.

\bibitem[Geist et~al.(2019)Geist, Scherrer, and Pietquin]{geist2019theory}
M.~Geist, B.~Scherrer, and O.~Pietquin.
\newblock A theory of regularized markov decision processes.
\newblock In \emph{International Conference on Machine Learning}, pages 2160--2169. PMLR, 2019.

\bibitem[Ghadimi et~al.(2016)Ghadimi, Lan, and Zhang]{ghadimi2016mini}
S.~Ghadimi, G.~Lan, and H.~Zhang.
\newblock Mini-batch stochastic approximation methods for nonconvex stochastic composite optimization.
\newblock \emph{Mathematical Programming}, 155\penalty0 (1):\penalty0 267--305, 2016.

\bibitem[Grudzien et~al.(2022)Grudzien, De~Witt, and Foerster]{grudzien2022mirror}
J.~Grudzien, C.~A.~S. De~Witt, and J.~Foerster.
\newblock Mirror learning: A unifying framework of policy optimisation.
\newblock In \emph{International Conference on Machine Learning}, pages 7825--7844. PMLR, 2022.

\bibitem[Hiriart-Urruty and Lemar{\'e}chal(2004)]{hiriart2004fundamentals}
J.-B. Hiriart-Urruty and C.~Lemar{\'e}chal.
\newblock \emph{Fundamentals of convex analysis}.
\newblock Springer Science \& Business Media, 2004.

\bibitem[Johnson et~al.(2023)Johnson, Pike-Burke, and Rebeschini]{johnson2023optimal}
E.~Johnson, C.~Pike-Burke, and P.~Rebeschini.
\newblock Optimal convergence rate for exact policy mirror descent in discounted markov decision processes.
\newblock \emph{Advances in Neural Information Processing Systems}, 36:\penalty0 76496--76524, 2023.

\bibitem[Ju and Lan(2022)]{ju2022policy}
C.~Ju and G.~Lan.
\newblock Policy optimization over general state and action spaces.
\newblock \emph{arXiv preprint arXiv:2211.16715}, 2022.

\bibitem[Kakade and Langford(2002)]{kakade2002approximately}
S.~Kakade and J.~Langford.
\newblock Approximately optimal approximate reinforcement learning.
\newblock In \emph{Proceedings of the Nineteenth International Conference on Machine Learning}, pages 267--274, 2002.

\bibitem[Kakade(2001)]{kakade2001natural}
S.~M. Kakade.
\newblock A natural policy gradient.
\newblock \emph{Advances in neural information processing systems}, 14, 2001.

\bibitem[Khodadadian et~al.(2021)Khodadadian, Jhunjhunwala, Varma, and Maguluri]{khodadadian2021linear}
S.~Khodadadian, P.~R. Jhunjhunwala, S.~M. Varma, and S.~T. Maguluri.
\newblock On the linear convergence of natural policy gradient algorithm.
\newblock In \emph{2021 60th IEEE Conference on Decision and Control (CDC)}, pages 3794--3799. IEEE, 2021.

\bibitem[Khodadadian et~al.(2022)Khodadadian, Jhunjhunwala, Varma, and Maguluri]{khodadadian2022linear}
S.~Khodadadian, P.~R. Jhunjhunwala, S.~M. Varma, and S.~T. Maguluri.
\newblock On linear and super-linear convergence of natural policy gradient algorithm.
\newblock \emph{Systems \& Control Letters}, 164:\penalty0 105214, 2022.

\bibitem[Lan(2023)]{lan2023policy}
G.~Lan.
\newblock Policy mirror descent for reinforcement learning: Linear convergence, new sampling complexity, and generalized problem classes.
\newblock \emph{Mathematical programming}, 198\penalty0 (1):\penalty0 1059--1106, 2023.

\bibitem[Lillicrap(2015)]{lillicrap2015continuous}
T.~Lillicrap.
\newblock Continuous control with deep reinforcement learning.
\newblock \emph{arXiv preprint arXiv:1509.02971}, 2015.

\bibitem[Liu et~al.(2019)Liu, Cai, Yang, and Wang]{liu2019neural}
B.~Liu, Q.~Cai, Z.~Yang, and Z.~Wang.
\newblock Neural trust region/proximal policy optimization attains globally optimal policy.
\newblock \emph{Advances in neural information processing systems}, 32, 2019.

\bibitem[Liu et~al.(2020)Liu, Zhang, Basar, and Yin]{liu2020improved}
Y.~Liu, K.~Zhang, T.~Basar, and W.~Yin.
\newblock An improved analysis of (variance-reduced) policy gradient and natural policy gradient methods.
\newblock \emph{Advances in Neural Information Processing Systems}, 33:\penalty0 7624--7636, 2020.

\bibitem[Lu et~al.(2018)Lu, Freund, and Nesterov]{lu2018relatively}
H.~Lu, R.~M. Freund, and Y.~Nesterov.
\newblock Relatively smooth convex optimization by first-order methods, and applications.
\newblock \emph{SIAM Journal on Optimization}, 28\penalty0 (1):\penalty0 333--354, 2018.

\bibitem[McSherry and Talwar(2007)]{mcsherry2007mechanism}
F.~McSherry and K.~Talwar.
\newblock Mechanism design via differential privacy.
\newblock In \emph{48th Annual IEEE Symposium on Foundations of Computer Science (FOCS'07)}, pages 94--103. IEEE, 2007.

\bibitem[Mei et~al.(2020)Mei, Xiao, Szepesvari, and Schuurmans]{mei2020global}
J.~Mei, C.~Xiao, C.~Szepesvari, and D.~Schuurmans.
\newblock On the global convergence rates of softmax policy gradient methods.
\newblock In \emph{International conference on machine learning}, pages 6820--6829. PMLR, 2020.

\bibitem[Mei et~al.(2021)Mei, Gao, Dai, Szepesvari, and Schuurmans]{mei2021leveraging}
J.~Mei, Y.~Gao, B.~Dai, C.~Szepesvari, and D.~Schuurmans.
\newblock Leveraging non-uniformity in first-order non-convex optimization.
\newblock In \emph{International Conference on Machine Learning}, pages 7555--7564. PMLR, 2021.

\bibitem[Mu and Klabjan(2024)]{mu2024second}
S.~Mu and D.~Klabjan.
\newblock On the second-order convergence of biased policy gradient algorithms.
\newblock In \emph{Forty-first International Conference on Machine Learning}, 2024.

\bibitem[Munos and Szepesv{\'a}ri(2008)]{munos2008finite}
R.~Munos and C.~Szepesv{\'a}ri.
\newblock Finite-time bounds for fitted value iteration.
\newblock \emph{Journal of Machine Learning Research}, 9\penalty0 (5), 2008.

\bibitem[Nemirovskij and Yudin(1983)]{nemirovskij1983problem}
A.~S. Nemirovskij and D.~B. Yudin.
\newblock Problem complexity and method efficiency in optimization, 1983.

\bibitem[Nesterov(2013)]{nesterov2013gradient}
Y.~Nesterov.
\newblock Gradient methods for minimizing composite functions.
\newblock \emph{Mathematical programming}, 140\penalty0 (1):\penalty0 125--161, 2013.

\bibitem[Peters and Schaal(2006)]{peters2006policy}
J.~Peters and S.~Schaal.
\newblock Policy gradient methods for robotics.
\newblock In \emph{2006 IEEE/RSJ international conference on intelligent robots and systems}, pages 2219--2225. IEEE, 2006.

\bibitem[Peters and Schaal(2008)]{peters2008reinforcement}
J.~Peters and S.~Schaal.
\newblock Reinforcement learning of motor skills with policy gradients.
\newblock \emph{Neural networks}, 21\penalty0 (4):\penalty0 682--697, 2008.

\bibitem[Rockafellar(1976)]{rockafellar1976monotone}
R.~T. Rockafellar.
\newblock Monotone operators and the proximal point algorithm.
\newblock \emph{SIAM journal on control and optimization}, 14\penalty0 (5):\penalty0 877--898, 1976.

\bibitem[Schulman et~al.(2015)Schulman, Levine, Abbeel, Jordan, and Moritz]{schulman2015trust}
J.~Schulman, S.~Levine, P.~Abbeel, M.~Jordan, and P.~Moritz.
\newblock Trust region policy optimization.
\newblock In \emph{International conference on machine learning}, pages 1889--1897. PMLR, 2015.

\bibitem[Schulman et~al.(2017)Schulman, Wolski, Dhariwal, Radford, and Klimov]{schulman2017proximal}
J.~Schulman, F.~Wolski, P.~Dhariwal, A.~Radford, and O.~Klimov.
\newblock Proximal policy optimization algorithms.
\newblock \emph{arXiv preprint arXiv:1707.06347}, 2017.

\bibitem[Shani et~al.(2020)Shani, Efroni, and Mannor]{shani2020adaptive}
L.~Shani, Y.~Efroni, and S.~Mannor.
\newblock Adaptive trust region policy optimization: Global convergence and faster rates for regularized mdps.
\newblock In \emph{Proceedings of the AAAI Conference on Artificial Intelligence}, pages 5668--5675, 2020.

\bibitem[Sutton et~al.(1999)Sutton, McAllester, Singh, and Mansour]{sutton1999policy}
R.~S. Sutton, D.~McAllester, S.~Singh, and Y.~Mansour.
\newblock Policy gradient methods for reinforcement learning with function approximation.
\newblock \emph{Advances in neural information processing systems}, 12, 1999.

\bibitem[Teboulle(2018)]{teboulle2018simplified}
M.~Teboulle.
\newblock A simplified view of first order methods for optimization.
\newblock \emph{Mathematical Programming}, 170\penalty0 (1):\penalty0 67--96, 2018.

\bibitem[Tomar et~al.(2020)Tomar, Shani, Efroni, and Ghavamzadeh]{tomar2020mirror}
M.~Tomar, L.~Shani, Y.~Efroni, and M.~Ghavamzadeh.
\newblock Mirror descent policy optimization.
\newblock \emph{arXiv preprint arXiv:2005.09814}, 2020.

\bibitem[Tseng(2010)]{tseng2010approximation}
P.~Tseng.
\newblock Approximation accuracy, gradient methods, and error bound for structured convex optimization.
\newblock \emph{Mathematical Programming}, 125\penalty0 (2):\penalty0 263--295, 2010.

\bibitem[Vaswani et~al.(2022)Vaswani, Bachem, Totaro, M\"uller, Garg, Geist, Machado, Samuel~Castro, and Le~Roux]{vaswani22ageneral}
S.~Vaswani, O.~Bachem, S.~Totaro, R.~M\"uller, S.~Garg, M.~Geist, M.~C. Machado, P.~Samuel~Castro, and N.~Le~Roux.
\newblock A general class of surrogate functions for stable and efficient reinforcement learning.
\newblock In \emph{Proceedings of The 25th International Conference on Artificial Intelligence and Statistics}, 2022.

\bibitem[Wang et~al.(2020)Wang, Cai, Yang, and Wang]{wang2020neural}
L.~Wang, Q.~Cai, Z.~Yang, and Z.~Wang.
\newblock Neural policy gradient methods: Global optimality and rates of convergence.
\newblock In \emph{International Conference on Learning Representations}, 2020.

\bibitem[Xiao(2022)]{xiao2022convergence}
L.~Xiao.
\newblock On the convergence rates of policy gradient methods.
\newblock \emph{Journal of Machine Learning Research}, 23\penalty0 (282):\penalty0 1--36, 2022.

\bibitem[Xiong et~al.(2024)Xiong, Fazel, and Xiao]{xiong2024dual}
Z.~Xiong, M.~Fazel, and L.~Xiao.
\newblock Dual approximation policy optimization.
\newblock \emph{arXiv preprint arXiv:2410.01249}, 2024.

\bibitem[Yuan et~al.(2022)Yuan, Gower, and Lazaric]{yuan2022general}
R.~Yuan, R.~M. Gower, and A.~Lazaric.
\newblock A general sample complexity analysis of vanilla policy gradient.
\newblock In \emph{International Conference on Artificial Intelligence and Statistics}, pages 3332--3380. PMLR, 2022.

\bibitem[Yuan et~al.(2023)Yuan, Du, Gower, Lazaric, and Xiao]{yuan2023loglinear}
R.~Yuan, S.~S. Du, R.~M. Gower, A.~Lazaric, and L.~Xiao.
\newblock Linear convergence of natural policy gradient methods with log-linear policies.
\newblock In \emph{The Eleventh International Conference on Learning Representations, {ICLR} 2023, Kigali, Rwanda, May 1-5, 2023}. OpenReview.net, 2023.
\newblock URL \url{https://openreview.net/forum?id=-z9hdsyUwVQ}.

\bibitem[Zanette(2023)]{zanette2023realizability}
A.~Zanette.
\newblock When is realizability sufficient for off-policy reinforcement learning?
\newblock In \emph{International Conference on Machine Learning}, pages 40637--40668. PMLR, 2023.

\bibitem[Zanette et~al.(2020)Zanette, Lazaric, Kochenderfer, and Brunskill]{zanette2020learning}
A.~Zanette, A.~Lazaric, M.~Kochenderfer, and E.~Brunskill.
\newblock Learning near optimal policies with low inherent bellman error.
\newblock In \emph{International Conference on Machine Learning}, pages 10978--10989. PMLR, 2020.

\bibitem[Zhan et~al.(2023)Zhan, Cen, Huang, Chen, Lee, and Chi]{zhan2023policy}
W.~Zhan, S.~Cen, B.~Huang, Y.~Chen, J.~D. Lee, and Y.~Chi.
\newblock Policy mirror descent for regularized reinforcement learning: A generalized framework with linear convergence.
\newblock \emph{SIAM Journal on Optimization}, 33\penalty0 (2):\penalty0 1061--1091, 2023.

\bibitem[Zhang et~al.(2020)Zhang, Koppel, Zhu, and Basar]{zhang2020global}
K.~Zhang, A.~Koppel, H.~Zhu, and T.~Basar.
\newblock Global convergence of policy gradient methods to (almost) locally optimal policies.
\newblock \emph{SIAM Journal on Control and Optimization}, 58\penalty0 (6):\penalty0 3586--3612, 2020.

\bibitem[Zhang and He(2018)]{zhang2018convergence}
S.~Zhang and N.~He.
\newblock On the convergence rate of stochastic mirror descent for nonsmooth nonconvex optimization.
\newblock \emph{arXiv preprint arXiv:1806.04781}, 2018.

\end{thebibliography}
\appendix


\section{Variational Gradient Dominance and Closure Conditions}
\label{sec:vgd_discussion}
In this section, we include detailed discussions regarding the VGD and closure conditions.
In \cref{sec:closure_implies_vgd}, we demonstrate closure $\implies$ VGD; 
In \cref{sec:vgd_not_imply_closure}, we show that VGD $\not\Rightarrow$ closure --- we present a simple example where VGD holds, but closure doesn't and furthermore that bounds of prior works fail to capture convergence of PMD;
Finally, in \cref{sec:vgd_discussion_additional}, we conclude with several general remarks.
\subsection{Closure implies VGD}
\label{sec:closure_implies_vgd}
In this section, we provide formal proofs that closure conditions employed by prior works imply the VGD condition.
Throughout this section, in favor of a simpler comparison, we assume the critic and actor errors are zero, i.e., all algorithms have access to exact action-value functions, and $\epsactM=0$. We introduce the following, slightly extended version of the VGD condition.

\begin{definition}\label{def:vgd_mdp_ext}
	We say a policy class $\Pi$ satisfies $(C_\star, \epsvgd; v^\star)$-VGD if for all $\pi\in \Pi$:
	\begin{align*}
	        V(\pi) - v^\star
	        \leq 
	        C_\star \max_{\tilde \pi \in \Pi}\abr{\nabla V(\pi), \pi - \tilde \pi} + \epsvgd.
	    \end{align*}
\end{definition}

The above extension of the VGD assumption enables a clearer comparison with prior works. As closure implies (approximate) realizability, prior works obtain bounds w.r.t.~the optimal (potentially out-of-class) value function $V^\star=\min_{\pi \in \Delta(\cA)^\cS} V(\pi)$. 
	Our original VGD condition \cref{def:vgd_mdp} is stated with the reasonable $v^\star=V^\star(\Pi)$ choice, however our bounds hold just the same under the assumption VGD holds for other $v^\star$ (such as $v^\star=V^\star$). 

As we show next, both \citet{yuan2023loglinear}\footnote{We refer to their results based on bounded approximation error.}(see \cref{lem:clvgd_log_linear}) and \citet{alfano2023novel} (see \cref{lem:clvgd_generic}) adopt assumptions that imply their policy classes satisfy \cref{def:vgd_mdp_ext} with suitable parameters $C_\star, \epsvgd$ and $v^\star=V^\star$. Notably, the error floors in their convergence results are indeed precisely (up to constant factors) $\epsvgd$. 
We note the implication we establish is not ``perfect'', to make the argument we need slight variations of the original algorithm dependent conditions --- our goal here is to highlight the strong relation between the two assumptions.
Before proceeding, we specifically note the following:
\begin{itemize}
	\item To simplify presentation, we consider closure assumptions (bounded approximation error, concentrability, and distribution mismatch) globally, rather than on the specific iterates selected by the algorithm. However, the same arguments can be made iterate specific, which would lead to VGD conditions on the specific iterates, which is indeed all that is required by our analyses.
	\item The concentrability assumptions employed by \citet{yuan2023loglinear,alfano2023novel} relate to the current policy $\pi^k$ and the next one $\pi^{k+1}$. The direct global extension of this condition would concern a policy $\pi$ and a policy $\pi^+$ selected by a step of the algorithm with the given step size and regularizer.
		Our proof requires $\pi^+$ to be selected differently (e.g., with a different step size choice), which leads to a concentrability assumption that relates to a different $\pi^+$ than the original ones. In this sense, the assumption we make here is a different one, but still qualitatively similar. Again, to simplify presentation we assume stronger concentrability in the lemma statements where $\pi^+$ may be arbitrary, but this can be relaxed as explained above. 
        In addition, our concentrability requires the sampling distribution $v^k$ to support \emph{the current} occupancy measure $\mu^k$ rather than the next one $\mu^{k+1}$. This may actually be considered a weaker assumption than the original one, as the next policy is only determined after performing the step that uses $v^k$. Further, we may always simply select $v^k=\mu^k\circ \pi^k$ to obtain optimal support for $\mu^k$.
	
        \item In \cref{lem:clvgd_generic}, we prove that when (the natural extension of) closure assumptions of \citet{alfano2023novel} hold for Euclidean regularization, VGD holds as well. The claim can be extended to other regularizers with the price of additional regularity assumptions. Regardless, the bounded approximation error assumption of \citet{alfano2023novel} may alternatively be interpreted as a bound on the statistical error, in which case the policy class operated over is $(\nu_\star, 0; V^\star)$-VGD; we provide further details in \cref{sec:closure_non_convex}.
        
        \item The work of \citet{bhandari2024global} demonstrated that closure to policy improvement implies VGD, and further observed there is also a connection between bounded approximation error and VGD (Lemma 16, Appendix B in their work).
		The arguments we give below may be considered a generalization of those in \citet{bhandari2024global}, strengthening  the connection between closure and VGD.

\end{itemize}
\begin{lemma}[Generic closure $\implies$ VGD]
\label{lem:closure_implies_VGD_main}
Let $\Pi$ be a policy class, $\pi \in \Pi$ a policy, and $v\in \Delta(\cS\times \cA)$ a state-action probability measure. Suppose there exists $\pi^+\in \Pi$ such that:
\begin{align*}
	\E_{s\sim \mu^\pi}\abr{\widehat Q_s^\pi, \pi^+_s}
	&\leq \E_{s\sim \mu^\pi}\min_{a} \widehat Q_{s, a}^\pi + \epsgreedy,
	\\
	\text{where }
	\E_{s, a\sim v}\sbr{\br{\widehat Q^\pi_{s, a} - Q^\pi_{s, a}}^2}
	&\leq \epsapproxM,
\end{align*}
and further:
\begin{align*}
    \E_{s,a \sim v}\sbr{\br{\frac{\tilde \mu(s)\tilde \pi_{s, a}}{v(s, a)}}^2}
	\leq C_v,
	\tag{$v$-concentrability}
\end{align*}
For $\tilde \pi\in\cbr{\pi, \pi^+, \pi^\star}, \tilde \mu\in \cbr{\mu^\pi, \mu^\star}$, where $\pi^\star=\argmin_{\pi\in \Delta(\cA)^\cS}V(\pi), \mu^\star=\mu^{\pi^\star}$.
Then, it holds that
$$
	V(\pi) - V^\star
	\leq 
	\norm{\frac{\mu^\star}{\mu^\pi}}_\infty
		\max_{\tilde \pi\in \Pi}\abr{\nabla V(\pi), \pi - \tilde \pi}
         + H\norm{\frac{\mu^\star}{\mu^\pi}}_\infty\br{\epsgreedy + 4\sqrt{C_v \epsapproxM}}.
$$
\end{lemma}
\begin{proof}
We first establish bounds on approximation error terms, then proceed to leverage the approximate greedification assumption to establish VGD.
\paragraph{Approximation error.}
    For any policy $\tilde \pi$ and state-occupancy $\tilde \mu$, , we have:
$$
\E_{s\sim \tilde \mu}\abr{Q^\pi_s - \widehat Q^\pi_s, \tilde \pi_s}
=
\abr{Q^\pi - \widehat Q^\pi, \tilde \mu\circ \tilde \pi}
\leq
\norm{Q^\pi - \widehat Q^\pi}_{L^2(v)} 
\norm{\tilde \mu \circ \tilde \pi}_{L^2(v)}^*
\leq 
\sqrt{\epsapproxM}\norm{\tilde \mu \circ \tilde \pi}_{L^2(v)}^*,
$$
where the last inequality is by our assumption.
Further, by $v$-concentrability,
$$
\norm{\tilde \mu \circ \tilde \pi}_{L^2(v)}^*
=
\sqrt{\E_{s, a \sim v}
\br{\frac{\tilde \mu(s)\tilde \pi_{s, a}}{v(s, a)}}^2}
\leq \sqrt{C_v},
$$
holds for $\tilde \pi\in\cbr{\pi, \pi^+, \pi^\star}, \tilde \mu\in \cbr{\mu^\pi, \mu^\star}$. Now, for such $\tilde \mu, \tilde \pi$, we have:
\begin{align*}
	\av{\E_{s\sim \tilde \mu}\abr{Q^\pi_s - \widehat Q^\pi_s, \pi_s - \tilde \pi_s}}
	\leq
	\av{\E_{s\sim \tilde \mu}\abr{Q^\pi_s - \widehat Q^\pi_s, \pi_s}}
	+ \av{\E_{s\sim \tilde \mu}\abr{Q^\pi_s - \widehat Q^\pi_s, \tilde \pi_s}}
	\leq
	2\sqrt{C_v\epsapproxM},
\end{align*}
therefore,
\begin{align*}
	\av{\E_{s\sim \mu^\pi}\abr{Q^\pi_s - \widehat Q^\pi_s, \pi_s - \pi_s^+}}
	&\leq
	2\sqrt{C_v\epsapproxM},
	\\
	\av{\E_{s\sim \mu^\star}\abr{Q^\pi_s - \widehat Q^\pi_s, \pi_s - \pi_s^\star}}
	&\leq
	2\sqrt{C_v\epsapproxM}.
\end{align*}
\paragraph{Greedification.}
Observe,
\begin{align*}
	\E_{s\sim \mu^\pi}\abr{Q^\pi_s, \pi_s - \pi^+_s}
	&= \E_{s\sim \mu^\pi}\abr{\widehat Q^\pi_s, \pi_s - \pi^+_s}
	+ \E_{s\sim \mu^\pi}\abr{Q^\pi_s - \widehat Q^\pi_s, \pi_s - \pi^+_s}
	\\
	&\geq
	\E_{s\sim \mu^\pi}\max_p\abr{\widehat Q^\pi_s, \pi_s - p}
	-\epsgreedy
	- 2\sqrt{C_v \epsapproxM}
    \\
    \implies
    \E_{s\sim \mu^\pi}\max_p\abr{\widehat Q^\pi_s, \pi_s - p}
    &\leq 
    \E_{s\sim \mu^\pi}\abr{Q^\pi_s, \pi_s - \pi^+_s}
    +\epsgreedy
	+ 2\sqrt{C_v \epsapproxM}
    .
\end{align*}
Therefore, by \cref{lem:value_diff} (value difference),
\begin{align*}
	\frac{1}{H}\br{V(\pi) - V^\star}
	&= \E_{s\sim \mu^\star}\abr{Q^\pi_s, \pi_s - \pi^\star_s}
	\\
	&= \E_{s\sim \mu^\star}\abr{\widehat Q^\pi_s, \pi_s - \pi^\star_s}
	+ \E_{s\sim \mu^\star}\abr{Q^\pi_s - \widehat Q^\pi_s, \pi_s - \pi^\star_s}
	\\
	&\leq 
	\E_{s\sim \mu^\star}\max_{p\in\Delta(\cA)}\abr{\widehat Q^\pi_s, \pi_s - p}
	+ 2\sqrt{C_v \epsapproxM}
	\\
	&\leq 
	\norm{\frac{\mu^\star}{\mu^\pi}}_\infty\E_{s\sim \mu^\pi}\max_{p\in\Delta(\cA)}\abr{\widehat Q^\pi_s, \pi_s - p}
	+ 2\sqrt{C_v \epsapproxM}
	\\
	&\leq 
	\norm{\frac{\mu^\star}{\mu^\pi}}_\infty
		\E_{s\sim \mu^\pi}\abr{Q^\pi_s, \pi_s - \pi^+_s}
         + \norm{\frac{\mu^\star}{\mu^\pi}}_\infty\br{\epsgreedy + 2\sqrt{C_v \epsapproxM}}
         + 2\sqrt{C_v \epsapproxM}
	\\
	&\leq 
	\norm{\frac{\mu^\star}{\mu^\pi}}_\infty
		\E_{s\sim \mu^\pi}\abr{Q^\pi_s, \pi_s - \pi^+_s}
         + \norm{\frac{\mu^\star}{\mu^\pi}}_\infty\br{\epsgreedy + 4\sqrt{C_v \epsapproxM}}
    \\
	&= 
	\frac1H\norm{\frac{\mu^\star}{\mu^\pi}}_\infty
		\abr{\nabla V(\pi), \pi - \pi^+}
         + \norm{\frac{\mu^\star}{\mu^\pi}}_\infty\br{\epsgreedy + 4\sqrt{C_v \epsapproxM}}
    \\
	&\leq 
	\frac1H\norm{\frac{\mu^\star}{\mu^\pi}}_\infty
		\max_{\tilde \pi\in \Pi}\abr{\nabla V(\pi), \pi - \tilde \pi}
         + \norm{\frac{\mu^\star}{\mu^\pi}}_\infty\br{\epsgreedy + 4\sqrt{C_v \epsapproxM}}
    ,
\end{align*}
which completes the proof after multiplying by $H$.
\end{proof}

\begin{lemma}[Log-linear dual closure $\implies$ VGD]
\label{lem:clvgd_log_linear}
	Let $\cbr{\phi_{s, a}}_{s\in\cS, a\in\cA}\subseteq\R^d$ be state-action feature vectors,
	and let $\Pi$ be the log-linear policy class $\Pi = \cbr{\pi(\theta) \mid \theta\in \R^d}$, where
	\begin{align*}
		\pi_{s, a}(\theta) 
		\eqq 
		\frac{\exp(\phi_{s, a}\T \theta)}{\sum_{a'\in \cA}\exp(\phi_{s, a'}\T \theta)}.
	\end{align*}
	Assume further that for all $\pi\in \Pi$
	it holds that
	\begin{align*}
		\min_w\E_{s, a \sim (\mu^\pi \circ \pi)}\sbr{\br{
			w\T \phi_{s, a} - Q^\pi_{s, a}
		}^2}\leq \epsapproxM,
	\end{align*}
	and, 
	\begin{align*}
		\norm{\frac{\mu^\star}{\mu^\pi}}_\infty\leq \nu_\star,
	\end{align*}
	and,
	\begin{align*}
		\E_{s, a \sim (\mu^\pi \circ \pi)}\sbr{
			\br{\frac{h^\pi_{s,a}}{\mu^k(s)\pi_{s, a}}}^2
		}\leq C_{\nu},
	\end{align*}
	where $h^\pi$ represents $\tilde \mu \circ \tilde \pi$ 
	for all
	$\tilde \pi\in\Pi, \tilde \mu\in \cbr{\mu^\pi, \mu^\star}$, 
	and we denote $\pi^\star=\argmin_{\pi\in \Delta(\cA)^\cS}V(\pi), \mu^\star=\mu^{\pi^\star}$.
	Then $\Pi$ satisfies $(\nu_\star, 5\nu_\star H \sqrt{C_\nu \epsapproxM}; V^\star)$-VGD (\cref{def:vgd_mdp_ext}).
\end{lemma}
\begin{proof}
	Let $\pi\in \Pi$, and denote $\widehat Q^\pi_{s, a} = \phi_{s, a}\T w_\star^\pi$ where 
	\begin{align*}
		w_\star^\pi \eqq \argmin_w\E_{s, a \sim (\mu^\pi \circ \pi)}\sbr{\br{
			w\T \phi_{s, a} - Q^\pi_{s, a}
		}^2}.
	\end{align*}
	By \cref{lem:softmax_approx}, the policy $\pi^+ \eqq \pi(\theta^+)\in \Pi$ defined by
	$\theta^+ \eqq (\log(d)/\epsgreedy) w_\star^\pi$ satisfies
	\begin{align*}
		\forall s: \abr{\widehat Q_s^\pi, \pi^+_s}
		&\leq \min_{a} \widehat Q_{s, a}^\pi + \epsgreedy.
	\end{align*}
	Now, by the above and our assumptions, we are in the position to apply \cref{lem:closure_implies_VGD_main}, which immediately implies the desired for $\epsgreedy=\sqrt{C_\nu \epsapproxM}$.
\end{proof}

\begin{lemma}[\cite{mcsherry2007mechanism, epasto2020optimal}]
\label{lem:softmax_approx}
	Let $x_1, \ldots, x_d \in \R$.
	 Then if $\tau \geq (\log d) / \delta$, it holds that
	\begin{align*}
		\frac{\sum_i e^{-\tau x_i} x_i}
		{\sum_i e^{-\tau x_i}} 
		\leq \min_i x_i + \delta.
	\end{align*}
\end{lemma}
\begin{proof}
	We have
	\begin{align*}
		\frac{\sum_i e^{-\tau x_i} x_i}
		{\sum_i e^{-\tau x_i}} - \min_i x_i 
		= 
		\max_i \cbr{-x_i}
		-
		\frac{\sum_i e^{-\tau x_i} (-x_i)}{\sum_i e^{-\tau x_i}}
	\end{align*}
	The result now follows from the original statement, which says that for any $z_1, \ldots, z_n \in \R$,
	\begin{align*}
		\max_i z_i - \frac{\sum_i e^{\tau z_i} z_i}{\sum_i e^{\tau z_i}}
		\leq \delta.
            \qquad \qedhere
	\end{align*}
\end{proof}

Next, we provide a proof for closure conditions of \citet{alfano2023novel} in the case of a regularizer with a bounded Bregman divergence, which simplifies some technical issues and is sufficient for the Euclidean case. The implication can be shown to hold more generally subject to some additional regularity conditions on the policy class. We note that such a general version of the lemma would in particular imply \cref{lem:clvgd_log_linear}, thus rendering the above proof redundant. However, we opted for an independent proof of \cref{lem:clvgd_log_linear} to avoid the additional regularity assumptions.
\begin{lemma}[Generic dual closure $\implies$ VGD]
\label{lem:clvgd_generic}
Let $\Pi\subset \Delta(\cA)^\cS$ be a policy class, and $R\colon \R^A \to \R$ be an action regularizer.
For any policy $\pi$ let $\eta > 0$ be a chosen step size and $v$
be a chosen state-action probability measure. Define
\begin{align*}
	f^+ \eqq f^{+}(\pi, \eta) &\eqq \argmin_{f\in \cF}
	\norm{f - \br{\eta^{-1}\nabla R(\pi) -  Q^\pi}}_{L^2(v)}^2
	\\
	\pi^+ \eqq \pi^+(\pi, \eta) &\eqq P_R(\eta f^+),
\end{align*}
where $P_R(\eta f)_{s} \eqq \Pi_{\Delta(\cA)}^R(\nabla R^*(\eta f_s))$.
Assume that:
\begin{align*}
  	\norm{f^+  - \br{\eta^{-1}\nabla R(\pi) - Q^\pi}}_{L^2(v)}^2	\leq \epsapproxM,
	\tag{A1}	  
\end{align*}
and for $\tilde \pi\in\cbr{\pi, \pi^+, \pi^\star}, \tilde \mu\in \cbr{\mu^\pi, \mu^\star}$:
\begin{align*}
    \E_{s,a \sim v}\sbr{\br{\frac{\tilde \mu(s)\tilde \pi_{s, a}}{v(s, a)}}^2}
	\leq C_v,
	\tag{A2}
\end{align*}
and finally,
\begin{align*}
    \sup_s \frac{\mu^\star(s)}{\mu^\pi(s)}
    \leq \nu_\star.
	\tag{A3}
\end{align*}
Then, if $R$ has a bounded Bregman divergence, $B\geq \max_{p, q\in \Delta(\cA)}B_R(p, q)$, and the above holds for any $\eta$, it holds that 
$\Pi$ satifies $\br{\nu_\star, 5H\nu_\star\sqrt{C_v \epsapproxM}; V^\star}$-VGD (\cref{def:vgd_mdp_ext}).
\end{lemma}
\begin{proof}
Fix $\pi \in \Pi$, and define
\begin{align*}
	\widehat Q^\pi
	&\eqq \eta^{-1}\nabla R(\pi) - f^+
	\\
	\implies
	f^+ 
	&= \eta^{-1}\nabla R(\pi) - \widehat Q^\pi,
\end{align*}
which implies that:
\begin{align*}
	\forall s, \;
	\pi^+_s &= \argmin_{p\in \Delta(\cA)}
		\abr{\widehat Q^\pi_s, p} + \frac1\eta B_R(p, \pi_s)
\end{align*}
We have:
$$
\norm{Q^\pi - \widehat Q^\pi}_{L^2(v)} ^2
=
\norm{f^+ -\br{\eta^{-1}\nabla R(\pi)-Q^\pi}}_{L^2(v)}^2
\leq \epsapproxM,
$$
and for $\eta=B/\epsgreedy$, by \cref{lem:omd_to_greedy}:
$$
	\forall s,\; \abr{\widehat Q^\pi_s, \pi^+_s}
	\leq 
	\min_a \widehat Q^\pi_{s,a} + \epsgreedy.
$$
Choosing $\epsgreedy = \sqrt{C_v\epsapproxM}$, the result follows by \cref{lem:closure_implies_VGD_main}.
\end{proof}

\begin{lemma}\label{lem:omd_to_greedy}
  Let $\epsilon>0$, $R\colon \R^A \to \R$ be a convex regularizer 
  with bounded Bregman divergence $B\geq \max_{p, q\in \Delta(\cA)} B_R(p, q)$,
and $g\in \R^A$ be a linear objective, with $a^\star=\argmin_a g_a$.
Then, for any  $x\in \Delta(A)$, for $\eta\geq B/\epsilon$, we have:
$$
	x^+ = \argmin_{z\in \Delta(A)} \cbr{\abr{g,z} 
            + \frac1\eta B_R(z, x)}
	\implies
	g(x^+) \leq g_{a^\star} + \epsilon.
$$
\end{lemma}
\begin{proof}
    By optimality of $x^+$:
    \begin{align*}
        g(x^+) 
        &\leq 
        g(e_{a^\star}) + \frac1\eta B_R(e_{a^\star}, x)
        - \frac1\eta B_R(x^+, x)
        \\
        &\leq 
        g_{a^\star}
        + B/\eta
        \\
        &= 
        g_{a^\star}
        +\epsilon,
    \end{align*}
    and the result follows.
\end{proof}

\subsection{Closure without convexity}
\label{sec:closure_non_convex}
In this section, we explain how the approximate closure conditions of \citet{alfano2023novel}   eliminate the need for convexity of $\Pi$ in our analysis.
Roughly speaking, closure conditions imply approximate optimality conditions hold for the PMD iterates w.r.t.~the complete policy class. And, in our analysis, we obtain guarantees w.r.t.~the policy class the PMD iterates satisfy optimality conditions with respect to, regardless of actual policy class the algorithm operates over.
To make the argument formal, we consider the following assumption, which characterizes the behavior of the algorithm in relation to an ``ambient'' policy class $\widetilde \Pi$.
\begin{assumption}[PMD w.r.t.~ambient $\widetilde \Pi$]\label{assm:comp_main}
	For $\widetilde \Pi$ a policy class, and $\pi^1, \ldots, \pi^{K+1}$ is a sequence of policies, it holds that:
	\begin{enumerate}
		\item $\widetilde \Pi$ is convex.
		\item $\widetilde \Pi$ satisfies $(C_\star, \epsvgd; v^\star)$-VGD on the iterates $\pi^1, \ldots, \pi^{K+1}$:
	    \begin{align*}
	        V(\pi^k) - v^\star
	        \leq 
	        C_\star \max_{\tilde \pi \in \widetilde \Pi}\abr{\nabla V(\pi^k), \pi^k - \tilde \pi} + \epsvgd.
	    \end{align*}
		\item $\pi^1, \ldots, \pi^{K+1}$ satisfy PMD approximate optimality conditions w.r.t.~$\widetilde \Pi'$:
		\begin{align*}
	        \forall \pi \in \widetilde \Pi', 
	        \abr{\nabla \phi_k(\pi^{k+1}), \pi - \pi^{k+1}} \geq -\epsactM,
	    \end{align*}
	    where
	    \begin{aligni*}
	        \phi_k(\pi) \eqq \E_{s\sim \mu^k}\sbr{
	            \abr{Q^{k}_s, \pi_s} 
	            + \frac1\eta B_R(\pi_s, \pi_s^k)}.
	    \end{aligni*}
	\end{enumerate}
\end{assumption}
We note that a PMD algorithm does not necessarily need access to $\widetilde \Pi$ to satisfy \cref{assm:comp_main}. In particular, it may be that the algorithm operates over non-convex $\Pi$, but satisfies \cref{assm:comp_main} with $\widetilde \Pi=\widetilde \Pi' = \Delta(\cA)^\cS$.
Next, we restate our guarantees for the Euclidean case reframed in the context of \cref{assm:comp_main}, and then proceed to demonstrate assumptions of \citet{alfano2023novel} imply \cref{assm:comp_main}.
\begin{theorem*}[Restatement of \cref{thm:pmd_main}; Euclidean case]
	Let $\widetilde \Pi\subset \Delta(\cA)^\cS$ be a policy class
	and suppose $\pi^1, \pi^2, \ldots, \pi^{K+1}$ is a sequence of policies for which \cref{assm:comp_main} holds with $\widetilde \Pi' = \widetilde \Pi^\epsexpl$,
	$R(p)=\frac12\norm{p}_2^2$, and $\eta, \epsexpl$ properly tuned.
	Then, it holds that:
       \begin{align*}
           V(\pi^K) - v^\star=
            \bigO\br{
            \frac{C_\star^2 A^{3/2} H^3}{K^{2/3}}
            + \br{C_\star H + A H^2 K^{1/6}}\sqrt{\epsactM}
            + \epsvgd}
        \end{align*}
\end{theorem*}

To establish the next lemma, we interpret the colure conditions of \citet{alfano2023novel} as perfect closure, where $\epsapproxM$ bounds the actor error, rather than relating to expressivity of the dual policy parametrization.
\begin{lemma}
Suppose that for all $k\in [K]$,
\begin{align*}
	\norm{f^{k+1} - \br{\eta^{-1}\nabla R(\pi^k) -  Q^k}}_{L^2(v^k)}^2
	&\leq \epsapproxM
	\tag{A1},
\end{align*}
and $\pi^{k+1} = P_R(\eta f^{k+1})$
where $P_R(\eta f)_{s} \eqq \Pi_{\Delta(\cA)}^R(\nabla R^*(\eta f_s))$.
Suppose further that for all $k$,
\begin{align*}
    \sup_s \frac{\mu^\star(s)}{\mu^k(s)}
    \leq \nu_\star.
	\tag{A3}
\end{align*}
Then, with the choice of $v^k=\mu^k\circ u$, i.e., $s, a\sim v^k \implies s\sim \mu^k, a\sim \Unif(\cA)$, we have that 
\cref{assm:comp_main} is satisfied with $\widetilde \Pi = \widetilde \Pi'=\Delta(\cA)^\cS$, $v^\star=V^\star$, $C_\star=\nu_\star$, $\epsvgd=0$, and $\epsactM \leq 2\sqrt {A \epsapproxM}$.
\end{lemma}
\begin{proof}
Let $\zeta^{k+1} \eqq f^{k+1} - \br{\eta^{-1}\nabla R(\pi^k) - Q^k}$. Then $\norm{\zeta^{k+1}}_{L^2(v^k)}^2 \leq \epsapproxM$, and
$$
	f^{k+1} = \eta^{-1}\nabla R(\pi^k) - \br{Q^k + \zeta^{k+1}}.
$$
Now by definition of $\pi^{k+1}$, 
\begin{align*}
\pi^{k+1}_s
&= \argmin_{\pi_s \in \Delta(\cA)} \abr{Q^k_s + \zeta^{k+1}_s, \pi_s} + \frac1\eta B_R(\pi_s, \pi^k_s)
\end{align*}
hence, by optimality conditions, for any $\pi\in \Delta(\cA)^\cS$:
\begin{align*}
	\abr{Q^k_s + \zeta^{k+1}_s + \frac1\eta\br{\nabla R(\pi^{k+1}_s) - \nabla R(\pi^k_s)}, \pi_s - \pi^{k+1}_s}
	&\geq 0
	\\
	\iff
	\abr{Q^k_s + \frac1\eta\br{\nabla R(\pi^{k+1}_s) - \nabla R(\pi^k_s)}, \pi_s - \pi^{k+1}_s}
	&\geq \abr{\zeta^{k+1}_s, \pi^{k+1}_s - \pi_s}
\end{align*}
Now, note that
\begin{align*}
	\E_{s\sim \mu^k}\abr{\zeta^{k+1}_s, \pi^{k+1}_s - \pi_s}
	&=
	\abr{\mu^k\circ \zeta^{k+1}, \pi^{k+1} - \pi}
	\\
	&\leq 
	\norm{\mu^k\circ \zeta^{k+1}}_{L^2(\mu^k), 2}^*\norm{\pi^{k+1} - \pi}_{L^2(\mu^k), 2}
	\\
	&\leq 
	2\sqrt{\E_{s\sim \mu^k}\norm{\zeta^{k+1}_s}_2^2}
\end{align*}
Further, by the choice of $v^k = \mu^k\circ u$,
$$
	\E_{s\sim \mu^k}\norm{\zeta^{k+1}_s}_2^2
	=
	\E_{s\sim \mu^k} \sbr{\sum_{a\in \cA} \br{\zeta^{k+1}_{s, a}}^2}
	= A \E_{s\sim \mu^k} \sbr{\sum_{a\in \cA} \frac1A\br{\zeta^{k+1}_{s, a}}^2}
	= A \norm{\zeta^{k+1}}_{L^2(v^k)}^2.
$$
Therefore, for all $\pi\in \Delta(\cA)^\cS$:
$$
	\E_{s\sim \mu^k}\abr{
		Q^k_s + \frac1\eta\br{\nabla R(\pi^{k+1}_s) - \nabla R(\pi^k_s)}, \pi_s - \pi^{k+1}_s
	}\geq -\epsactM
$$
with $\epsactM \leq 2\sqrt {A \epsapproxM}$.	
Finally, the complete class satisfies $(\nu^\star, 0)$-VGD on the $\pi^k$ iterates by \cref{lem:vgd_complete}, with $\nu^\star$ in place of $H\nu_0$ owed to our assumption (A3).
\end{proof}
Finally, we note that we could have traded the dependence on the action set with an additional concentrability assumption. 

\subsection{VGD does not imply closure}
\label{sec:vgd_not_imply_closure}
In this section, we present a sinple example where the VGD condition holds but closure does not, and as a result existing analyses fail to establish convergence of PMD. We note that the fact that VGD does not imply closure is immediate, as closure implies realizability but VGD does not. We go further here to show that the bounds of prior works may indeed become vacuous in setups where VGD holds and closure does not.
We consider the MDP depicted in \cref{fig:simple_mdp} 
 with the log-linear policy class $\Pi$ induced by the state-action feature vectors shown in the diagram. 
 
\begin{figure}[ht!]
    \centering
\begin{tikzpicture}[
    state/.style={circle, draw, minimum size=2cm, node distance=2cm},
    every node/.style={font=\sffamily},
    >={Stealth[round]}
]

\node[state] (S0) at (0, 0) {$S_0$};
\node[state] (S1) at (-2, -3) {$S_1$};
\node[state] (S2) at (2, -3) {$S_2$};

\draw[->, color=orange] (S0) -- 
		node[pos=0, anchor=south, font=\scriptsize, shift={(0.15, -0.1)}] 
		{\tiny $\actA$} 
	(S1) node[midway, above] {};
\draw[->, color=violet] (S0) -- 
		node[pos=0, anchor=south, font=\scriptsize, shift={(-0.1, -0.1)}] 
		{\tiny $\actB$} 
	(S2) node[midway, left] {};

\draw[->, color=violet] (S1) to[bend right] 
	node[pos=0, anchor=north east, font=\scriptsize, shift={(0.2, 0.15)}] 
		{\tiny $\actB$}
	(S0) node[midway, above] {};
\draw[->, color=orange] (S2) to[bend left] 
		node[pos=0, anchor=north west, font=\scriptsize, shift={(-0.2, 0.1)}] 
		{\tiny $\actA$}
	(S0) node[midway, below] {};

\draw[->, very thick, color=orange] (S1) 
	to[bend left] node[midway, left] {1} 
		node[pos=0, anchor=north, font=\scriptsize, shift={(0, 0.05)}] 
		{\tiny $\actA$}
	(S0) node[midway, above] {};

\draw[->, very thick, color=violet] (S2) 
	to[bend right] node[midway, right] {$1$}
		node[pos=0, anchor=north, font=\scriptsize, shift={(0, 0.05)}] 
		{\tiny $\actB$}
	(S0) node[midway, below] {};
\end{tikzpicture}
    \caption{A simple MDP with a convex value landscape. Each action represented by a (feature-vector, edge) pair leads deterministically to the state at the other end of the edge. The two outer bold edges labeled $1$ inflict a cost of $1$, the others have cost $0$.}
    \label{fig:simple_mdp}
\end{figure}
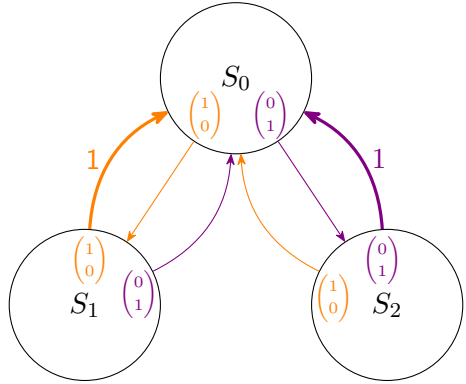

 For simplicity we assume there are no statistical errors in the execution of the algorithm ($\epsstatM=0$). In this example the value landscape is convex (in state-action space) over $\Pi$, and thus $\Pi$ is $(1,0)$-VGD and convergence of PMD follows by our main theorem:
\begin{align*}
        V(\pi^{K}) - &\min_{\pi^\star\in \Pi}V(\pi^\star)
        \xrightarrow[K\to \infty]{} 0.
    \end{align*}
At the same time, results based on closure imply convergence to an error floor that is larger than $H$.
For instance, by Theorem 1 of \citet{yuan2023loglinear} establishes that:
\begin{align}\label{eq:thm_yuan}
        V(\pi^{K}) - &\min_{\pi^\star\in \Pi}V(\pi^\star)
        \lesssim 
        2H\br{1- \frac{1}{\nu_0}}^{K}
        + 2H \nu_0 \sqrt{A C_0 \epsbiasM},
    \end{align}
meaning:
\begin{align*}
        V(\pi^{K}) - &\min_{\pi^\star\in \Pi}V(\pi^\star)
        \xrightarrow[K\to \infty]{}
        2H \nu_0 \sqrt{A C_0 \epsbiasM}
        \geq 10 H,
    \end{align*}
    where $\epsbiasM=\Omega(1)$, $\nu_0 \eqq H\norm{\frac{\mu^\star}{\rho_0}}_\infty$, and $C_0$ is a certain concentrability coefficient larger than $1$. Here, both the transfer error and approximation error are $\Omega(\epsbiasM)$.
A rigorous analysis is given below in \cref{sec:vgd_not_cl_analysis}.
Recent papers such as \citet{alfano2023novel,xiong2024dual} accommodate more general policy parameterizations but still include
the log-linear setup as a special case (see discussion in \citealp{alfano2023novel} and \cref{sec:dual_pc}). The error floor in their results is also larger than $H$ for the example in question for exactly the same reasons; their results depend on the approximation error, which for this example as mentioned behaves the same as the transfer error.

Finally, we note that the example is not realizable and the discussion focuses on best-in class convergence as the objective.
If we were to look for convergence w.r.t.~the true optimal policy, our \cref{thm:pmd_main} establishes convergence to an error floor of $V(\Pi^\star) - V^\star \approx H/2$, while closure based analyses suffer from the same $\geq H$ error floor.
In all that follows, we focus on the transfer error $\epsbiasM$; the argument for the approximation error is the same.

\subsubsection{Analysis}
\label{sec:vgd_not_cl_analysis}
We denote the actions:
\begin{align*}
	u \eqq  \actA,
	\quad 
	b \eqq  \actB,
\end{align*}
and the state-action features, for all $s$:
\begin{align*}
	\phi_{s, 1} \eqq \actA,
	\quad 
	\phi_{s, 2} \eqq \actB,
	\quad 
	\phi_{s} 
	\eqq (\phi_{s, 1}, \phi_{s, 2})
	= (\actA \actB)
	\in \R^{2\times 2}.
\end{align*}
In favor of conciseness, we will let
\begin{align*}
	\phi_{i, \cdot} \eqq \phi_{S_i, \cdot}.
\end{align*}
For $\theta \in \R^2$, we denote the log-linear policy $\pi^\theta_{s}\eqq \sigma(\phi_{s}\T \theta)$, where $\sigma$ is the softmax function:
\begin{align*}
	\sigma(u)_i \eqq \frac{e^{u_i}}{\sum_{j}e^{u_j}}.
\end{align*}
This gives rise to the log-linear policy class:
\begin{align*}
	\Pi \eqq \cbr{\pi^\theta \mid \theta \in \R^2}.
\end{align*}
Since such a policy $\pi^\theta$ in this MDP must select actions independent of the state, we let $\alpha$ denote the probability it chooses $u$ and $1-\alpha$ the probability it chooses $b$; $\alpha\eqq\pi^\theta_{s, u} \implies 1-\alpha = \pi^\theta_{s, b}$.
Now, denote $V_i^\alpha \eqq V_{S_i}\br{\pi^\theta}, Q^\alpha_{i, \cdot}\eqq Q_{S_i, \cdot}^{\pi^\theta}$, and observe that by direct computation:
\begin{align*}
	V_0(\alpha) 
	&= \frac{\gamma}{(1-\gamma)(1+\gamma)}\br{\alpha^2 + (1-\alpha)^2}
	=: \widetilde H\br{\alpha^2 + (1-\alpha)^2}
	\\
	V_1(\alpha) &= \alpha + \gamma \widetilde H\br{\alpha^2 + (1-\alpha)^2}
	\\
	V_2(\alpha) &= (1-\alpha) + \gamma \widetilde H\br{\alpha^2 + (1-\alpha)^2}.
\end{align*}
and,
\begin{alignat*}{2}
	Q^\alpha_{0,u} &= \gamma V_1(\alpha),
	\quad &
	Q^\alpha_{0,b} &= \gamma V_2(\alpha);
	\\
	Q^\alpha_{1,u} &= 1 + \gamma V_0(\alpha),
	\quad &
	Q^\alpha_{1,b} &= \gamma V_0(\alpha);
	\\
	Q^\alpha_{2,u} &= \gamma V_0(\alpha),
	\quad &
	Q^\alpha_{2,b} &= 1 + \gamma V_0(\alpha).
\end{alignat*}
Let $\rho_0(S0)=1-p,\rho_0(S1)=\rho_0(S2)=p/2$ for some $p\in[0,1)$. Then
\begin{align*}
	V(\alpha) 
	&= (1-p+\gamma p)\widetilde H\br{\alpha^2 + (1-\alpha)^2}
	+ p/2.
\end{align*}

\paragraph{The VGD condition holds.}
It is not hard to verify the value function is convex (in state-action space) over this policy class. Indeed, we have
\begin{align*}
	\abr{\nabla_{\pi^\theta} V(\pi^\theta), \pi^{\tilde \theta} - \pi^\theta}
	= \frac{\partial V^\alpha}{\partial \alpha}\br{\tilde \alpha - \alpha},
\end{align*}
and therefore convexity of $V^\alpha$ w.r.t.~$\alpha$ implies convexity in the direct parametrization over $\Pi$.
Hence in particular, $\Pi$ is $(1, 0)$-VGD w.r.t.~the MDP in question. Thus, convergence of PMD follows by \cref{thm:pmd_main}, which in this case guarantees the sub-optimality of $\pi^K$ tends to $0$ as $K$ grows (since there is no error floor).

\paragraph{Closure does not hold, and the error floor in closure based analyses is $\geq H=\frac{1}{1-\gamma}$.}
Let $\mu^\alpha \eqq \mu^{\pi^\theta}$, then
\begin{align*}
	\mu^\alpha(S_0) 
	&= \frac{(1-\gamma)(1-p) + \gamma}{1+\gamma}
	= \frac{1-p + \gamma H}{(1+\gamma) H}
	\\
	\mu^\alpha(S_1) &= (1-\gamma)p + \gamma \alpha \mu^\alpha(S_0)
	\\
	\mu^\alpha(S_2) &= (1-\gamma)p + \gamma (1-\alpha) \mu^\alpha(S_0).
\end{align*}
It is immediate that the optimal in-class policy is given by $\theta^\star\eqq (1, 1), \alpha^\star = 1/2$, and satisfies,
\begin{align*}
	\mu^\star(S_0) = \frac{1-p + \gamma H}{(1+\gamma)H}, \quad 
	\mu^\star(S_1) = \frac{H + p -1}{2(1+\gamma)H}, \quad 
	\mu^\star(S_2) = \frac{H + p -1}{2(1+\gamma)H}.
\end{align*}
Now suppose that $\gamma\geq0.99$ and $p\leq 1/100$, then by direct computation,
\begin{align*}
	\mu^\star(S_0) 
	\approx \frac12, \quad
	\mu^\star(S_1) \approx \frac{1}{4}, \quad
	\mu^\star(S_2) \approx \frac{1}{4},
\end{align*}
where the approximation is correct up to error of $1/100$.
Recall that for a policy $\pi^{(k)}$, in the NPG update step \cite{agarwal2021theory,yuan2023loglinear}
\begin{align*}
	w^{(k)}_\star
	&\eqq \argmin_{w}	\E_{s \sim \mu^k, a\sim \pi^k_s}\sbr{
		\br{\phi_{s, a}\T w - Q^k_{s, a}}^2
	}
\end{align*}
Meanwhile, by definition
\begin{align*}
	\epsilon_{\rm bias}
	&\geq  \E_{s \sim \mu^\star, a\sim {\rm Unif}(\cA)}\sbr{
		\br{\phi_{s, a}\T w_\star^{(k)} - Q^k_{s, a}}^2
	}
	\\
	&\geq \frac12\argmin_{w_1}\E_{s \sim \mu^\star}\sbr{
		\br{w_1 - Q^k_{s, u}}^2
	}
	\\
	&\approx  \frac12\argmin_{w_1}
	\cbr{
		\frac12\br{w_1 - \gamma V_1(\alpha)}^2
		+
		\frac14\br{w_1 - 1 - \gamma V_0(\alpha)}^2
		+
		\frac14\br{w_1 - \gamma V_0(\alpha)}^2
	}
	\\
	&\geq
	\frac18\argmin_{w_1}
	\cbr{
		\br{w_1 - 1 - \gamma V_0(\alpha)}^2
		+
		\br{w_1 - \gamma V_0(\alpha)}^2
	}
	\\
	&= \frac{1}{32}
\end{align*}
Now the bias term in \cref{eq:thm_yuan} is at least as large as
\begin{align*}
	H\norm{\frac{\mu^\star}{\rho_0}}_\infty\sqrt{\epsilon_{\rm bias}}
	\gtrsim H\frac{1}{p}\sqrt{\frac{1}{32}}	
	\geq
	10 H.
\end{align*}

\subsection{Additional Remarks}
\label{sec:vgd_discussion_additional}
In this section we include several additional points for consideration regarding closure and VGD conditions.

\paragraph{On-policy PMD is prone to local optima.}
The necessity of some structural assumption
    (whether VGD or closure) is motivated in the introduction by the fact that policy gradient methods over non-complete policy classes $\Pi \neq \Delta(\cA)^\cS$ are prone to local optima \citep{bhandari2024global}.
     While PMD and vanilla policy gradients are not the same algorithm, the example given in \citet{bhandari2024global} (Example 1) also applies to PMD with Euclidean regularization, as we explain next.
    A vanilla policy gradient update in the direct parametrization case is equivalent to:
    \begin{align*}
        \pi^{k+1} = \argmin_{\pi\in \Pi}\sbr{\E_{s\sim \mu^k}\sbr{\abr{Q_s^k, \pi_s}} + \frac{1}{2\eta}\norm[b]{\pi - \pi^k}^2_2},
    \end{align*}
    which is an ``unweighted regularization'' version of Euclidean PMD.
    While this is equivalent to PMD for $\Pi = \Delta(\cA)^\cS$ (in the error free case), it is indeed not equivalent in general.
	However, Example 1 of \citet{bhandari2024global} indeed also applies to Euclidean PMD because the policy class in question contains only policies $\pi$ such that $\pi_{s, a} = \pi_{s', a}$ for all $s,s',a$. Hence, for any two policies $\pi, \pi^k\in \Pi$,
    $\norm{\pi_s - \pi^k_s}^2_2 = \norm{\pi_{s'} - \pi^k_{s'}}^2_2$ for all $s, s'$, and
    $\norm{\pi - \pi_k}^2_2 = S\E_{s\sim \mu^k}\norm{\pi_s - \pi^k_s}^2_2 = 2\E_{s\sim \mu^k}\norm{\pi_s - \pi^k_s}^2_2$. Thus, for the example in question the two algorithms are equivalent up to scaling of the step-size by a constant factor.

\paragraph{Closure conditions in practice.}
Closure conditions (that are based on bounded approximation error) roughly stipulate the policy class is closed to a soft policy improvement step.
This has a flavor that is similar to Bellman completeness \citep{munos2008finite,chen2019information,zanette2020learning,zanette2023realizability}, a property of a $Q$-function class that says the class is closed to a Bellman backup step.  
Bellman completeness is widely considered too strong a condition to hold in practice, the reasoning being that increasing capacity of a function class that violates completeness inadvertently introduces new functions for which completeness needs to be satisfied. Therefore, an increase in capacity may actually cause completeness to be further violated.
The same can be argued for closure conditions, with one difference being that the complete policy class $\Delta(\cA)^\cS$ is naturally closed to any policy improvement step. However, in a large scale environment setting, the complete policy class is typically many orders of magnitude too large to be well approximated by realistically sized neural network architectures (at least at the present time).

\paragraph{PMD and VGD from the optimization perspective.}
Standard arguments from optimization literature are insufficient to establish convergence of PMD with the VGD condition. First, PMD is not an algorithm that has (prior to our work) a formulation within a purely optimization-based framework.
    Second, convergence in a smooth non-convex setting typically scales with the distance to the optimal solution, measured by the norm induced by smoothness of the objective. Prior works that establish convergence of gradient descent type methods (though not of PMD; e.g., \citealp{agarwal2021theory,bhandari2024global,xiao2022convergence}) exploit smoothness of the value function w.r.t.~the Euclidean norm (established  in \citealp{agarwal2021theory}), and as a result obtain bounds that scale with the cardinality of the state-space.

\paragraph{Divergence of Policy Iteration.}
Our setup with the VGD condition is general enough to accommodate examples where the policy iteration algorithm does not converge (the same example we discuss in \cref{sec:vgd_not_imply_closure} demonstrates this). Here, since the policy class is non-complete, the policy improvement step is performed over the current policy occupancy measure (see \citealp{bhandari2024global} who introduce this natural adaptation). Arguably, it should not be expected that policy iteration converges for real world, large-scale problems, as it is a very ``non-regularized'' algorithm from an optimization perspective.
    At the same time,
    in setups where closure conditions based on bounded approximation error hold, in particular, closure to policy improvement as studied in \citet{bhandari2024global}, the policy iteration algorithm converges at a linear rate. Thus it is not immediately clear why should we employ more sophisticated algorithms such as PMD in such settings.

\paragraph{Convergence beyond the VGD condition.}
    Using our framework, it can be shown that PMD converges to a stationary point regardless of any VGD condition; see \cref{sec:prox_convergence_stationary_point}.

\section{Deferred Discussions}

\subsection{Assumption on the critic error}
\label{sec:discuss_critic_error}
Our results can be easily adapted to 
    the (generally weaker) assumption that 
    \begin{align*}
        \E_{s \sim \mu^\pi}
        \norm[b]{\widehat Q_{s}^{\pi} - Q_{s}^{\pi}}_2 
        \leq \epscritM.
    \end{align*}
    (In which case the bounds would depend on $\epscritM$ rather than $\sqrt \epscritM$.) \cref{assm:Q_oracle} in its current form simplifies presentation, since it allows working with the weighted $L^2$ norm for both smoothness and errors in the gradient approximation. Also noteworthy, when working with the negative entropy regularizer, approximation w.r.t.~the $\norm{\cdot}_\infty$ norm would suffice. Since the statistical errors are not the focus of this work, we make these concessions in favor of a more streamlined and clear presentation. 
    
\subsection{Local smoothness of the value function requires greedy exploration}
\label{sec:discuss_value_local_smoothness}
In this section we discuss why the dependence on $\epsilon$ in the bound of \cref{lem:value_local_smoothness} cannot be improved in general.
We consider the MDP in \cref{fig:mdp_smoothness_lb}, for which we can show \cref{lem:value_local_smoothness} has tight dependence on the $\epsilon$-exploration parameter.
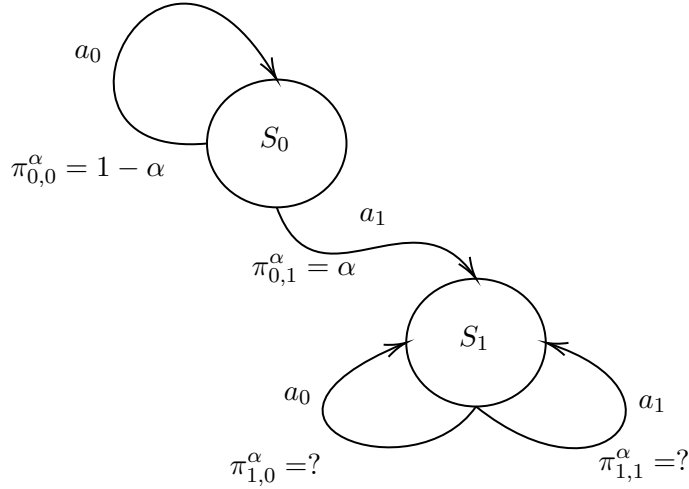
\begin{figure}[ht!]
    \centering
\tikzset{every picture/.style={line width=0.75pt}} 
\begin{tikzpicture}[x=0.75pt,y=0.75pt,yscale=-1,xscale=1]

\draw   (244,99.93) .. controls (244,82.21) and (259.67,67.85) .. (279,67.85) .. controls (298.33,67.85) and (314,82.21) .. (314,99.93) .. controls (314,117.64) and (298.33,132) .. (279,132) .. controls (259.67,132) and (244,117.64) .. (244,99.93) -- cycle ;
\draw    (244,99.93) .. controls (148.48,108.81) and (219.28,-35.69) .. (278.11,66.3) ;
\draw [shift={(279,67.85)}, rotate = 240.67] [color={rgb, 255:red, 0; green, 0; blue, 0 }  ][line width=0.75]    (10.93,-3.29) .. controls (6.95,-1.4) and (3.31,-0.3) .. (0,0) .. controls (3.31,0.3) and (6.95,1.4) .. (10.93,3.29)   ;
\draw   (344,199.93) .. controls (344,182.21) and (359.67,167.85) .. (379,167.85) .. controls (398.33,167.85) and (414,182.21) .. (414,199.93) .. controls (414,217.64) and (398.33,232) .. (379,232) .. controls (359.67,232) and (344,217.64) .. (344,199.93) -- cycle ;
\draw    (279,132) .. controls (298.8,186.3) and (348,122.16) .. (378.09,166.48) ;
\draw [shift={(379,167.85)}, rotate = 237.45] [color={rgb, 255:red, 0; green, 0; blue, 0 }  ][line width=0.75]    (10.93,-3.29) .. controls (6.95,-1.4) and (3.31,-0.3) .. (0,0) .. controls (3.31,0.3) and (6.95,1.4) .. (10.93,3.29)   ;
\draw    (379,232) .. controls (444.67,284.59) and (490.54,224.46) .. (415.15,200.29) ;
\draw [shift={(414,199.93)}, rotate = 17.26] [color={rgb, 255:red, 0; green, 0; blue, 0 }  ][line width=0.75]    (10.93,-3.29) .. controls (6.95,-1.4) and (3.31,-0.3) .. (0,0) .. controls (3.31,0.3) and (6.95,1.4) .. (10.93,3.29)   ;
\draw    (379,232) .. controls (348.16,276.63) and (246.03,240.22) .. (342.53,200.52) ;
\draw [shift={(344,199.93)}, rotate = 158.04] [color={rgb, 255:red, 0; green, 0; blue, 0 }  ][line width=0.75]    (10.93,-3.29) .. controls (6.95,-1.4) and (3.31,-0.3) .. (0,0) .. controls (3.31,0.3) and (6.95,1.4) .. (10.93,3.29)   ;

\draw (269,90) node [anchor=north west][inner sep=0.75pt]   [align=left] {$\displaystyle S_{0}$};
\draw (369,190) node [anchor=north west][inner sep=0.75pt]   [align=left] {$\displaystyle S_{1}$};
\draw (144,104) node [anchor=north west][inner sep=0.75pt]   [align=left] {$\displaystyle \pi _{0,0}^{\alpha } =1-\alpha $};
\draw (265,154) node [anchor=north west][inner sep=0.75pt]   [align=left] {$\displaystyle \pi _{0,1}^{\alpha } =\alpha $};
\draw (439,252) node [anchor=north west][inner sep=0.75pt]   [align=left] {$\displaystyle \pi _{1,1}^{\alpha } =?$};
\draw (254,254) node [anchor=north west][inner sep=0.75pt]   [align=left] {$\displaystyle \pi _{1,0}^{\alpha } =?$};
\draw (176,49) node [anchor=north west][inner sep=0.75pt]   [align=left] {$\displaystyle a_{0}$};
\draw (281,221) node [anchor=north west][inner sep=0.75pt]   [align=left] {$\displaystyle a_{0}$};
\draw (459,225) node [anchor=north west][inner sep=0.75pt]   [align=left] {$\displaystyle a_{1}$};
\draw (319,131) node [anchor=north west][inner sep=0.75pt]   [align=left] {$\displaystyle a_{1}$};

\end{tikzpicture}
\caption{A two state deterministic MDP, with $\rho_0(S_0)=1$. Each edge is labeled with an action ($a\in \cbr{a_0, a_1}$) that takes the agent to the state at the other end. A policy $\pi^\alpha$, $\alpha\in [0,1]$ takes actions in $S_0$ with the probabilities displayed in the diagram next to the relevant action. The probabilities $\pi^\alpha$ assigns to actions in $S_1$ denoted by $?$ are unrelated to $\alpha$ and left for later.}
    \label{fig:mdp_smoothness_lb}
\end{figure}
Let $p\in(0,1/2)$ and $0< \epsilon < p$. Define:
\begin{align*}
	\pi \eqq \pi^\epsilon, \quad \pi_{1,0}=1,\pi_{1,0}=0,
	\\
	\tilde \pi \eqq \pi^p, \quad  \tilde \pi_{1,0}=0, \tilde \pi_{1,0}=1.
\end{align*}
\paragraph{Idea.}
Think of $\epsilon$ as much smaller than $p$. When measuring distance with the local norm $\norm{\tilde \pi - \pi}_{L^2(\mu^\pi), 1}$, the large difference $\norm{\tilde \pi_1 - \pi_1}_1^2$ gets little weight: $\mu^\pi(S_1) \approx \epsilon$.
Meanwhile, the error of the linear approximation at $\pi$ behaves like (see proof of \cref{lem:value_local_smoothness} in \cref{sec:proof:lem:value_local_smoothness}):
\begin{align*}
	\av{\sum_s \mu^\pi(s) \sum_{a}
    		\br{\tilde \pi_{sa} - \pi_{sa}}
    		\br{\sum_{s'}\mu^\pi_{\P_{sa}}(s')\norm{\tilde \pi_{s'}- \pi_{s'}}_1}
    },
\end{align*}
where the weight assigned to $\norm{\tilde \pi_1 - \pi_1}_1^2$ is approximately $(\tilde \pi_{0, 1} - \pi_{0, 1}) = p-\epsilon$.
Hence, if $\epsilon=p^2$,
\begin{align*}
	\av{V(\tilde \pi) - V(\pi) - \abr{\nabla V(\pi), \tilde \pi - \pi}}
	&\approx p,
	\\
	\norm{\tilde \pi - \pi}_{L^2(\mu^\pi), 1}^2 
	&\approx p^2,
\end{align*}
so
\begin{align*}
	\av{V(\tilde \pi) - V(\pi) - \abr{\nabla V(\pi), \tilde \pi - \pi}}
	\gtrsim  \frac{1}{\sqrt \epsilon}\norm{\tilde \pi - \pi}_{L^2(\mu^\pi), 1}^2.
\end{align*}

\paragraph{Computations.}
The term that is equal to the linearization error, up to constant factors, is the following:
\begin{align*}
\av{\sum_s \mu^\pi(s) \sum_{a}
    		\br{\tilde \pi_{sa} - \pi_{sa}}
    		\br{\sum_{s'}\mu^\pi_{\P_{sa}}(s')\abr{Q^{\tilde \pi}_{s'}, \tilde \pi_{s'}- \pi_{s'}}}
    }.
\end{align*}
Assume $p > \epsilon$.
By choosing a cost function $r(s,i)=i$ for $s\in \cbr{S_0, S_1}$, $i\in \cbr{0,1}$ we have that for all $s$,
\begin{align*}
	\abr{Q^{\tilde \pi}_{s}, \tilde \pi_{s}- \pi_{s}}
	= \Omega(\norm{\tilde \pi_s - \pi_s}_1),
\end{align*}
hence we focus on lower bounding
\begin{align*}
	(*) \eqq \av{\sum_s \mu^\pi(s) \sum_{a}
    		\br{\tilde \pi_{sa} - \pi_{sa}}
    		\br{\sum_{s'}\mu^\pi_{\P_{sa}}(s')\norm{\tilde \pi_{s'}- \pi_{s'}}_1}
    }.
\end{align*}
By direct computation,
\begin{align*}
	\mu^\pi(S_0) = \frac{1}{1+\gamma\epsilon}, \quad
	\mu^\pi(S_1) = \frac{\gamma\epsilon}{(1+\gamma\epsilon)(1-\gamma)}
\end{align*}
and 
\begin{align*}
\norm{\tilde \pi_0 - \pi_0}_1=2|p-\epsilon|,
\quad \norm{\tilde \pi_1 - \pi_1}_1=2.
\end{align*}
Thus,
\begin{align*}
	(\tilde \pi_{0,0} - \pi_{0,0})\sum_{s'}\mu^\pi_{\P_{0,0}}(s')\norm{\tilde \pi_{s'}- \pi_{s'}}_1
	&\approx
	(1-\epsilon)(\epsilon - p)|p-\epsilon| + \epsilon
	\geq - p^2 + \epsilon
	\\
	(\tilde \pi_{0,1} - \pi_{0,1})\sum_{s'}\mu^\pi_{\P_{0,1}}(s')\norm{\tilde \pi_{s'}- \pi_{s'}}_1
	&\approx
	p-\epsilon,
\end{align*}
and further,
\begin{align*}
\av{\sum_{a}
    		\br{\tilde \pi_{1,a} - \pi_{1,a}}
    		\br{\sum_{s'}\mu^\pi_{\P_{1,a}}(s')\norm{\tilde \pi_{s'}- \pi_{s'}}_1}
    } = 0.
\end{align*}
We obtain
\begin{align*}
	(*) 
	\gtrsim \mu^\pi(S_0) \br{p - p^2}
	\approx p - p^2.
\end{align*}
Meanwhile,
\begin{align*}
	\norm{\tilde \pi - \pi}_{L^2(\mu^\pi), 1}^2 
	=
	\frac{4(p-\epsilon)^2}{1+\gamma\epsilon}
	+ \frac{4\gamma\epsilon}{(1+\gamma\epsilon)(1-\gamma)}
	\approx (p-\epsilon)^2 + \epsilon.
\end{align*}
Now,  
\begin{align*}
	\frac{\av{V(\tilde \pi) - V(\pi) - \abr{\nabla V(\pi), \tilde \pi - \pi}}}
	{\norm{\tilde \pi - \pi}_{L^2(\mu^\pi), 1}^2 }
	\approx
		\frac{(*)}
	{\norm{\tilde \pi - \pi}_{L^2(\mu^\pi), 1}^2 }
	\approx \frac{p - p^2}{(p-\epsilon)^2 + \epsilon}.
\end{align*}
Now, for $\epsilon\eqq p^2, p< 1/2$, we obtain 
\begin{align*}
	\frac{p - p^2}{(p-\epsilon)^2 + \epsilon}
	=
	\frac{p - p^2}{(p-p^2)^2 + p^2}
	\geq \frac{p}{4 p^2} = \frac{1}{4 p} 
		= \frac{1}{4\sqrt \epsilon}.
\end{align*}
\section{State-weighted state-action space: Basic Facts}
Given a state probability measure $\mu\in \Delta(\cS)$, and an action space norm $\norm{\cdot}_\circ\colon \R^A \to \R$, we define the induced state-action weighted $L^p$ norm $\norm{\cdot}_{L^p(\mu), \circ}\colon \R^{SA}\to \R$:
\begin{align}
	\norm{u}_{L^p(\mu), \circ}
	&\eqq \br{\E_{s \sim \mu}\norm{u_s}_\circ^p}^{1/p}.
\end{align}
In addition, for $\mu\in \R^S, Q\in \R^{SA}$, we define the state to state-action element-wise product $\mu \circ Q \in \R^{SA}$:
\begin{align}
	\br{\mu \circ Q}_{s, a} \eqq \mu(s)Q_{s, a}.
\end{align}
\begin{lemma}\label{lem:weighted_norm_dual}
	For any strictly positive measure $\mu\in \R_{++}^S$,
	the dual norm of $\norm{\cdot}_{L^2(\mu),\circ}$ is given by
	\begin{align}
		\norm{z}_{L^2(\mu), \circ}^*
		&= \sqrt{\int \mu(s)^{-1}\br{\norm{z_s}_\circ^*}^2 {\rm d}s}
	\end{align}
\end{lemma}
\begin{proof}
	First denote
$$\begin{aligned}
	z_s^* &\eqq \argmax_{u_s \in \R^A, \norm{u_s}_\circ\leq 1}\abr{u_s, z_s}
	\\
	\implies
	\norm{z_s}^*_\circ &= \abr{z_s^*, z_s},
	\text{ and } \norm{z_s^*}_\circ = 1.
\end{aligned}$$
Now let $x\in \R^{SA}$ be defined by $x_s \eqq \frac{\norm{z_s}_\circ^*}{\mu(s)} z_s^*$, then
$$\begin{aligned}
	\abr{x, z}
	= \int \frac{\norm{z_s}_\circ^*}{\mu(s)}\abr{z_s^*, z_s} \rmd s
	= \int \frac{1}{\mu(s)} \br{\norm{z_s}_\circ^*}^2 {\rm d}s.
\end{aligned}$$
Now, note that
$$
	\norm{x}_{L^2(\mu), \circ}
	= \int \mu(s)\br{\frac{\norm{z_s}_\circ^*}{\mu(s)}}^2\norm{z_s^*}_\circ^2
	= \int \frac{1}{\mu(s)}\br{\norm{z_s}_\circ^*}^2
	= \abr{x, z},
$$
hence, for $\bar x \eqq x/\norm{x}_{L^2(\mu), \circ}$ we have $\norm{\bar x}_{L^2(\mu), \circ}=1$, and
$$
	\abr{\bar x, z} = \sqrt{\int \frac{1}{\mu(s)} \br{\norm{z_s}_\circ^*}^2 \rmd s.}
$$
On the other hand, for any $v$ such that $\norm{v}_{L^2(\mu), \circ} \leq 1$, we have 
\begin{align*} \abr{v, z}
	= \int \abr{v_s, z_s} \rmd s
	&= \int \abr{\mu(s)v_s, \mu(s)^{-1}z_s} \rmd s
	\\
	&\leq \int \norm{\sqrt{\mu(s)}v_s}_\circ
		\norm{\sqrt{\mu(s)^{-1}}z_s}_\circ^* \rmd s
	\\
	&\leq \sqrt{\int \mu(s)\norm{v_s}_\circ^2 \rmd s}
	 \sqrt{\int \mu(s)^{-1}\br{\norm{z_s}_\circ^*}^2 \rmd s}
	\\
	&\leq 
	 \sqrt{\int \mu(s)^{-1}\br{\norm{z_s}_\circ^*}^2 \rmd s},
\end{align*}
and the proof is complete.
\end{proof}
\begin{lemma}\label{lem:wnorm_dual_mu}
	Let $\mu\in \Delta(\cS)$, and consider the state-action norm $\norm{\cdot}_{L^2(\mu), \circ}$. For any $W\in \R^{SA}$, we have
	\begin{align*}
		\norm{\mu \circ W}_{L^2(\mu), \circ}^*
		= \sqrt{\E_{s\sim \mu} \br{\norm{W_s}_\circ^*}^2}
	\end{align*}
\end{lemma}
\begin{proof}
	By \cref{lem:weighted_norm_dual}, 
	\begin{align*}
		\norm{\mu \circ W}_{L^2(\mu), \circ}^*
		&= \sqrt{\int \mu(s)^{-1} \br{\norm{\mu(s)W_s}_\circ^*}^2}
		= \sqrt{\E_{s\sim \mu} \br{\norm{W_s}_\circ^*}^2}
		. \qedhere
	\end{align*}
	
\end{proof}

\begin{lemma}\label{lem:reg_transform}
	Assume $h\colon \R^A \to \R$ is $1$-strongly convex and has $L$-Lipschitz gradient w.r.t.~$\norm{\cdot}$.
	Let  $\mu\in \Delta(\cS)$, and define $R_\mu(\pi) \eqq \E_{s\sim \mu}[h(\pi_s)]$.
	Then
	\begin{enumerate}
		\item $B_{R_\mu}(\pi, \tilde \pi) = \E_{s \sim \mu} B_R(\pi_s, \tilde \pi_s)$.
		\item $R_\mu$ is $1$-strongly convex and has an $L$-Lipschitz gradient w.r.t. $\norm{\cdot}_{L^2(\mu), \circ}$. 
	\end{enumerate}

\end{lemma}
\begin{proof}
We have
\begin{align*}
	\forall s, \nabla R_\mu(\pi)_{s} &= \mu(s) \nabla R(\pi_s) \in \R^A
	\\
	\implies
	B_{R_\mu}(\pi, \tilde \pi) 
	&=
	R_\mu(\pi) - R_\mu(\tilde \pi) - \abr{\nabla R_\mu(\tilde \pi), \pi - \tilde \pi}
	\\
	&=
	\E_{s \sim \mu}\sbr{
		R(\pi_s) - R(\tilde \pi_s) - \abr{\nabla R(\tilde \pi_s), \pi_s - \tilde \pi_s}
	}
	\\
	&=
	\E_{s \sim \mu} B_R(\pi_s, \tilde \pi_s)
	.
\end{align*}
Further, $1$-strongly convexity follows by
\begin{align*}
		\E_{s \sim \mu} B_R(\pi_s, \tilde \pi_s) \geq 
		\frac12\E_{s \sim \mu} \norm{\pi_s - \tilde \pi_s}_\circ^2,
\end{align*}
and the Lipschitz gradient condition from \cref{lem:wnorm_dual_mu}:
\begin{align*}
	\norm{\nabla R_\mu(\pi) - \nabla R_\mu(\pi^+)}_{L^2(\mu), \circ}^*
	&=
	\norm{\mu\circ\br{\nabla h(\pi_s) - \nabla h(\pi_s^+)}}_{L^2(\mu), \circ}^*
	\\
	&=
	\sqrt{\E_{s \sim \mu}
	\br{\norm{\nabla h(\pi_s) - \nabla h(\pi_s^+)}_\circ^*}^2}
	\\
	&\leq
	L\sqrt{\E_{s \sim \mu}
	\norm{\pi_s - \pi_s^+}_\circ^2}
	\\
	&=
	L
	\norm{\pi - \pi^+}_{L^2(\mu),\circ},
\end{align*}
which completes the proof.
\end{proof}

\section{Deferred proofs}

\subsection{Auxiliary Lemmas}
\label{sec:aux_lemmas}
\begin{lemma}[Value difference; \citealp{kakade2002approximately}]\label{lem:value_diff}
    For any $\rho\in \Delta(\cS)$,
    \begin{align*}
        V_\rho\br{\tilde \pi} -  V_\rho\br{ \pi}
        = \frac{1}{1-\gamma}
        \E_{s\sim \mu_\rho^\pi}\abr{Q^{\tilde \pi}_s, \tilde \pi_s - \pi_s}.
    \end{align*}
\end{lemma}

\begin{lemma}[Policy gradient theorem; \citealp{sutton1999policy}]\label{lem:value_pg}
    For any $\rho\in \Delta(\cS)$,
    \begin{align*}
        \br{\nabla V_\rho(\pi)}_{s, a}
        &= \frac{1}{1-\gamma} \mu^\pi_\rho(s) Q^\pi_{s, a},
        \\
        \abr{\nabla V_\rho(\pi), \tilde \pi - \pi}
        &= \frac{1}{1-\gamma} \E_{s\sim \mu^\pi_\rho} 
        \abr{Q^\pi_s, \tilde \pi_s - \pi_s}.
    \end{align*}
\end{lemma}

The following lemma can be found in e.g., \cite{bhandari2024global, agarwal2021theory}. The proof below is provided for convenience.

\begin{lemma}\label{lem:vgd_complete}
	Let $\Pi\subset \Delta(\cA)^\cS, \Pi_{\rm all} \eqq \Delta(\cA)^\cS$ and suppose that for any policy $\pi\in \Pi$, we have 
	\begin{align*}
		\max_{\pi^+ \in \Pi}\E_{s\sim \mu^\pi}\abr{Q^\pi, \pi - \pi^+}
		\geq 
		\max_{\pi' \in \Pi_{\rm all}}\E_{s\sim \mu^\pi}\abr{Q^\pi, \pi - \pi'} - \epsilon.
	\end{align*}
	Then $\Pi$ is $(H\nu_0, \epsilon H^2 \nu_0)$-VGD w.r.t.~$\cM$, for $\nu_0\eqq \norm{\frac{\mu^{\star}}{\rho_0}}_\infty$.
\end{lemma}
\begin{proof}
	Let $\pi^\star\in \argmin_{\pi\in \Pi} V(\pi)$.
    By value difference \cref{lem:value_diff},
    \begin{align*}
	V(\pi) - V(\pi^\star)
	&= H \E_{s \sim \mu^{\star}}\sbr{\abr{Q^\pi_s, \pi_s - \pi_s^\star}}
	\\
	&\leq H \max_{\pi'\in \Pi_{\rm all}}\E_{s \sim \mu^{\star}}\sbr{
		\abr{Q^\pi_s, \pi_s - \pi'_s}}
	\\
	&\overset{(*)}{\leq} H \norm{\frac{\mu^{\star}}{\mu^\pi}}_\infty \max_{\pi'\in \Pi_{\rm all}}
		\E_{s \sim \mu^\pi}\sbr{\abr{Q^\pi_s, \pi_s - \pi'_s}}
	\\
	&\leq H \norm{\frac{\mu^{\star}}{\mu^\pi}}_\infty \max_{\pi^+\in \Pi}
		\E_{s \sim \mu^\pi}\sbr{\abr{Q^\pi_s, \pi_s - \pi^+_s}}
		+ \epsilon H\norm{\frac{\mu^{\star}}{\mu^\pi}}_\infty
	\\
	&= \norm{\frac{\mu^{\star}}{\mu^\pi}}_\infty \max_{\pi^+\in \Pi}
		\abr{\nabla V^\pi, \pi - \pi^+}
		+ \epsilon H\norm{\frac{\mu^{\star}}{\mu^\pi}}_\infty
	\\
	&\overset{(**)}{\leq} H \norm{\frac{\mu^{\star}}{\rho_0}}_\infty \max_{z\in \Pi}
		\abr{\nabla V^\pi, \pi - z}
		+ \epsilon H^2\norm{\frac{\mu^{\star}}{\rho_0}}_\infty
	.
\end{align*}
    To explain the transitions above, $(*)$ follows by the fact that within the complete policy class we may choose $\pi'$ to be greedy w.r.t. $Q^\pi$, which means $\abr{Q^\pi_s, \pi_s - \pi'_s} \geq 0$ for all $s\in \cS$. The last transition $(**)$ follows from the fact that:
    \begin{align*}
        \mu^\pi(s) 
        &= \frac1H \sum_{t=0}^\infty \Pr(s_t = s \mid \rho_0, \pi)
        = \frac1H\rho_0(s) + \sum_{t=1}^\infty \Pr(s_t = s \mid \rho_0, \pi)
        \geq 
        \frac1H \rho_0(s).
        \qedhere
    \end{align*}
\end{proof}

\begin{lemma}\label{lem:Qvalue_diff}
	For any policy $\pi\colon \cS \to \Delta(\cA)$, $s, a \in \cS \times \cA$:
	\begin{align*}
	Q_{s, a}^{\tilde \pi} - Q_{s, a}^\pi
	&= \gamma H\E_{s' \sim \mu^\pi_{\P_{s, a}}} 
	\abr{Q^{\tilde\pi}_{s'}, \tilde \pi_{s'} - \pi_{s'}}.
    \end{align*}
\end{lemma}
\begin{proof}
	By \cref{lem:value_diff},
	we have:
	\begin{align*}
	Q_{s, a}^{\tilde \pi} - Q_{s, a}^\pi
	&= \gamma \E_{s' \sim \P_{s, a}}\sbr{V^{\tilde \pi}(s') - V^\pi(s')}
	\\
	&= \gamma H\E_{s' \sim \P_{s, a}}
            \sbr{\E_{s''\sim \mu^\pi_{s'}}\abr{Q^{\tilde\pi}_{s''}, \tilde \pi_{s''} - \pi_{s''}}}
	\\
	&= \gamma H\sum_{s'} \P(s' |s, a)
		\sum_{s''} \mu^\pi_{s'}(s'')
	\abr{Q^{\tilde\pi}_{s''}, \tilde \pi_{s''} - \pi_{s''}}
	\\
	&= \gamma H\sum_{s''} \sum_{s'} \P(s' |s, a)
            \mu^\pi_{s'}(s'')
	\abr{Q^{\tilde\pi}_{s''}, \tilde \pi_{s''} - \pi_{s''}}
	\\
	&= \gamma H\sum_{s''} \mu^\pi_{\P_{s, a}}(s'')
	\abr{Q^{\tilde\pi}_{s''}, \tilde \pi_{s''} - \pi_{s''}}
        \\
	&= \gamma H\E_{s'' \sim \mu^\pi_{\P_{s, a}}} 
	\abr{Q^{\tilde\pi}_{s''}, \tilde \pi_{s''} - \pi_{s''}}.
    \qedhere
    \end{align*}
\end{proof}
\begin{lemma}\label{lem:negent_smooth}
	Let $h\colon \R^A \to \R$ be the negative entropy regularizer $h(p)\eqq \sum_i p_i\log p_i$, and assume $\Delta_\epsilon (\cA)\subset \Delta(\cA)$ is such that $p_i \geq \epsilon$ for all $p\in \Delta_\epsilon (\cA)$. Then $h$ has $1/\epsilon$-Lipschitz gradient w.r.t.~$\norm{\cdot}_1$ over $\Delta_\epsilon (\cA)$.
\end{lemma}
\begin{proof}
Let $p, \tilde p\in \Delta_\epsilon(\cA)$, and note,
$$
	\norm{\nabla h(p) - \nabla h(\tilde p)}^*_1
	=
	\norm{\nabla h(p) - \nabla h(\tilde p)}_\infty.
$$
Let $i\in \cA$, and observe that by the mean value theorem, for some $\alpha\in [p_i, \tilde p_i]$,
$$
	\av{\log(p_i) - \log(\tilde p_i)}
	= \av{\frac{\partial \log(x)}{\partial x}}_{x=\alpha}\av{p_i - \tilde p_i}
	= \frac{1}{\alpha}\av{p_i - \tilde p_i}
	\leq \frac{1}{\epsilon}\av{p_i - \tilde p_i}
	\leq \frac{1}{\epsilon}\norm{p - \tilde p}_1,
$$
since $p_i, \tilde p_i \geq \epsilon$.
\end{proof}

\subsection{Proof of \cref{lem:value_local_smoothness}}
\label{sec:proof:lem:value_local_smoothness}
\begin{proof}[\unskip\nopunct]
    We have, by \cref{lem:value_diff,lem:value_pg},
    \begin{align}
	\av{V^{\tilde \pi} - V^{\pi} - \abr{\nabla V^\pi, \tilde \pi - \pi}}
	&= \av{H\E_{s \sim \mu^\pi}\abr{Q^{\tilde \pi}_s, \tilde \pi_s - \pi_s}
	- H\E_{s \sim \mu^\pi}\abr{Q^{\pi}_s, \tilde \pi_s - \pi_s}}
	\nonumber \\
        &= H\av{\E_{s \sim \mu^\pi}\abr{Q^{\tilde \pi}_s - Q^\pi_s, \tilde \pi_s - \pi_s}}
	\nonumber.
    \end{align}
    Applying \cref{lem:Qvalue_diff} yields,
    \begin{align*}
    &\frac{1}{\gamma H^2}\av{V^{\tilde \pi} - V^{\pi} - \abr{\nabla V^\pi, \tilde \pi - \pi}}
    \\
    &=\av{
     \E_{s\sim \mu^\pi}\sbr{\sum_a \br{
    	\E_{s' \sim \mu^\pi_{\P_{s, a}}} 
    	\abr{Q^{\tilde\pi}_{s'}, \tilde \pi_{s'} - \pi_{s'}}
        }
    \br{\tilde \pi_{sa} - \pi_{sa}}}
    }
    \\
    &= \av{\sum_s \mu^\pi(s) \sum_{a}
    		\br{\tilde \pi_{sa} - \pi_{sa}}
    		\br{\sum_{s'}\mu^\pi_{\P_{sa}}(s')\abr{Q^{\tilde \pi}_{s'}, \tilde \pi_{s'}- \pi_{s'}}}
    }
    \\
    &= \av{\sum_{s,a} \sqrt{\mu^\pi(s)}
    \br{\tilde \pi_{sa} - \pi_{sa}}
    \br{\sqrt{\mu^\pi(s)}\sum_{s'}\mu^\pi_{\P_{sa}}(s')\abr{Q^{\tilde \pi}_{s'}, \tilde \pi_{s'}- \pi_{s'}}}
    } \\
    &\leq 
    \sqrt{\sum_{s,a} \mu^\pi(s)
    \br{\tilde \pi_{sa} - \pi_{sa}}^2}
    \sqrt{\sum_{s,a}\mu^\pi(s)
        \br{\sum_{s'}\mu^\pi_{\P_{sa}}(s')\abr{Q^{\tilde \pi}_{s'}, \tilde \pi_{s'}- \pi_{s'}}}^2}
    \tag{Cauchy-Schwarz}
    \\
    &\leq
    \sqrt{\sum_{s,a} \mu^\pi(s)
    \br{\tilde \pi_{sa} - \pi_{sa}}^2}
    \sqrt{\sum_{s, a}\mu^\pi(s)
        \sum_{s'}\mu^\pi_{\P_{sa}}(s')\abr{Q^{\tilde \pi}_{s'}, \tilde \pi_{s'}- \pi_{s'}}^2}
    \tag{Jensen}
    \\
    &=
    \sqrt{\sum_{s} \mu^\pi(s)
    \norm{\tilde \pi_{s} - \pi_{s}}_2^2}
    \sqrt{\sum_{s'}
        \br{\sum_{s, a}\frac{1}{\pi_{sa}}\mu^\pi(s)\pi_{sa}\mu^\pi_{\P_{sa}}(s')}
        \abr{Q^{\tilde \pi}_{s'}, \tilde \pi_{s'}- \pi_{s'}}^2}
    \\
    &\leq
    \frac{1}{\sqrt \epsilon}
    \sqrt{\sum_{s} \mu^\pi(s)
    \norm{\tilde \pi_{s} - \pi_{s}}_2^2}
    \sqrt{\sum_{s'}
        \br{\sum_{s, a}\mu^\pi(s)\pi_{sa}\mu^\pi_{\P_{sa}}(s')}
        \abr{Q^{\tilde \pi}_{s'}, \tilde \pi_{s'}- \pi_{s'}}^2}
    ,
    \end{align*}
    for $\epsilon \eqq \min_{s, a}\cbr{\pi_{s a}}$.
    Now, by the law of total probability (applied on the discounted probability measure $\mu^\pi$):
    \begin{align*}
	\sum_{s,a}\mu^\pi(s)\pi_{sa}\mu^\pi_{\P_{sa}}(s')
	&= \sum_{s,a}\mu^\pi(s \mid s_0 \sim \rho_0) \pi(a|s)\mu^\pi(s' \mid s_0' \sim \P_{sa})
	\\
	&= \sum_{s,a}\mu^\pi(s, a \mid s_0 \sim \rho_0)\mu^\pi(s' \mid s_0' \sim \P_{sa})
	\\
	&= \mu^\pi(s' \mid s_0 \sim \rho_0)
	\\
	&= \mu^\pi(s').
    \end{align*}
    Combining with our previous inequality, we obtain
    \begin{align*}
    \av{V^{\tilde \pi} - V^{\pi} - \abr{\nabla V^\pi, \tilde \pi - \pi}}
    &\leq\frac{\gamma H^2}{\sqrt \epsilon}
    \sqrt{\sum_{s} \mu^\pi(s)
    \norm{\tilde \pi_{s} - \pi_{s}}_2^2}
    \sqrt{\sum_{s'}\mu^\pi(s')
        \abr{Q^{\tilde \pi}_{s'}, \tilde \pi_{s'}- \pi_{s'}}^2}
    \\
    &=
    \frac{\gamma H^2}{\sqrt \epsilon}
    \norm{\tilde \pi - \pi}_{L^2(\mu^\pi), 2}
    \sqrt{\sum_{s'}\mu^\pi(s')
        \abr{Q^{\tilde \pi}_{s'}, \tilde \pi_{s'}- \pi_{s'}}^2}.
    \end{align*}
    Further,
    \begin{align*}
        \sqrt{\sum_{s'}\mu^\pi(s')
        \abr{Q^{\tilde \pi}_{s'}, \tilde \pi_{s'}- \pi_{s'}}^2}
        \leq 
        \sqrt{\sum_{s'}\mu^\pi(s')
        \norm{Q^{\tilde \pi}_{s'}}_\infty^2
        \norm{\tilde \pi_{s'}- \pi_{s'}}_1^2}
        \leq H
        \norm{\tilde \pi- \pi}_{L^2(\mu^\pi), 1},
    \end{align*}
    and 
    \begin{align*}
        \sqrt{\sum_{s'}\mu^\pi(s')
        \abr{Q^{\tilde \pi}_{s'}, \tilde \pi_{s'}- \pi_{s'}}^2}
        \leq 
        \sqrt{\sum_{s'}\mu^\pi(s')
        \norm{Q^{\tilde \pi}_{s'}}_2^2
        \norm{\tilde \pi_{s'}- \pi_{s'}}_2^2}
        \leq A H
        \norm{\tilde \pi- \pi}_{L^2(\mu^\pi), 2}.
    \end{align*}
    The first inequality above gives
    \begin{align*}
    \av{V^{\tilde \pi} - V^{\pi} - \abr{\nabla V^\pi, \tilde \pi - \pi}}
    \leq \frac{\gamma H^3}{\sqrt \epsilon}
    \norm{\tilde \pi - \pi}_{L^2(\mu^\pi), 2}
    \norm{\tilde \pi - \pi}_{L^2(\mu^\pi), 1}
    \leq \frac{\gamma H^3}{\sqrt \epsilon}
    \norm{\tilde \pi - \pi}_{L^2(\mu^\pi), 1}^2,
    \end{align*}
    which proves the first claim, and the second one 
    \begin{align*}
    \av{V^{\tilde \pi} - V^{\pi} - \abr{\nabla V^\pi, \tilde \pi - \pi}}
    \leq \frac{\gamma A H^3}{\sqrt \epsilon}
    \norm{\tilde \pi - \pi}_{L^2(\mu^\pi), 2}
    \norm{\tilde \pi - \pi}_{L^2(\mu^\pi), 2}
    = \frac{\gamma A H^3}{\sqrt \epsilon}
    \norm{\tilde \pi - \pi}_{L^2(\mu^\pi), 2}^2,
    \end{align*}
    which proves the second and completes the proof.
\end{proof}

\subsection{Proof of \cref{thm:pmd_main}}
\label{sec:proof:thm:pmd_main}
The theorem makes use of the following.

\begin{lemma}\label{lem:epsgreedy_vgd}
	Assume $\Pi$ is $(C_\star, \epsvgd)$-VGD w.r.t. $\cM$, and consider the $\epsilon$-greedy exploratory version of $\Pi$, $\Pi^\epsilon \eqq \cbr{(1-\epsilon)\pi + \epsilon u \mid \pi \in \Pi}$, where $u_{s,a} \equiv 1/A$. Then $\Pi^\epsilon$ is $(C_\star, \delta)$-VGD with $\delta \eqq \epsvgd + 12 C_\star H^2 A \epsilon $. Concretely, for any $\pi^\epsilon \in \Pi^\epsilon$, we have:
	\begin{align*}
	C_\star \max_{\tilde \pi^\epsilon \in \Pi^\epsilon}\abr{\nabla V(\pi^\epsilon), \tilde \pi^\epsilon - \pi^\epsilon}
	\geq V\br{\pi^\epsilon} - V^\star(\Pi^\epsilon)
		- \epsvgd - 12 \epsilon C_\star H^2 A
	.
\end{align*}
\end{lemma}
We now prove our corollary and return to prove the above lemma later in \cref{sec:lem:epsgreedy_vgd}.
\begin{proof}[Proof of \cref{thm:pmd_main}]
    By \cref{lem:epsgreedy_vgd}, we have that $\Pi^\epsexpl$ is $(C_\star, \delta)$-VGD with $\delta=\epsvgd + 12 \epsexpl C_\star H^2 A$.
    Therefore, under the conditions of \cref{thm:opt_main} and the value difference \cref{lem:value_diff},
    \begin{align*}
		V(\pi^{K+1}) - V^\star(\Pi)
		&\leq V(\pi^{K+1}) - V^\star(\Pi^\epsexpl) 
			+ \av{V^\star(\Pi^\epsexpl) - V^\star(\Pi)}
		\\
		&= O\br{
            \frac{C_\star^2 L^2 c_1^2}{\eta K}
            + \br{C_\star D + c_1 L^2} H \sqrt \epscritM
            + C_\star \epsactM
            + c_1 L \eta^{-1/2}\sqrt \epsactM
            + \delta
            },
	\end{align*}
    where $c_1 \eqq D + \eta H M$.
    Next we apply \cref{lem:pmd_main} in the both cases considered, using the fact that for all $\pi\in \Pi^\epsexpl$, we have $\min_{s, a} \cbr{\pi_{s, a}} \geq \epsexpl/A$.
    In the euclidean case, we argue the following:
    \begin{enumerate}
        \item $R$ is $1$-strongly convex and has $1$-Lipschitz gradient w.r.t.~$\norm{\cdot}_2$.
        \item $\forall s, \norm{\pi_s - \tilde \pi_s}_2\leq D = 2$, $\norm{Q_s}_2\leq M=\sqrt{A H}$.
        \item The value function is 
        $\br{\beta \eqq \frac{A^{3/2} H^3}{\sqrt \epsexpl}}$-locally smooth w.r.t.~$\pi \mapsto \norm{\cdot}_{L^2(\mu^\pi), 2}$.
    \end{enumerate}
    Hence, $c_1=O(1)$, and \cref{lem:pmd_main} gives:
    \begin{align*}
        V(\pi^{K+1}) - V^\star(\Pi)
        &\lesssim \frac{C_\star^2}{\eta K}
        + C_\star\br{H \sqrt{\epscritM} + \epsactM}
        + \eta^{-1/2}\sqrt \epsactM
        + \delta
        \\
        &=
        \frac{2 A^{3/2}H^3 C_\star^2}{\sqrt{\epsexpl} K}
        + C_\star\br{H\sqrt{\epscritM} + \epsactM}
        + \frac{\sqrt{2 A^{3/2}H^3}}{\epsexpl^{1/4}}
        \sqrt \epsactM
        + \delta.
    \end{align*}
    Setting $\epsexpl = K^{-2/3}$, we obtain
    \begin{align*}
        V(\pi^{K+1}) - V^\star(\Pi)
        = O\br{
        \frac{C_\star^2 A^{3/2} H^3}{K^{2/3}}
        + C_\star\br{H\sqrt{\epscritM} + \epsactM}
        + A H^2 K^{1/6} \sqrt{\epsactM}
        + \epsvgd}.
    \end{align*}
    In the negative-entropy case, we have the following.
    \begin{enumerate}
        \item $R$ is $1$-strongly convex and has a $\br{A/\epsexpl}$-Lipschitz gradient w.r.t.~$\norm{\cdot}_1$ (by Pinsker's inequality and \cref{lem:negent_smooth}).
        \item $\forall s, \norm{\pi_s - \tilde \pi_s}_1\leq D = 2$, $\norm{Q_s}_1\leq M=H$.
        \item The value function is 
        $\br{\beta \eqq \frac{A^{1/2} H^3}{\sqrt \epsexpl}}$-locally smooth w.r.t.~$\pi \mapsto \norm{\cdot}_{L^2(\mu^\pi), 1}$.
    \end{enumerate}
    Hence, $c_1=O(1)$, and \cref{lem:pmd_main} gives:
    \begin{align*}
        V(\pi^{K+1}) - V^\star(\Pi)
        &\lesssim \frac{C_\star^2 A^2}{\epsexpl^2 \eta K}
        + \br{C_\star + \frac{A^2}{\epsexpl^2}}H\sqrt{\epscritM} 
        + C_\star\epsactM
        + \frac{A}{\epsexpl\sqrt\eta}\sqrt \epsactM
        + \delta
        \\
        &=
        \frac{2 A^{5/2}H^3 C_\star^2}{\epsexpl^{5/2} K}
        + \br{C_\star + \frac{A^2}{\epsexpl^2}}H\sqrt{\epscritM} 
        + C_\star\epsactM
        + \frac{A^{3/2}H^3}{\epsexpl^{5/4}}\sqrt \epsactM
        + \delta.
    \end{align*}
    We now set $\epsexpl = K^{-2/7} A^{2/5}$ in order to balance the terms,
    \begin{align*}
        \frac{2 A^{5/2}H^3 C_\star^2}{\epsexpl^{5/2} K}
        + C_\star H^2 A \epsexpl,
    \end{align*}
    which yields,
    \begin{align*}
        &V(\pi^{K+1}) - V^\star(\Pi)
        \\
        &=O\br{
        \frac{C_\star^2 A^{3/2} H^3 }{K^{2/7}}
        + \br{C_\star + A^2 K^{4/7}}H\sqrt{\epscritM} 
        + C_\star\epsactM
        + A^{3/2}H^3 K^{5/14}\sqrt \epsactM
        + \epsvgd},
    \end{align*}
    and completes the proof.
\end{proof}

\subsection{Proof of \cref{lem:epsgreedy_vgd}}
\label{sec:lem:epsgreedy_vgd}
\begin{lemma}\label{lem:om_valuediff}
	For any MDP $\cM=(\cS, \cA, \P, \l, \gamma, \rho_0)$ and two policies $\pi, \tilde \pi\colon \cS \to \Delta(\cA)$, we have:
	\begin{align*}
\norm{\mu^{\tilde \pi} - \mu^\pi}_1
	\leq H \norm{\tilde \pi - \pi}_{L^1(\mu^\pi), 1}.
	\end{align*}
\end{lemma}
\begin{proof}
	Consider the MDP $\cM_x = (\cS, \cA, \P, r_x, \gamma, \rho_0)$; i.e., the same MDP $\cM$ but with reward function defined by $r_x(s, a) \eqq \I\cbr{s = x}$.
	Let $V_{\cdot;r_x}, Q_{\cdot, \cdot; r_x}$ denote its value and action-value functions, respectively.
	We have 
\begin{align*}
	Q^{\tilde \pi}_{s,a;r_x} 
	&= \E\sbr{\sum_{t=0}^\infty \gamma^t \I\cbr{s_t = x} \mid s_0=s, a_0=a, \tilde \pi}
	\\
	&= \sum_{t=0}^\infty \gamma^t \Pr\br{s_t = x \mid s_0=s, a_0=a, \tilde \pi}
	\\
	&= \I\cbr{s=x} + \sum_{t=1}^\infty \gamma^t \Pr\br{s_t = x \mid s_0=s, a_0=a, \tilde \pi}
	\\
	&= \I\cbr{s=x} + \gamma \sum_{t=1}^\infty \gamma^{t-1} \Pr\br{s_t = x \mid s_1 \sim \P_{sa}, \tilde \pi}
	\\
	&= \I\cbr{s=x} + \gamma \mu_{\P_{sa}}^{\tilde \pi}(x).
\end{align*}
Hence,
\begin{align*}
	\mu^{\tilde \pi}(x) - \mu^\pi(x)
	&= V^{\tilde \pi}_{\rho_0; r_x} - V_{\rho_0; r_x}^\pi
	\\
	&= H \E_{s \sim \mu^\pi}
		\abr{Q_{s;r_x}^{\tilde \pi}, \tilde \pi_s -\pi_s}
	\tag{\cref{lem:value_diff}}
	\\
	&= H \E_{s \sim \mu^\pi}\sbr{
		\sum_a \br{\I\cbr{s=x} + \gamma \mu_{\P_{sa}}^{\tilde \pi}(x)}
		\br{\tilde \pi_{sa} -\pi_{sa}}}
	\\
	&= H \E_{s \sim \mu^\pi}\sbr{
	\sum_a \I\cbr{s=x} 
		\br{\tilde \pi_{sa} -\pi_{sa}}
		+
		\gamma\sum_a \mu_{\P_{sa}}^{\tilde \pi}(x)
		\br{\tilde \pi_{sa} -\pi_{sa}}}
	\\
	&= \gamma H\E_{s\sim \mu^\pi}\sbr{\sum_a \mu_{\P_{sa}}^{\tilde \pi}(x)
		\br{\tilde \pi_{sa} -\pi_{sa}}}.
\end{align*}
Therefore, 
\begin{align*}
	\sum_x \av{\mu^{\tilde \pi}(x) - \mu^\pi(x)}
	&=
	\gamma H \sum_x \av{\E_{s\sim \mu^\pi}\sbr{\sum_a \mu_{\P_{sa}}^{\tilde \pi}(x)
		\br{\tilde \pi_{sa} -\pi_{sa}}}}
	\\
	&\leq 
	\gamma H \sum_x 
		\E_{s\sim \mu^\pi}\sbr{\sum_a \mu_{\P_{sa}}^{\tilde \pi}(x)
		\av{\tilde \pi_{sa} -\pi_{sa}}}
	\\
	&=
	\gamma H  
		\E_{s\sim \mu^\pi}\sbr{\sum_a \br{\sum_x\mu_{\P_{sa}}^{\tilde \pi}(x)}
		\av{\tilde \pi_{sa} -\pi_{sa}}}
	\\
	&=
	\gamma H  
		\E_{s\sim \mu^\pi}\sbr{\sum_a
		\av{\tilde \pi_{sa} -\pi_{sa}}}
	\\
	&=
	\gamma H \norm{\tilde \pi -\pi}_{L^1(\mu^\pi), 1},
\end{align*}
and the proof is complete.
\end{proof}

\begin{proof}[Proof of \cref{lem:epsgreedy_vgd}]
Let $\pi^\epsilon \in \Pi^\epsilon$, and set $\pi\in \Pi$ to be the non-exploratory version of $\pi^\epsilon$. We have, by \cref{lem:value_diff}:
\begin{align}\label{eq:appxvgd_valuediff}
	V^\pi - V^{\pi^\epsilon}
	= \E_{s\sim\mu^\pi}\abr{Q^{\pi^\epsilon}_s, \pi_s - \pi^\epsilon_s}
	= \epsilon \E_{s\sim\mu^\pi}\abr{Q^{\pi^\epsilon}_s, \pi_s - u}
	\leq 2 \epsilon H.
\end{align}
In addition,
\begin{align*}
	\norm{\nabla V\br{\pi^\epsilon} - \nabla V(\pi)}_1
	&=
	\sum_s \norm{\mu^{\pi^\epsilon}(s) Q^{\pi^\epsilon}_s
	- \mu^{\pi}(s) Q^{\pi}_s}_1
	\\
	&\leq 
	\sum_s \norm{Q^{\pi^\epsilon}_s}_1
		\av{\mu^{\pi^\epsilon}(s) - \mu^\pi(s)} 
	+ \sum_s \mu^{\pi}(s) \norm{Q^{\pi}_s - Q^{\pi^\epsilon}_s}_1
	\\
	&\leq 
	A H \norm{\mu^{\pi^\epsilon} - \mu^\pi}_1
	+ \sum_s \mu^{\pi}(s) \norm{Q^{\pi}_s - Q^{\pi^\epsilon}_s}_1.
\end{align*}
To bound the first term, apply \cref{lem:om_valuediff}:
\begin{align*}
AH\norm{\mu^{\pi^\epsilon} - \mu^\pi}_1
	\leq A H^2 \norm{\pi^\epsilon - \pi}_{L^1(\mu^\pi), 1}
	\leq \epsilon A H^2 \norm{\pi - u}_{L^1(\mu^\pi), 1}
	\leq 2 \epsilon A H^2.
\end{align*}
To bound the second term, we have for any $\tilde \pi$:
\begin{align*}
	\sum_s \mu^{\pi}(s) \norm{Q^{\pi}_s - Q^{\tilde \pi}_s}_1
	&\leq
	H^2\sum_s \mu^\pi(s)\sum_a\sum_{s'} \mu^\pi_{\P_{s, a}}(s')
	\norm{\tilde \pi_{s'} - \pi_{s}}_1
	\\
	&=
	H^2 A\sum_{s'} \sum_{s,a} \mu^\pi(s)\frac1A\mu^\pi_{\P_{s, a}}(s')
	\norm{\tilde \pi_{s'} - \pi_{s}}_1
	\\
	&=
	H^2 A \norm{\tilde \pi - \pi}_{L^1(\nu), 1},
\end{align*}
where $\nu\in \R^\cS$ is defined by 
$$\nu(s') = \sum_{s,a} \mu^\pi(s)\frac1A\mu^\pi_{\P_{s, a}}(s').$$
By the law of total probability, $\nu\in \Delta(\cS)$ is in fact a state probability measure. Hence, we obtain 
\begin{align*}
	\sum_s \mu^{\pi}(s) \norm{Q^{\pi}_s - Q^{\pi^\epsilon}_s}_1
	&\leq
	H^2 A \norm{\pi^\epsilon - \pi}_{L^1(\nu), 1}
	=
	\epsilon H^2 A \norm{\pi - u}_{L^1(\nu), 1}
	\leq 
	2 \epsilon H^2 A.
\end{align*}
The bounds on both terms, combined with the previous display now yields
\begin{align}\label{eq:appxvgd_graddiff}
	\norm{\nabla V\br{\pi^\epsilon} - \nabla V(\pi)}_1
	&\leq 
	4\epsilon A H^2.
\end{align}
We now turn to apply \cref{eq:appxvgd_valuediff,eq:appxvgd_graddiff} to establish the claimed VGD condition.
Let $\pi^\epsilon\in \Pi^\epsilon$ be an arbitrary $\epsilon$-greedy policy and $\pi \in \Pi$ the non-exploratory version of $\pi^\epsilon$. The assumption that $\Pi$ is $(C_\star, \epsvgd)$-VGD implies
\begin{align*}
	\max_{\tilde \pi \in \Pi}\abr{\nabla V(\pi), \tilde \pi - \pi}
	\geq \frac1{C_\star}\br{V(\pi) - V^\star(\Pi) - \epsvgd}.
\end{align*}
Let $\tilde \pi\in \Pi$ be the policy maximizing the LHS, and $\tilde \pi^\epsilon=(1-\epsilon)\tilde \pi + \epsilon u\in \Pi^\epsilon$ its corresponding greedy exploration policy. We have,
\begin{align*}
	\abr{\nabla V(\pi^\epsilon), \tilde \pi^\epsilon - \pi^\epsilon}	
	&=(1-\epsilon)
		\abr{\nabla V(\pi^\epsilon), \tilde \pi - \pi}	
	\\
	&=
	(1-\epsilon)
		\abr{\nabla V(\pi), \tilde \pi - \pi}
	+
	(1-\epsilon)
		\abr{\nabla V(\pi^\epsilon) - \nabla V(\pi), \tilde \pi - \pi}
	\\
	&\geq
	\frac{1-\epsilon}{C_\star}\br{V(\pi) - V^\star(\Pi) - \epsvgd}
	+
	(1-\epsilon)
		\abr{\nabla V(\pi^\epsilon) - \nabla V(\pi), \tilde \pi - \pi}
	\\
	&\geq
	\frac{1}{C_\star}\br{V(\pi) - V^\star(\Pi) - \epsvgd}
	-
	2\norm{\nabla V(\pi^\epsilon) - \nabla V(\pi)}_1
	\\
	&\geq
	\frac{1}{C_\star}\br{V(\pi) - V^\star(\Pi) - \epsvgd}
	-
	8\epsilon H^2 A
	\tag{\cref{eq:appxvgd_graddiff}}
	\\
	&\geq
	\frac1{C_\star}
	\br{V(\pi^\epsilon) - V^\star(\Pi) - \epsvgd - \av{V(\pi^\epsilon) - V(\pi)}}
	-
	8\epsilon H^2 A
	\\
	&\geq
	\frac1{C_\star}
	\br{V(\pi^\epsilon) - V^\star(\Pi) - \epsvgd - 2\epsilon H}
	-
	8\epsilon H^2 A
	\tag{\cref{eq:appxvgd_valuediff}}
	\\
	&\geq
	\frac1{C_\star}
	\br{V(\pi^\epsilon) - V^\star(\Pi^\epsilon) - \epsvgd - 4\epsilon H}
	-
	8\epsilon H^2 A.
	\tag{\cref{eq:appxvgd_valuediff}}
\end{align*}
(Indeed, we pay for the difference $V^\star(\Pi^\epsilon)-V^\star(\Pi)$ here, only to pay it again in the other direction later, but it is cleaner this way and results in only an extra constant numerical factor.)
Therefore, 
\begin{align*}
	C_\star \max_{\hat \pi^\epsilon \in \Pi^\epsilon}\abr{\nabla V(\pi^\epsilon), \hat \pi^\epsilon - \pi^\epsilon}
	&\geq V(\pi^\epsilon) - V^\star(\Pi^\epsilon) - \epsvgd - 12 \epsilon C_\star H^2 A
	,
\end{align*}
which completes the proof.
\end{proof}

\section{Constrained non-convex optimization for locally smooth objectives: Analysis}
\label{sec:opt_analysis}
In this section, we provide the full technical details for \cref{sec:opt_main}. Recall that we consider the constrained optimization problem:
\begin{align}
	\min_{x\in \cX} f(x),
\end{align}
where the decision set $\cX\subseteq \R^d$ is convex and endowed with a local norm $x \mapsto \norm{\cdot}_x$ (see \cref{def:local_norm}), and access to the objective is granted through an approximate first order oracle, as defined in \cref{assm:opt_grad_oracle}.
We assume $f\colon \cX \to \R$ is differentiable and defined over an open domain $\dom f\subseteq \R^d$ that contains $\cX$.
We consider an approximate version of the algorithm described in \cref{eq:alg_opt_omd}, hence for the sake of rigor, we introduce some additional notation.
Given any convex regularizer $h\colon \R^d \to \R$, we define a Bregman proximal point update with step-size $\eta > 0$ by:
\begin{align}
	T_{\eta}(x; h) \eqq 
        \argmin_{y\in \cX} 
    		\cbr{\abr{\hgrad{f}{x}, y} + \frac1\eta B_{h}(y, x)},
\end{align}
and the set of $\epsopt$-approximate solutions by:
\begin{align}\label{def:bregman_prox_approx}
	T_{\eta}^{\epsopt}(x; h) 
        \eqq 
	\cbr{x^+ \in \cX \mid \forall z\in \cX: 
		\abr{\hgrad{f}{x} + \frac1\eta\nabla B_h(x^+, x), z - x^+} \geq -\epsopt}.
\end{align}
Now, the approximate version of our algorithm is given by:
\begin{align}\label{alg:opt_omd}
    k=1, \ldots, K: \quad 
    		x_{k+1} &\in T_{\eta}^{\epsopt}(x_k; R_{x_k}).
\end{align}
We recall our main theorem below.
\begin{theorem*}[restatement of \cref{thm:opt_main}]
    Suppose that $f$ is $(C_\star, \epsvgd)$-VGD as per \cref{def:gradient_dominance}, and that $f^\star \eqq \min_{x\in \cX}f(x) > -\infty$. 
    Assume further that:
    \begin{enumerate}[label=(\roman*)]
        \item The local regularizer $R_x$ is $1$-strongly convex and has an $L$-Lipschitz gradient w.r.t.~$\norm{\cdot}_x$ for all $x\in \cX$.
        
        \item For all $x\in \cX$, 
        $\max_{u,v\in \cX}\norm{u - v}_x\leq D$, and
        $\norm{\nabla f(x)}^*_x \leq M$.
        
        \item $f$ is $\beta$-locally smooth w.r.t.~$x \mapsto \norm{\cdot}_{x}$. 
    \end{enumerate}
    Then, for the algorithm described in \cref{alg:opt_omd} we have following guarantee when $\eta\leq 1/(2\beta)$:
	\begin{align*}
		f(x_{K+1}) - f^\star 
		&= O\br{\frac{C_\star^2 L^2 c_1^2}{\eta K}
            + \br{C_\star D + c_1 L^2}\epsgrad
        + C_\star \epsopt 
        + c_1 L \eta^{-\frac12}\sqrt{\epsopt}
            + \epsvgd
            }
        \end{align*}
	where $c_1 \eqq D + \eta M$.
\end{theorem*}
Evidently, since the objective is not convex, standard mirror descent analyses are inadequate, and our analysis takes the proximal point update view of \cref{alg:opt_omd}. While there are numerous prior works that investigate non-euclidean proximal point methods for both convex and non-convex objective functions
(e.g., \citealp{tseng2010approximation,ghadimi2016mini,bauschke2017descent,lu2018relatively,zhang2018convergence,fatkhullin2024taming}; see also \citealp{beck2017first})
, non of them fit into the specific setting we study here. The notable differences being the use of \emph{local} smoothness (\cref{def:local_smoothness}), and the goal of seeking convergence in function values for a non-convex objective by exploiting variational gradient dominance. 

Our approach may be best described as one that adapts the work of \citet{xiao2022convergence} to the non-euclidean (and, ``local'') setup, but without relying on the objective having a Lipschitz gradient (note that we do not claim our definition of local smoothness implies a Lipschitz gradient condition). Since \citet{xiao2022convergence} relies on global smoothness of the objective w.r.t.~the euclidean norm (as was established by \citealp{agarwal2021theory}), their bounds inevitably scale with the size of the state-space $S$, which we want to avoid.
Given any convex regularizer $h\colon \R^d \to \R$, we define Bregman gradient mapping by:
\begin{align}\label{def:breg_gm}
    G_\eta(x, x^+; h)
    \eqq 
    \frac1\eta\br{\nabla h(x) - \nabla h(x^+)},
\end{align}
where $x^+\in \R^d$ should be interpreted as an approximate proximal point update step, i.e., $x^+ \in T_\eta^{\epsopt}(x; h)$.

\subsection{Bregman prox: Descent and Stationarity}
In this section we provide basic results relating to proximal point descent and stationarity conditions.
Our first lemma is (roughly) a non-euclidean version of a similar lemma given in \cite{nesterov2013gradient} for the euclidean case.
\begin{lemma}[Bregman proximal step descent]
\label{lem:prox_descent}
        Let $\norm{\cdot}$ be a norm, and 
	suppose $x\in \cX$ is such that
    \begin{align*}
        \forall y\in \cX: \av{f(y) - f(x) - \abr{\nabla f(x), y - x}}
        \leq 
        \frac\beta2\norm{y - x}^2.
    \end{align*}
    Assume further that:
    \begin{enumerate}
    	\item $0 < \eta \leq 1/(2\beta)$,
    	\item $h\colon \R^d \to \R$ is $1$-strongly convex and has an $L_h$-Lipschitz gradient, w.r.t. $\norm{\cdot}$.
    	\item $\norm{\hgrad{f}{x} - \grad{f}{x}}_*\leq \epsgrad$.
    \end{enumerate}
	Then, for $x^+ \in T_\eta^\epsopt(x; h)$ we have that:
	\begin{align*}
		f(x^+) \leq f(x) 
		-\frac{\eta}{ 2 L_h^2}\norm{G_\eta(x, x^+; h)}_*^2
	        +\eta \epsgrad \norm{G_\eta(x, x^+; h)}_* + \epsopt.	
	\end{align*}
\end{lemma}
\begin{proof}
Observe,
\begin{align*}
	f(x^+) 
	&\leq 
	f(x) + \abr{\nabla f(x), x^+ - x}
		+ \frac{\beta}{2}\norm{x^+ - x}^2
	\\
	&= 
	f(x) + \abr{\hgrad{f}{x}, x^+ - x}
		+ \frac{\beta}{2}\norm{x^+ - x}^2
		+ \abr{\grad{f}{x} - \hgrad{f}{x}, x^+ - x}
	\\
	&\leq 
	f(x) + \abr{\hgrad{f}{x}, x^+ - x}
		+ \frac{\beta}{2}\norm{x^+ - x}^2
		+ \epsgrad\norm{x^+ - x}
        \\
	&\leq 
	f(x) + \abr{\hgrad{f}{x}, x^+ - x}
		+ \frac{\beta}{2}\norm{x^+ - x}^2
		+ \eta\epsgrad\norm{G_\eta(x, x^+; h)}_*
        \tag{\cref{lem:breg_to_gm}}
        .
\end{align*}
Further, since $x^+\in T_\eta^\epsopt(x; h)$, for any $z\in \cX$, 
$$
	\abr{\hgrad{f}{x}, x^+ - z} 
	\leq 
	\abr{\frac1\eta\br{\nabla h(x^+) - \nabla h(x)}, z - x^+}
	+\epsopt.
$$
Hence,
\begin{align*}
	&f(x) + \abr{\hgrad{f}{x}, x^+ - x}
		+ \frac{\beta}{2}\norm{x^+ - x}^2
	\\
	&\leq
	f(x) + \frac1\eta\abr{\nabla h(x) -\nabla h(x^+), x^+ - x}
		+ \frac{\beta}{2}\norm{x^+ - x}^2
		+\epsopt
	\\
	&=
	f(x) 
        - \frac1\eta\br{B_h(x^+, x) + B_h(x, x^+)}
	 	+ \frac{\beta}{2}\norm{x^+ - x}^2
	 	+ \epsopt
	\\
	&\leq
	f(x) - \frac1\eta\br{B_h(x^+, x) + B_h(x, x^+)}
		+ \beta B_h(x^+, x)
		+ \epsopt
	\\
	&\leq
	f(x) - \frac1{2\eta}
        \br{B_h(x^+, x) + B_h(x, x^+)} + \epsopt,
\end{align*}
where the last line inequality follows from $\eta\leq 1/(2\beta)$.
Combining with the previous derivation, we now have
\begin{align}\label{eq:prox_descent_main}
	f(x^+)
	\leq 
		f(x) 
		- \frac1{2\eta}
        \br{B_h(x^+, x) + B_h(x, x^+)}
        + \epsopt
		+\eta \epsgrad \norm{G_\eta(x, x^+; h)}_*.
\end{align}
Finally, the assumption that $h$ has an $L_h$-Lipschitz gradient implies that 
\begin{align*}
	\eta^2\norm{G_\eta(x, x^+; h)}_*^2
	=
	\norm{\nabla h(x^+) - \nabla h(x)}_*^2
	\leq 
	L_h^2 \norm{x^+ - x}^2
	\leq 2L_h^2 B_h(x^+, x),
\end{align*}
and similarly $\eta^2\norm{G_\eta(x, x^+; h)}_*^2\leq 2L_h^2 B_h(x, x^+)$.
Hence,
\begin{align*}
	- \frac1{2\eta}
        \br{B_h(x^+, x) + B_h(x, x^+)}
        \leq 	
        - \frac{\eta}{ 2L_h^2}\norm{G_\eta(x, x^+; h)}_*^2,
\end{align*}
which completes the proof after combining with \cref{eq:prox_descent_main}.
\end{proof}

Our second lemma bounds the error in optimality conditions at any point $x\in \cX$ w.r.t.~the gradient mapping dual norm. We remark that here we do not assume a Lipschitz gradient condition holds for the objective function, as commonly done in similar arguments (e.g., \citealp{nesterov2013gradient,xiao2022convergence}).
\begin{lemma}\label{lem:prox_optcond}
	Let $\norm{\cdot}$ be a norm, and $x\in \cX$. Assume that:
	\begin{enumerate}
		\item $h\colon \R^d \to \R$ is $1$-strongly convex and has an $L_h$-Lipschitz gradient, w.r.t. $\norm{\cdot}$.
    	\item $\norm{\hgrad{f}{x} - \grad{f}{x}}_*\leq \epsgrad$,
    	\item $D>0$ upper bounds the diameter of $\cX$: $\max_{z,y\in \cX}\norm{z - y}\leq D$
    	\item $M>0$ upper bounds the gradient dual norm at $x$: $\norm{\nabla f(x)}_*\leq M$.
    \end{enumerate}
	Then, if $x^+ \in  T_\eta^\epsopt(x; h)$, it holds that:
	\begin{align*}
		\forall y\in \cX:
		\abr{\nabla f(x), x - y} 
		\leq 
		(D + \eta M)\norm{G_{\eta}(x, x^+; h)}_* + \epsgrad D + \epsopt.
	\end{align*}

\end{lemma}
\begin{proof}
		By assumption, we have for all $y\in \cX$,
	\begin{align*}
		\abr{\hgrad{f}{x} - G_{\eta}(x, x^+;h), y - x^+} 
		&\geq -\epsopt
		\\
		\iff
		\abr{\nabla f(x), x^+ - y} 
		&\leq 
		\abr{G_{\eta}(x, x^+;h), y - x^+}
		+ 
		\abr{\nabla f(x) - \hgrad{f}{x}, x^+ - y}
		+ \epsopt
		\\
		&\leq 
		\norm{G_{\eta}(x, x^+;h)}_* D + \epsgrad D
		+ \epsopt
		.
	\end{align*}
	Further,
	\begin{align*}
		\abr{\nabla f(x), x - y}
		&= 
		\abr{\nabla f(x), x^+ - y}
		+
		\abr{\nabla f(x), x - x^+}
		\\
		&\leq
		\norm{G_{\eta}(x, x^+;h)}_* D + \epsgrad D + \epsopt
		+ \abr{\nabla f(x), x - x^+}
		\\
		&\leq
		\norm{G_{\eta}(x, x^+;h)}_* D + \epsgrad D + \epsopt
		+ M \norm{x - x^+}
		\\
		&\leq
		\norm{G_{\eta}(x, x^+;h)}_* D + \epsgrad D + \epsopt
		+ \eta M \norm{G_\eta(x, x^+;h)}_*
		\tag{\cref{lem:breg_to_gm}}
		\\
		&\leq
		\br{D + \eta M}\norm{G_{\eta}(x, x^+;h)}_*  + \epsgrad D + \epsopt
		,
	\end{align*}
    which completes the proof.
\end{proof}

    \begin{lemma}\label{lem:breg_to_gm}
        For any norm $\norm{\cdot}$,
        and any $x, x^+\in \cX$, we have
        $
    	\norm{x-x^+}
    	\leq 
    	\eta \norm{G_{\eta}(x, x^+;h)}_*.
    	$
    \end{lemma}
    \begin{proof}
        For any $u,v$ it holds that (see e.g., \citealp{hiriart2004fundamentals}), 
    \begin{align*}
    	\frac12\norm{u-v}^2 \leq B_h(u, v)
    	= B_{h^*}(\nabla h(u), \nabla h(v))
    	\leq 
    	\frac12\norm{\nabla h(u) - \nabla h(v)}_*^2.
    \end{align*}
    The result now follows by the definition of $G_\eta(x, x^+; h)$.
\end{proof}

\subsection{Proof of \cref{thm:opt_main}}
We begin by establishing the objective satisfies a weak gradient mapping domination condition similar (but not identical, due to the differences mentioned above) to that considered in  \citet{xiao2022convergence}. 
\begin{definition}\label{def:weakgm_local}
	 We say that $f\colon \cX \to \R$ satisfies a weak gradient mapping domination condition w.r.t.~a local regularizer $R$ 
	 	if there exist $\delta, \omega > 0$ such that for all $x\in \cX$:
	 \begin{align*}
	 	\norm{G_{\eta}(x, x^+; h)}_{x}^*
	 	\geq \sqrt{2\omega}(f(x) - f^\star - \delta)
	 \end{align*}
\end{definition}
The lemma below establishes our objective function satisfies \cref{def:weakgm_local} with a suitable choice of parameters.
\begin{lemma}\label{lem:weakgm}
    Suppose that $f$ is $(C_\star, \epsvgd)$-VGD as per \cref{def:gradient_dominance}, and that $f^\star \eqq \min_{x\in \cX}f(x) > -\infty$. 
    Assume further that $R_x$ is $1$-strongly convex and has an $L$-Lipschitz gradient w.r.t. $\norm{\cdot}_x$ for all $x\in \cX$. Then, we have the following weak gradient mapping domination condition; 
    for all $x\in \cX, x^+\in T_{\eta}^{\epsopt}(x; R_x)$:
	\begin{align*}
		\norm{G_{\eta}(x, x^+; R_x)}_x^*
		\geq \sqrt{2\omega}\br{f(x) - f^\star - \delta},
	\end{align*}
	for $\omega\eqq \frac12\br{C_\star(D + \eta M)}^{-2}$, 
	$\delta\eqq \epsvgd + \epsopt C_\star + \epsgrad C_\star D$.
\end{lemma}
\begin{proof}
	Let $x\in \cX$, and
	apply \cref{lem:prox_optcond} with $\norm{\cdot}=\norm{\cdot}_x$ and $h=R_x$, to obtain:
	\begin{align*}
		\forall y\in \cX:
		\abr{\nabla f(x), x - y} 
		\leq (D+\eta M)\norm{G_{\eta}(x, x^+; R_x)}_x^*
            +\epsgrad D + \epsopt
		.
	\end{align*}
	Further, since $f$ is $(C_\star,\epsvgd)$-VGD, we have
	\begin{align*}
            \max_{y \in \cX}\abr{\nabla f(x), x - y}
        \geq \frac1{C_\star}\br{f(x) - f^\star - \epsvgd} .
        \end{align*}
	Combining both inequalities, the result follows.
\end{proof}
We are now ready to prove \cref{thm:opt_main}.
\begin{proof}[Proof of \cref{thm:opt_main}]
	In the sake of notational clarity, define:
	\begin{align}
		\cG_k \eqq \norm{G_\eta(x_k, x_{k+1}; R_{x_k})}^*_{x_k}.
	\end{align}
We begin by applying \cref{lem:prox_descent} for every $k \in [K]$ with 
$\norm{\cdot}=\norm{\cdot}_{x_k}$ and $h=R_{x_k}$, which implies,
\begin{align}\label{eq:opt_descent_main}
    f(x_{k+1}) - f(x_k)
    \leq  
    -\frac{\eta}{2L^2}\cG_k^2
        +\eta \epsgrad \cG_k
        +\epsopt.
\end{align}
Let us first assume that for all $k\in [K]$:
\begin{align}\label{eq:opt_small_grad_error}
    8L^2\epsgrad + \frac{4 L}{\sqrt\eta} \sqrt{\epsopt} \leq \cG_k.
\end{align}
Then \cref{eq:opt_descent_main} along with  \cref{lem:weakgm} gives
	\begin{align*}
		f(x_{k+1}) - f(x_k)
        \leq 
	-\frac{\eta}{4L_h^2}\cG_k^2
        \leq -\frac{\eta \omega}{ 4L^2}\br{f(x_k) - f^\star - \delta}^2,
	\end{align*}
	with $\omega\eqq \frac12\br{C_\star (D + \eta M)}^{-2}$ 
	and $\delta\eqq \epsvgd + \epsopt C_\star + \epsgrad C_\star D$.
    We proceed to define $E_k \eqq f(x_k) - f^\star$, and note that the above display implies $E_{k+1} \leq E_k$.
    Hence, we may assume that $E_k \geq 2\delta$ for all $k\in [K]$, otherwise the claim holds trivially.
    With this in mind, we now have,
	\begin{align*}
		E_{k+1} - E_k
		&\leq - \frac{\eta \omega}{4 L^2} (E_k - \delta)^2
            \leq 
            - \frac{\eta \omega}{16 L^2} E_k^2.
	\end{align*}
        Dividing both sides by $E_k E_{k+1}$ yields
	\begin{align*}
		\frac1{E_k} - \frac1{E_{k+1}}
		&\leq - \frac{\eta \omega}{16 L^2} \frac{E_k}{E_{k+1}}.
	\end{align*}
	Summing over $k=1, \ldots, K$ and telescoping the sum on the LHS, we obtain
	\begin{align*}
		\frac1{E_1} - \frac1{E_{K+1}}
		&\leq - \frac{\eta \omega}{16 L^2} \sum_{k=1}^K\frac{E_k}{E_{k+1}}
		\\
		\iff
		E_{K+1} - E_1
		&\leq - \frac{\eta \omega}{16 L^2} \br{E_{K+1}E_1}\sum_{k=1}^K\frac{E_k}{E_{k+1}}
		\leq - \frac{\eta \omega}{16 L^2} \br{E_{K+1}E_1}K,
	\end{align*}
	where the last inequality follows from the descent property $E_{k+1}\leq E_k$. Rearranging, we now have	
		\begin{align*}
			0\leq E_{K+1}
			&\leq E_1 \br{1- \frac{\eta \omega}{16 L^2} E_{K+1}K}
			\\
			\implies
			E_{K+1} &\leq \frac{16 L^2}{\eta \omega  K}
			=\frac{32 C_\star^2 L^2 \br{D + \eta M}^2}{\eta K},
		\end{align*}
        which completes the proof for the case that \cref{eq:opt_small_grad_error} holds for all $k\in [K]$.
        Assume now that this is not the case, and let $k_0\in [K]$ be the last iteration such that 
        \begin{align*}
            \cG_{k_0}
            < 8L^2\epsgrad + \frac{4 L}{\sqrt\eta} \sqrt{\epsopt}.
        \end{align*}
        Then by \cref{lem:prox_optcond},
        \begin{align*}
            E_{k_0} \leq 
            (D + \eta M)\cG_{k_0}
            + \epsgrad D + \epsopt 
            \leq 
            8(D + \eta M) \br{L^2 \epsgrad + L\sqrt{\epsopt/\eta}}
            + \epsgrad D + \epsopt,
        \end{align*}
        and therefore by \cref{eq:opt_descent_main},
        \begin{align*}
            E_{k_0+1} \leq E_{k_0} + \eta \epsgrad \cG_{k_0}
            = O\br{ (D + \eta M)\br{L^2 \epsgrad + L\sqrt{\epsopt/\eta}} }.
        \end{align*}
        Now, if $k_0 = K$ we are done. Otherwise, by the definition of $k_0$ we have that \cref{eq:opt_small_grad_error} holds for all $k\in [k_0+1, K]$, hence $E_{k+1} \leq E_k$ for all $k\geq k_0+1$. This implies that $E_{K+1} \leq E_{k_0 + 1}$, which completes the proof.
\end{proof}

\subsection{Convergence to stationary point without a VGD condition}
\label{sec:prox_convergence_stationary_point}
In this section, we include a proof that the proximal point algorithm we consider converges to a stationary point, also without assuming a VGD condition. The proof follows from standard arguments and is given for completeness; for simplicity, we provide analysis only for the error free case.
As an implication, we have that PMD converges to a stationary point in any MDP; this follows by combining the below theorem with \cref{lem:pmd_main} and \cref{lem:value_local_smoothness}, and proceeding with an argument similar to that of \cref{thm:pmd_main}.

\begin{theorem}\label{thm:opt_stationary_point}
    Suppose that $f^\star \eqq \min_{x\in \cX}f(x) > -\infty$, and assume:
    \begin{enumerate}[label=(\roman*)]
        \item The local regularizer $R_x$ is $1$-strongly convex and has an $L$-Lipschitz gradient w.r.t.~$\norm{\cdot}_x$ for all $x\in \cX$.
        
        \item For all $x\in \cX$, 
        $\max_{u,v\in \cX}\norm{u - v}_x\leq D$, and
        $\norm{\nabla f(x)}^*_x \leq M$.
        
        \item $f$ is $\beta$-locally smooth w.r.t.~$x \mapsto \norm{\cdot}_{x}$. 
    \end{enumerate}
	Consider an exact version of the proximal point algorithm \cref{eq:alg_opt_omd} with $\eta=1/(2\beta)$ where $\epsgrad=0$ and $x^{k+1}=T_\eta(x_k; R_{x_k})$ for all $k$.
	Then, after $K$ iterations, there exists $k^\star\in [K]$
	such that:
    \begin{align*}
    	\forall y\in \cX, \quad \abr{\nabla f(x_{k^\star}), y - x_{k^\star}}
    	\geq - \frac{2 (D+ \eta M)L\sqrt{\beta\br{f(x_1) - f(x^\star)}}}{\sqrt{K}},
    \end{align*}
\end{theorem}
\begin{proof}
In the sake of notational clarity, define:
	\begin{align*}
		\cG_k \eqq \norm{G_\eta(x_k, x_{k+1}; R_{x_k})}^*_{x_k}.
	\end{align*}
We begin by applying \cref{lem:prox_descent} for every $k \in [K]$ with 
$\norm{\cdot}=\norm{\cdot}_{x_k}$ and $h=R_{x_k}$, which implies,
\begin{align}
    f(x_{k+1}) - f(x_k)
    \leq  
    -\frac{\eta}{2L^2}\cG_k^2.
\end{align}
	Now,
    \begin{align*}
        f(x_{K+1}) - f(x_1)
        = 
        \sum_{k=1}^K
        f(x_{k+1}) - f(x_k)
        \leq - \frac{\eta}{2 L^2}\sum_{k=1}^K \cG_k^2,
    \end{align*}
    thus, rearranging and bounding $f(x_{K+1}) \geq f(x^\star)$ gives
    \begin{align*}
        \frac1K\sum_{t=1}^T \cG_k^2
        \leq 
        \frac{2 L^2\br{f(x_1) - f(x^\star)}}{\eta K}.
    \end{align*}
    Hence, it must hold for $k^\star\eqq \argmin_{k}\cG_k^2$; 
    \begin{align*}
	\cG_{k^\star}^2
      \leq  \frac{2 L^2\br{f(x_1) - f(x^\star)}}{\eta K}.
    \end{align*}
    We now apply \cref{lem:prox_optcond} to conclude,
    \begin{align*}
    	\forall y\in \cX, \quad \abr{\nabla f(x_{k^\star}), x_{k^\star} - y}
    	\leq \frac{ (D+ \eta M)L\sqrt{2\br{f(x_1) - f(x^\star)}}}{\sqrt{\eta K}},
    \end{align*}
    which implies the required result.
    
\end{proof}

\section{Policy Classes with Dual Parametrizations}
\label{sec:dual_pc}
In general, solving the following OMD problem in some state $s\in \cS$,
\begin{align}\label{eq:dualpc_1}
	\pi^{k+1}_s \gets \argmin_{p\in \Delta(\cA)} \abr{Q^k_s, p} + \frac1\eta B_R(p, \pi^k_s)
\end{align}
is equivalent to the following two updates:
\begin{align*}
	\nabla R(\tilde \pi_s^{k+1}) &\gets \nabla R(\pi^k_s) - \eta Q^k_s
	\\
	\pi^{k+1}_s &= \Pi_{\Delta(\cA)}^R\br{\tilde \pi_s^{k+1}}.
\end{align*}
Let us denote the composition of the dual-to-primal mirror-map and the projection:
$$
	P_R(y) \eqq \Pi_{\Delta(\cA)}^R\br{\nabla R^*(y)},
$$
and note that
$$
	\pi_s^{k+1} = P_R(\nabla R(\tilde \pi_s^{k+1})).
$$

When we are in a non-tabular setup and have a non-complete policy class $\Pi\neq \Delta(\cA)^\cS$, we cannot update each state independently according to \cref{eq:dualpc_1}.
There are however a number of places we can "intervene" in the policy class representation to derive slightly different update procedures based on the dual variables. The PMD step in its general form is given by:
\begin{align}\label{eq:dualpc_2}
	\pi^{k+1} \gets \argmin_{\pi \in \Pi} \E_{s\sim\mu^k}\sbr{\abr{Q^k_s, \pi_s} + \frac1\eta B_R(\pi_s, \pi^k_s)}
\end{align}

Without making any assumptions regarding the parametric form of $\Pi$, we cannot decompose \cref{eq:dualpc_2} into meaningful dual space steps. We discuss next two types of policy class parameterizations and the update steps associated with them.

\subsection{Generic dual parameterizations}
This is the approach taken in \cite{alfano2023novel} (see also the followup \citealp{xiong2024dual}), and perhaps the most general one that allows for an explicit dual space update as well as leads to an approximate solution of \cref{eq:dualpc_2} that satisfies approximate optimality conditions in the complete-class setting. Consider a parametric function class $\cF_\Theta \eqq \cbr{f_\theta \in \R^{SA} \mid \theta \in \Theta}$, and the policy class:
$$
	\Pi(\cF) \eqq \cbr{ \pi^f \mid  f \in \cF_\Theta},
	\quad \text{ where } \pi^f_s \eqq P_R(f_s),\;\forall s\in \cS.
$$
Then, to solve \cref{eq:dualpc_2} we can proceed by:
\begin{align*}
	f^{k+1} &\gets \argmin_{f \in \cF}\E_{s\sim \mu^k}\sbr{
		\norm{f_s - \nabla R(\pi^k_s) - \eta Q^k_s}_2^2}
	\\
	\pi^{k+1} &\gets \text{ the policy defined by }	\pi^{k+1}_s = P_R(f^{k+1}_s)
	\tag{A}
\end{align*}
\subsection{The log-linear policy class}
This is a special case of the one discussed in the previous sub-section. In general, when we try to approximate the true solution of the unconstrained mirror descent step in a specific state:
$$
	f_s \approx \nabla R(\pi^k_s) - \eta Q_s^k,
$$
we need to overcome two sources of error; one from the previous policy dual variable and one from the $Q$ function. More specifically, in general we have $\nabla R(\pi^k) \notin \cF$ and $Q^k \notin \cF$. (For $\pi\in \R^{SA}$ we define $\nabla R(\pi)_s \eqq \nabla R(\pi_s)$.)
In the special case that our function class $\cF$ can represent $\nabla R(\pi)$ perfectly for all $\pi\in \Pi(\cF)$ and is closed to linear combinations, we can focus our attention on approximating the $Q$ function. Now, we may proceed by the following special case of $(A)$:
\begin{align*}
	\widehat Q^k &\gets \argmin_{\widehat Q \in \cF}\E_{s, a\sim \mu^k}\sbr{
		\br{\widehat Q_{s, a} - Q^k_{s, a}}^2}
	\\
	f^{k+1} &\gets \nabla R(\pi^k) - \eta \widehat Q^k
	\\
	\pi^{k+1} &\gets \text{ the policy defined by }	\pi^{k+1}_s = P_R(f^{k+1}_s).
	\tag{B}
\end{align*}
Let $\phi_{s, a} \in \R^p$ be given feature vectors, and let $\phi_{s} \eqq [\phi_{s, a_1} \cdots \phi_{s, a_A}]\in \R^{p \times A}$, and consider
the log-linear policy class:
\begin{align*}
	\Pi &\eqq \cbr{\pi^\theta \mid \theta \in \R^p}
	\\
	 \text{ where }\forall s\in \cS,\; \pi^\theta_s &\eqq P_R(\phi_s\T\theta)
	 = \frac{e^{\phi_s\T \theta}}{\sum_{a}e^{\phi_{s, a}\T \theta}}.
\end{align*}
Note that:
\begin{enumerate}
    \item This is the class $\Pi(\cF)$ for $\cF = \cbr{\theta \mapsto \br{(s, a) \mapsto \phi_{s,a}\T \theta}}$.
    \item This is precisely a case where $\cF$ can model $\nabla R(\pi)$ if $R$ is the negative entropy regularizer. 
\end{enumerate}
Here, we may proceed as follows:
\begin{align*}
	w^{k} &\gets \argmin_{w\in \R^p}
		\E_{s, a \sim \mu^k}\sbr{\br{\phi_{s, a}\T w - Q_{s,a}^k}^2}
	\\
	\theta^{k+1} &\gets \theta^k - \eta w^k
	\\
	\pi^{k+1} &\gets \text{ the log-linear policy defined by } \theta^{k+1}
\end{align*}
The above can be seen as a special case of $(B)$, by considering the induced updates in state-action space:
\begin{align*}
\widehat Q^{k} &=
	\argmin_{\widehat Q \in \cF}
		\E_{s, a \sim \mu^k}\sbr{\br{\widehat Q_{s, a}  - Q_{s,a}^k}^2}
		= (s, a) \mapsto \phi_{s, a}\T w^k
	\\
	f^{k+1}_s
	&= \nabla R(\pi^k_s) - \eta \widehat Q^k_s
	\\
	&= \log(e^{\phi_s\T\theta^k}) - \log(Z^k_s)\boldsymbol 1 - \eta \widehat Q^k_s
	\\
	&= \phi_s\T\theta^k - \eta \widehat Q^k_s - \log(Z^k_s)\boldsymbol 1
	\\
	\pi^{k+1} &\gets  \text{ the policy defined by }\pi^{k+1}_s 
	= P_R(f^{k+1}_s) 
	= \frac{e^{\phi_s\T \theta^{k+1}}}{\sum_{a}e^{\phi_{s, a}\T \theta^{k+1}}}.
\end{align*}
Note that in the above,
$$
	f^{k+1}_{s, a} = \phi_{s, a}\T \theta^k - \eta \phi_{s, a}\T w^k - \log (Z^k_s),
$$
and that $Z^k_s$ is the same for all actions in $s$, hence makes no difference after the projection step..

\end{document}